\documentclass{article}

\usepackage{microtype}
\usepackage{graphicx}

\usepackage[accepted]{icml2026}

\definecolor{navy}{rgb}{0.00,0.078,0.447}
\usepackage[pagebackref,breaklinks,colorlinks,allcolors=navy]{hyperref}
\usepackage{algorithm}
\usepackage{algpseudocode}
\usepackage{multirow}
\usepackage{booktabs,array}
\usepackage{adjustbox}
\usepackage{graphicx}
\usepackage{caption}
\usepackage{lipsum}
\usepackage{subcaption}
\usepackage{amsmath,amsfonts,amssymb}
\usepackage{amsthm}
\newtheorem{theorem}{Theorem}[section]

\newtheorem{proposition}[theorem]{Proposition}
\usepackage{subcaption}
\usepackage{enumitem}
\usepackage{siunitx}
\usepackage{placeins}    
\usepackage{dblfloatfix}
\usepackage{hyperref}

\makeatletter
\newlength{\eqtagwd}
\newcommand{\eqfit}[1]{%
  \settowidth{\eqtagwd}{\tagform@{\theequation}}%
  \adjustbox{max width=\dimexpr\linewidth-\eqtagwd-1.05em\relax}{$\displaystyle #1$}%
}
\makeatother
\usepackage[textsize=tiny]{todonotes}

\icmltitlerunning{Geometry-Correct Diffusion Posterior Sampling with Denoiser-Pullback Curvature Guidance and Manifold-Aligned Damping}

\begin{document}
\twocolumn[
  \icmltitle{Geometry-Correct Diffusion Posterior Sampling with Denoiser-Pullback Curvature Guidance and Manifold-Aligned Damping}



  \icmlsetsymbol{equal}{*}
  \begin{icmlauthorlist}
    \icmlauthor{Seunghyeok Shin}{equal,sch}
    \icmlauthor{Minwoo Kim}{equal,sch}
    \icmlauthor{Dabin Kim}{sch}
    \icmlauthor{Hongki Lim}{sch}
  \end{icmlauthorlist}
  \icmlaffiliation{sch}{Department of Electrical and Computer Engineering, Inha University, Incheon, 22212, South Korea}
  \icmlcorrespondingauthor{Hongki Lim}{hklim@inha.ac.kr}
  \icmlkeywords{Machine Learning, ICML}

  \vskip 0.3in]
\begin{figure*}
\centering
\captionsetup{type=figure}
\includegraphics[width=1.0\linewidth]{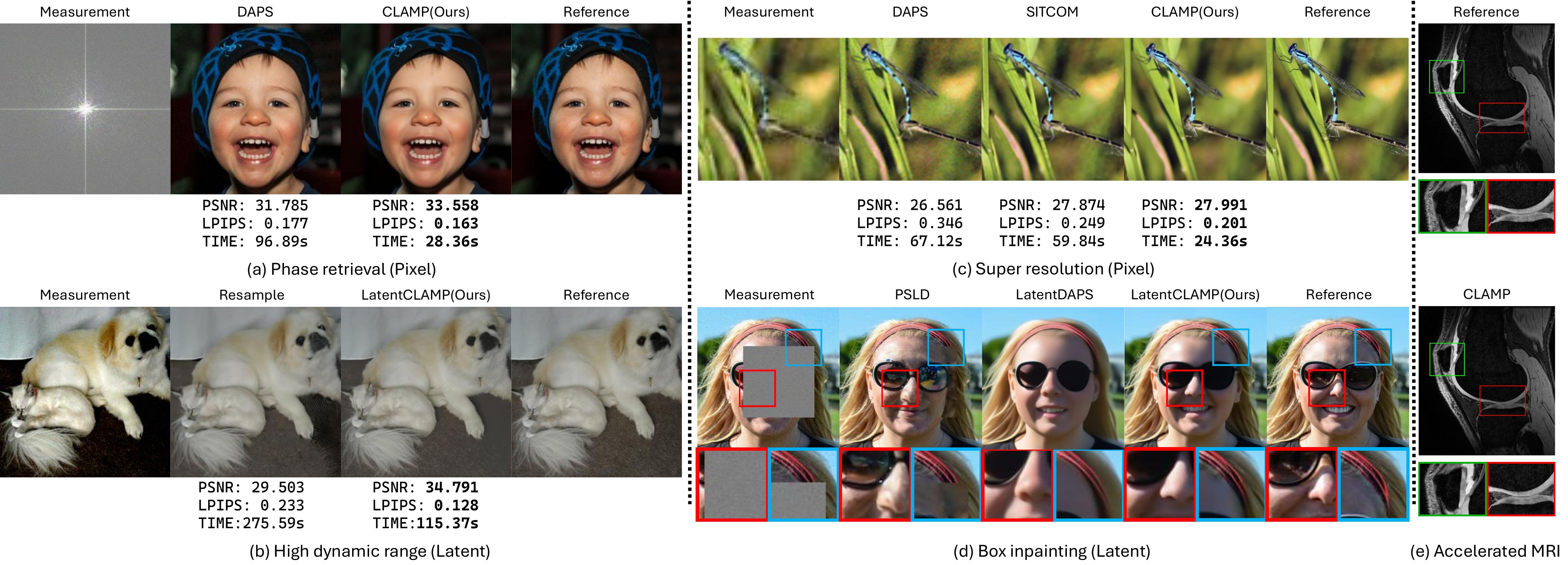}
\captionof{figure}{\textbf{Qualitative results on diverse inverse problems.}
(a--d) Qualitative comparisons on FFHQ and ImageNet at \(256\times256\), covering both pixel- and latent-space variants of \textsc{clamp} and baseline methods across a range of general inverse problems. For each task, we report the best competing results for a representative visual comparison. 
(e) Visualization of \textsc{clamp} on accelerated MRI reconstruction.
Overall, \textsc{clamp} produces high-fidelity, visually plausible reconstructions and remains effective across pixel/latent settings, linear/nonlinear operators, and both natural and medical imaging domains.}
\label{fig:main}
\end{figure*}



\printAffiliationsAndNotice{\icmlEqualContribution}

\begin{abstract}
Diffusion posterior sampling conditions diffusion priors on measurements, but data-consistency updates are typically scaled by hand-tuned guidance weights and can destabilize sampling under stiff, operator-dependent curvature. We replace scalar guidance with a per-noise-level damped Gauss--Newton correction computed in diffusion-state coordinates. The correction pulls likelihood gradients back through the denoiser, uses a one-sided curvature model that avoids forward denoiser Jacobians, and applies diffusion-calibrated rank-one damping aligned with the denoiser residual. Each correction is solved with matrix-free GMRES using automatic differentiation, and sampling proceeds with a variance-preserving Langevin transition with a closed-form drift/noise split. On FFHQ and ImageNet across inverse problems, it achieves competitive PSNR/SSIM/LPIPS while running markedly faster than most of the compared baselines; on accelerated MRI reconstruction, it achieves the best PSNR/SSIM among the compared baselines. Code is available at \href{https://github.com/Seunghyeok0715/CLAMP}{\nolinkurl{github.com/Seunghyeok0715/CLAMP}}.

\end{abstract}

\section{Introduction}

Inverse problems aim to infer an unknown signal from indirect and noisy measurements. We consider the standard additive Gaussian model
\begin{equation}
\mathbf{y}=A(\mathbf{x}_0)+\mathbf{n},\qquad \mathbf{n}\sim\mathcal{N}(\mathbf{0},\beta_{\mathbf{y}}^2\mathbf{I}),
\label{eq:intro_meas}
\end{equation}
where $A$ is a (possibly nonlinear) forward operator and $\mathbf{x}_0$ is the unknown clean image. Because the mapping $\mathbf{x}_0\mapsto \mathbf{y}$ is typically ill-posed, many $\mathbf{x}_0$ can explain the same observation~\cite{bertero1998inverse,calvetti2007bayesian}. In such settings, posterior sampling can be more informative than returning a single reconstruction: it generates multiple plausible explanations of the data and supports principled uncertainty quantification \cite{stuart2010bayesian,dashti2017bayesian}.

Pretrained diffusion models are strong natural-image priors and are widely used in training-free inverse problem solvers~\cite{feng2023principled_priors}. They generate samples by progressively denoising noisy states \cite{pmlr-v37-sohl-dickstein15,ho2020ddpm,song2019gradients,song2021sde}, and EDM provides a practical noise-conditioned denoiser interface and schedule \cite{karras2022edm}. Because runtime is dominated by denoiser calls, acceleration (e.g., distillation) is relevant under fixed compute budgets \cite{salimans2022distillation}.

A widely used conditional strategy is to run a diffusion sampler while injecting likelihood-driven data-consistency updates, as in diffusion posterior sampling (DPS) \cite{chung2023dps}. Under \eqref{eq:intro_meas}, the physical Gaussian negative log-likelihood is
$\Phi_{\mathrm{meas}}(\mathbf{x})=\tfrac{1}{2\beta_{\mathbf{y}}^2}\|A(\mathbf{x})-\mathbf{y}\|_2^2$.
Interleaving ``prior'' diffusion steps with likelihood corrections is conceptually straightforward, but the likelihood correction is often fragile: it must be strong enough to reduce measurement error yet small enough to avoid pushing iterates into low-density regions of the diffusion prior.

Two coupled issues make this correction hard to calibrate across operators and noise levels.
First, the local geometry of the likelihood can be strongly anisotropic: a single scalar step size must be small enough to remain stable in stiff directions, which can slow progress in flatter directions (especially for nonlinear $A$ or latent parameterizations).
Second, measurement consistency is defined on a clean variable (the input of $A$), while the sampler state is a noisy diffusion variable.
In practice, one evaluates the measurement residual on the denoised prediction produced at each noise level, so the effective mismatch depends on the composition of the forward model with the denoiser. As a result, clean-space likelihood information must be mapped back to the diffusion-state update; ignoring this mapping can mis-scale directions when the denoiser is anisotropic.

Existing work mitigates these failure modes in complementary ways, including trajectory modifications (e.g., DAPS \cite{Zhang_2025_CVPR}), per-step optimization (e.g., SITCOM \cite{pmlr-v267-alkhouri25b}), stabilization via preconditioning \cite{Garber_2024_CVPR} or projections \cite{chung2022mcg}, diffusion-bridge constructions \cite{chung2023ddb}, and likelihood-update redesigns \cite{yismaw2025gaussian,graikos2024fastconstrained,hamidi2025cdps}. Despite this progress, many practical samplers still rely on scalar guidance schedules, and they do not explicitly combine coordinate-correct likelihood geometry with direction-wise curvature scaling and diffusion-calibrated step control in a single per-noise-level correction.

This paper introduces \textsc{clamp} (\textbf{C}urvature-aware \textbf{L}angevin with \textbf{A}ligned \textbf{M}anifold \textbf{P}ullback), a diffusion posterior sampler that directly targets the two issues above.
To handle anisotropic likelihood geometry, \textsc{clamp} replaces scalar guidance with a curvature-scaled correction based on a damped Gauss--Newton (Fisher) quadratic model of the measurement residual at each noise level.
To ensure coordinate correctness, the correction is computed in diffusion-state coordinates by pulling likelihood sensitivity back through the denoiser, rather than applying clean-space derivatives directly.
For robustness and efficiency, we use a one-sided curvature approximation that avoids forward denoiser-Jacobian products inside curvature terms while retaining the pullback that maps measurement sensitivity into diffusion-state updates.
Step sizes are stabilized via a trust-region/Levenberg--Marquardt viewpoint \cite{nocedal2006numerical,conn2000trust}: we regularize the correction in a prior-aligned geometry and calibrate its strength to the diffusion noise level, avoiding per-task step-size schedules. The resulting update is computed with a fixed, lightweight matrix-free solve using Jacobian actions, yielding a favorable accuracy--runtime trade-off.
After the correction, \textsc{clamp} advances the diffusion schedule with a variance-preserving Langevin-style step with a closed-form drift/noise split, avoiding an additional per-step stochasticity parameter.
The same construction applies in latent space by composing the forward model with a decoder, replacing $A$ by $A\circ D$.

Our contributions are:
\begin{itemize}
\item \textbf{Curvature-scaled, coordinate-correct likelihood corrections.} We formulate per-noise-level data-consistency updates in diffusion-state coordinates using denoiser pullback geometry and a damped Gauss--Newton curvature model, addressing both anisotropic stiffness and coordinate mismatch.
\item \textbf{Diffusion-calibrated, prior-aligned anisotropic step control.}
We introduce a rank-one damping metric aligned with the denoiser residual direction and a schedule-calibrated regularization rule that stabilizes likelihood corrections across noise levels without an explicit guidance schedule.
\item \textbf{Variance-preserving Langevin propagation.}
After each correction, we propagate to the next noise level with a Langevin-style transition whose drift and injected-noise coefficients are set in closed form by variance matching and fluctuation--dissipation coupling, preserving the prescribed diffusion noise schedule.
\item \textbf{Fast matrix-free implementation in pixel and latent spaces.}
We solve each correction with a fixed-budget Krylov method using only Jacobian actions and evaluate \textsc{clamp} on natural-image and medical-imaging inverse problems, demonstrating improved reconstruction quality and sampling efficiency over representative baselines.
\end{itemize}
Section~\ref{sec:background} reviews diffusion denoisers, damped Gauss--Newton, and matrix-free Krylov solvers. Section~\ref{sec:Method} derives \textsc{clamp}, and Section~\ref{sec:Experiments} reports quantitative results and ablations on FFHQ/ImageNet inverse problems and MRI reconstruction; Section~\ref{sec:Conclusion} concludes.

\begin{figure*}
\centering
  \makebox[\textwidth][c]{\includegraphics[width=1.0\linewidth]{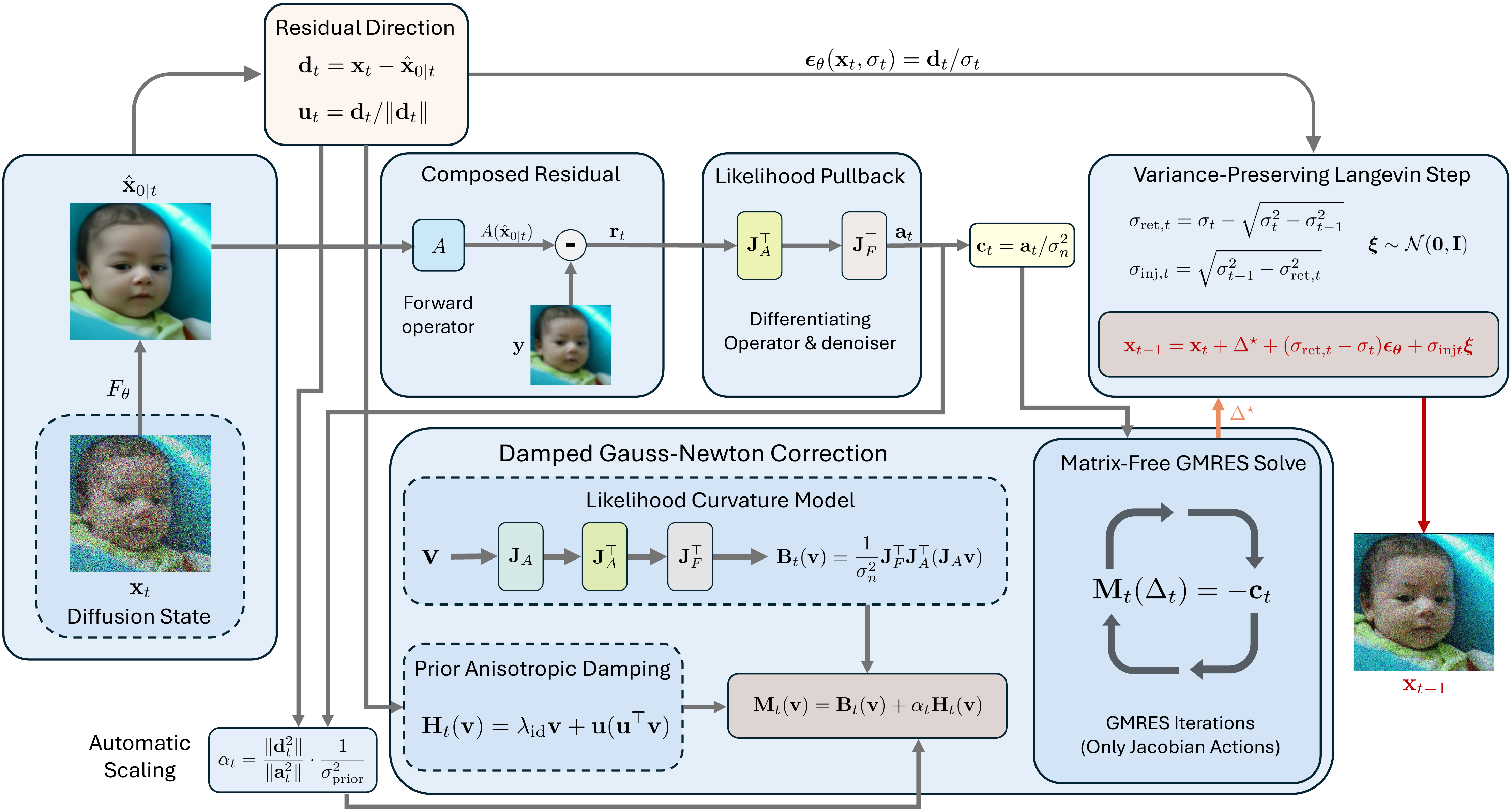}
  }
\caption{\textbf{\textsc{clamp} update at one noise level.}
(1) \emph{Prior query \& residual:} from $\mathbf{x}_t$, compute $\hat{\mathbf{x}}_{0|t}$ and the composed residual by comparing $A(\hat{\mathbf{x}}_{0|t})$ to $\mathbf{y}$.
(2) \emph{Curvature-aware correction:} pull back measurement sensitivity through the denoiser to form a one-sided Gauss--Newton model with prior-aligned, automatically scaled damping, then solve the linear system with matrix-free GMRES.
(3) \emph{Schedule transition:} apply the correction in a variance-preserving Langevin step to obtain $\mathbf{x}_{t-1}$.}
\label{fig:method}
\end{figure*}

\section{Background}
\label{sec:background}\label{sec:prelim}

This section summarizes the technical prerequisites for \textsc{clamp}: (i) diffusion denoisers and noise schedules, (ii) Gauss--Newton (GN) curvature with damping via local quadratic surrogates, and (iii) matrix-free Krylov solvers based on Jacobian actions.

\subsection{Diffusion denoisers and noise schedules}
Diffusion models generate samples by iteratively denoising a sequence of noisy states \cite{ho2020ddpm,song2021sde}.
We work with a decreasing noise schedule $\{\sigma_t\}_{t=T}^{0}$ and an EDM-style denoiser interface \cite{karras2022edm}:
\begin{equation}
\hat{\mathbf{x}}_{0|t}=F_\theta(\mathbf{x}_t;\sigma_t),
\quad
\epsilon_\theta(\mathbf{x}_t,\sigma_t):=\frac{\mathbf{x}_t-\hat{\mathbf{x}}_{0|t}}{\sigma_t}.
\end{equation}
Here $\mathbf{x}_t$ is the diffusion-state iterate at noise level $\sigma_t$, and $\hat{\mathbf{x}}_{0|t}$ is the associated clean prediction used for measurement evaluation in conditional sampling.

\subsection{Gauss--Newton quadratic surrogate and damping}
Consider the Gaussian-form residual objective with residual scale $\sigma_n>0$
\begin{equation}
\Phi(\mathbf{x})=\frac{1}{2\sigma_n^2}\|\mathbf{r}(\mathbf{x})\|_2^2,
~
\mathbf{r}(\mathbf{x}):=A(\mathbf{x})-\mathbf{y},
~
\mathbf{J}:=\nabla_{\mathbf{x}}\mathbf{r}(\mathbf{x}).
\end{equation}
For an increment $\Delta$, linearize the residual $\mathbf{r}(\mathbf{x}+\Delta)\approx \mathbf{r}(\mathbf{x})+\mathbf{J}\Delta$ and substitute into $\Phi$:
\begin{align}
&\Phi(\mathbf{x}+\Delta)
\approx
\frac{1}{2\sigma_n^2}\|\mathbf{r}+\mathbf{J}\Delta\|_2^2 \nonumber
\\
=
&\Phi(\mathbf{x})
+\frac{1}{\sigma_n^2}\langle \mathbf{J}^\top\mathbf{r},\Delta\rangle
+\frac{1}{2\sigma_n^2}\Delta^\top(\mathbf{J}^\top\mathbf{J})\Delta, \label{eq:bg_gn_quadratic}  
\end{align}
where $\top$ denotes the Euclidean adjoint.
Minimizing the quadratic model \eqref{eq:bg_gn_quadratic} yields the GN normal equations \cite{nocedal2006numerical}
\begin{equation}
(\mathbf{J}^\top\mathbf{J})\,\Delta = -\mathbf{J}^\top\mathbf{r}.
\label{eq:bg_gn_normal}
\end{equation}
Damping (trust-region / Levenberg--Marquardt) stabilizes GN updates when the local model is imperfect or ill-conditioned by adding a positive semidefinite metric \cite{nocedal2006numerical,conn2000trust,hanke1997regularizinglm}:
\begin{equation}
(\mathbf{J}^\top\mathbf{J}+\lambda \mathbf{H})\,\Delta = -\mathbf{J}^\top\mathbf{r},
\quad
\mathbf{H}\succeq 0.
\label{eq:bg_lm_aniso}
\end{equation}
The curvature $\mathbf{J}^\top\mathbf{J}$ acts as a matrix-valued step size, suppressing updates along sensitive directions; the metric $\mathbf{H}$ provides (possibly anisotropic) step control.

\subsection{Matrix-free Jacobian actions and Krylov solvers}
Large-scale inverse problems rarely permit explicit formation of $\mathbf{J}$ or $\mathbf{J}^\top\mathbf{J}$.
Instead, second-order updates are computed by solving linear systems $ \mathbf{M}\Delta=\mathbf{b}$ using only matrix--vector products $\mathbf{M}(\mathbf{v})$.
These products are implemented via Jacobian actions: a Jacobian–vector product (JVP) computes $\mathbf{J}\mathbf{v}$ and a vector–Jacobian product (VJP) computes $\mathbf{J}^\top\mathbf{u}$~\cite{pearlmutter1994hessian,baydin2018autodiff}.
For linear $A$, $\mathbf{J}\mathbf{v}=A\mathbf{v}$ and $\mathbf{J}^\top\mathbf{u}=A^\top\mathbf{u}$ exactly; for differentiable nonlinear $A$, JVP/VJP can be obtained via autodiff.

When $\mathbf{M}$ is symmetric positive-(semi)definite, conjugate gradients (and truncated trust-region variants) are standard \cite{steihaug1983cgtr,conn2000trust}.
When $\mathbf{M}$ is not symmetric, one instead uses a nonsymmetric Krylov solver such as GMRES \cite{saad1986gmres,dembo1982inexact_newton,knoll2004jfnk}, which remains fully matrix-free and only requires repeated evaluations of $\mathbf{M}(\mathbf{v})$.

\section{Method}
\label{sec:Method}

\subsection{Setup, notation, and per-step structure}
\label{subsec:method_setup}

We consider inverse problems with additive Gaussian measurement noise as in \eqref{eq:intro_meas}, whose physical standard deviation is denoted by $\beta_{\mathbf{y}}$, and target posterior sampling
\begin{equation}
p(\mathbf{x}_0\mid \mathbf{y})
\;\propto\;
\exp\!\left(-\frac{1}{2\beta_{\mathbf{y}}^2}\|A(\mathbf{x}_0)-\mathbf{y}\|_2^2\right)\,p(\mathbf{x}_0).
\label{eq:posterior}
\end{equation}
In the finite-step correction below, $\sigma_n$ denotes the data-consistency scale used by \textsc{clamp} rather than the physical noise level itself. Setting $\sigma_n=\beta_{\mathbf{y}}$ recovers the physical Gaussian scaling; smaller values strengthen the local measurement-consistency correction and can compensate for finite-step, denoiser, and linearization effects.
We access the prior through a pretrained diffusion model with a decreasing noise schedule
$\{\sigma_t\}_{t=T}^{0}$, $\sigma_{t-1}<\sigma_t$, and an EDM/Tweedie denoiser interface
\begin{equation}
\hat{\mathbf{x}}_{0|t} = F_\theta(\mathbf{x}_t;\sigma_t),
~
\boldsymbol{\epsilon}_\theta(\mathbf{x}_t,\sigma_t) := \frac{\mathbf{x}_t-\hat{\mathbf{x}}_{0|t}}{\sigma_t}.
\label{eq:tweedie_eps_def}
\end{equation}
Here $\mathbf{x}_t\in\mathbb{R}^d$ denotes the diffusion-state iterate at noise level $\sigma_t$.

\paragraph{Per-step structure.}
At each noise level $\sigma_t$, \textsc{clamp} performs:
(i) a likelihood correction $\Delta_t$ computed in diffusion-state coordinates by solving a damped linear system, and
(ii) a prior propagation step that advances $\mathbf{x}_t\mapsto \mathbf{x}_{t-1}$ while preserving the schedule variance.
Algorithm~\ref{alg:CLAMP} in Appendix~\ref{sec:appendix_algo} provides the full procedure; below we derive each component.

\subsection{Overview of the CLAMP correction}

We first summarize one \textsc{clamp} update before deriving its components. At noise level $\sigma_t$, \textsc{clamp} queries the denoiser and evaluates the measurement residual on the denoised prediction,
\[
    \hat{\mathbf{x}}_{0|t} = F_\theta(\mathbf{x}_t;\sigma_t),
    \qquad
    \mathbf{r}_t = A(\hat{\mathbf{x}}_{0|t}) - \mathbf{y}.
\]
Let $\mathbf{J}_A$ be the Jacobian of $A$ at $\hat{\mathbf{x}}_{0|t}$ and $\mathbf{J}_F$ the Jacobian of $F_\theta(\cdot;\sigma_t)$ at $\mathbf{x}_t$. The likelihood signal is pulled back to diffusion-state coordinates as
\[
    \mathbf{c}_t
    =
    \frac{1}{\sigma_n^2} \mathbf{J}_F^\top \mathbf{J}_A^\top \mathbf{r}_t .
\]
Instead of taking a scalar-guided step along $-\mathbf{c}_t$, \textsc{clamp} computes a matrix-scaled correction $\Delta_t$ by solving
\[
    \mathbf{M}_t(\Delta_t) = -\mathbf{c}_t,
    \qquad
    \mathbf{M}_t(\mathbf{v}) = \mathbf{B}_t(\mathbf{v}) + \alpha_t \mathbf{H}_t(\mathbf{v}),
\]
where
\[
    \mathbf{B}_t(\mathbf{v})
    =
    \frac{1}{\sigma_n^2}
    \mathbf{J}_F^\top \mathbf{J}_A^\top (\mathbf{J}_A \mathbf{v}),
    \qquad
    \mathbf{H}_t(\mathbf{v})
    =
    \lambda_{\mathrm{id}}\mathbf{v} + \mathbf{u}_t(\mathbf{u}_t^\top \mathbf{v}),
\]
\[
    \mathbf{d}_t 
    =
    \mathbf{x}_t - \hat{\mathbf{x}}_{0|t},
    \qquad
    \mathbf{u}_t 
    =
    \frac{\mathbf{d}_t}{\|\mathbf{d}_t\|_2}.
\]
Here $\mathbf{B}_t$ is the one-sided denoiser-pullback curvature operator, and $\mathbf{H}_t$ is a rank-one damping metric aligned with the denoiser residual direction. The system is solved matrix-free with a fixed GMRES budget.

After the likelihood correction, \textsc{clamp} advances the diffusion chain using a variance-preserving Langevin transition,
\[
    \mathbf{x}_{t-1}
    =
    \mathbf{x}_t
    +
    \Delta_t
    +
    (\sigma_{\mathrm{ret},t}-\sigma_t)
    \boldsymbol{\epsilon}_\theta(\mathbf{x}_t,\sigma_t)
    +
    \sigma_{\mathrm{inj},t}\boldsymbol{\xi},
\]
where $\boldsymbol{\xi} \sim \mathcal{N}(\mathbf{0},\mathbf{I})$. 

Thus, the core of \textsc{clamp} is a local, repeatedly re-linearized correction: $\mathbf{J}_F^\top \mathbf{J}_A^\top \mathbf{r}_t$ provides the coordinate-correct likelihood pullback, $\mathbf{B}_t$ supplies curvature scaling, $\mathbf{H}_t$ stabilizes the step, and the Langevin transition propagates the corrected sample to the next noise level.

\subsection{Why scalar likelihood guidance is brittle}
\label{subsec:scalar_barrier}

Many diffusion posterior samplers scale a likelihood gradient in $\mathbf{x}_t$-space by a scalar guidance weight.
When the local geometry is stiff (ill-conditioned), stable scalar steps are constrained by the largest-curvature direction, which can
leave insufficient progress along lower-curvature (“flatter”) directions.

\begin{proposition}[Scalar-step barrier under stiff local geometry]
\label{prop:scalar_barrier}
Fix $t$ and suppose a local objective in the increment $\Delta$ admits a quadratic model
$
\mathcal{J}(\Delta)\approx \mathcal{J}(\mathbf{0})
+\langle \mathbf{g}_t,\Delta\rangle
+\frac{1}{2}\Delta^\top\mathbf{Q}_t\Delta
$
with $\mathbf{Q}_t\succ 0$.
A sufficient condition for global stability of the scalar-guided update
$\Delta_{\rm sc}=-\eta \mathbf{g}_t$ on this quadratic model is
\begin{equation}
\eta\le \frac{2}{\lambda_{\max}(\mathbf{Q}_t)}.
\end{equation}
Let $\Delta^\star=-\mathbf{Q}_t^{-1}\mathbf{g}_t$ denote the quadratic minimizer.
In the eigenbasis $\mathbf{Q}_t\mathbf{v}_i=\lambda_i\mathbf{v}_i$, for any $i$ with $\langle \mathbf{v}_i,\mathbf{g}_t\rangle\neq 0$,
\begin{equation}
\frac{|\langle \mathbf{v}_i,\Delta_{\rm sc}\rangle|}{|\langle \mathbf{v}_i,\Delta^\star\rangle|}
= \eta\lambda_i
\le \frac{2\lambda_i}{\lambda_{\max}(\mathbf{Q}_t)}.
\end{equation}
In particular, along the flattest direction $\,\mathbf{v}_{\min}$ (associated with $\lambda_{\min}(\mathbf{Q}_t)$),
the stability constraint implies $\eta\lambda_{\min}(\mathbf{Q}_t)\le 2/\kappa(\mathbf{Q}_t)$.
Thus when $\kappa(\mathbf{Q}_t)\gg 1$, any globally stable scalar step yields only
$O(\lambda_i/\lambda_{\max})$ relative progress in low-curvature directions, and at most $O(1/\kappa)$
relative progress along the flattest direction.
\end{proposition}

Proposition~\ref{prop:scalar_barrier} motivates using a matrix-scaled correction that allocates step sizes direction-wise via curvature.

\subsection{Stepwise surrogate objective in diffusion-state coordinates}
\label{subsec:surrogate}

At each noise level $\sigma_t$, \textsc{clamp} computes a likelihood correction $\Delta_t$ in diffusion-state coordinates by solving for the mode of a local surrogate that combines (i) the stepwise likelihood and (ii) a Gaussian surrogate over increments around $\mathbf{x}_t$.
Concretely, we define a local Gaussian surrogate on the increment $\Delta$:
\begin{equation}
p_t(\mathbf{x}_t+\Delta\mid \mathbf{x}_t)
\;\propto\;
\exp\!\left(-\frac{\alpha_t}{2}\,\|\Delta\|_{\mathbf{H}_t}^2\right),
\label{eq:pb_local_prior}
\end{equation}
where $\mathbf{H}_t\succeq 0$ is a (possibly anisotropic) damping metric and $\alpha_t>0$ sets its strength (both specified in Sec.~\ref{subsec:metric_alpha}).
Combining this with the stepwise likelihood $\Phi_t$ (defined in Sec.~\ref{subsec:likelihood_pullback}) yields the stepwise surrogate conditional density
\begin{equation}
\tilde p_t(\mathbf{x}\mid \mathbf{y},\mathbf{x}_t)
\;\propto\;
\exp\!\big(-\Phi_t(\mathbf{x})\big)\,p_t(\mathbf{x}\mid \mathbf{x}_t),
\label{eq:pb_local_post}
\end{equation}
whose negative log (up to an additive constant) is
\begin{equation}
\mathcal{E}_t(\mathbf{x})
=
\Phi_t(\mathbf{x})
+
\frac{\alpha_t}{2}\|\mathbf{x}-\mathbf{x}_t\|_{\mathbf{H}_t}^2.
\label{eq:pb_energy}
\end{equation}
We compute $\Delta_t$ by minimizing a quadratic surrogate of $\mathcal{E}_t(\mathbf{x}_t+\Delta)$ derived below.
The quadratic term plays the role of a trust-region/Levenberg--Marquardt stabilizer \cite{nocedal2006numerical,conn2000trust} for step control; it is not intended to match the exact diffusion transition density.

\subsection{Denoiser-pulled-back likelihood geometry}
\label{subsec:likelihood_pullback}

The measurement residual is defined on a clean variable (the input of $A$), whereas the algorithm updates the noisy diffusion state $\mathbf{x}_t$.
We therefore evaluate measurement residuals on the denoised prediction and define the composed residual and stepwise likelihood as functions of the diffusion-state argument $\mathbf{x}$:
\begin{align}
\mathbf{r}_t(\mathbf{x})
:= A\!\big(F_\theta(\mathbf{x};\sigma_t)\big)-\mathbf{y},
~
\Phi_t(\mathbf{x})
:=\frac{1}{2\sigma_n^2}\|\mathbf{r}_t(\mathbf{x})\|_2^2.
\label{eq:rt_phi_def}
\end{align}
Let $\mathbf{z}_t := F_\theta(\mathbf{x}_t;\sigma_t)$ and $\mathbf{r}_t:=\mathbf{r}_t(\mathbf{x}_t)=A(\mathbf{z}_t)-\mathbf{y}$.
Define Jacobians at the current evaluation point:
\begin{align}
&\mathbf{J}_A := \nabla_{\mathbf{z}}A(\mathbf{z})\big|_{\mathbf{z}=\mathbf{z}_t},
\quad
\mathbf{J}_F := \nabla_{\mathbf{x}}F_\theta(\mathbf{x};\sigma_t)\big|_{\mathbf{x}=\mathbf{x}_t}\\
&\mathbf{J}_t := \nabla_{\mathbf{x}}\mathbf{r}_t(\mathbf{x})\big|_{\mathbf{x}=\mathbf{x}_t}=\mathbf{J}_A\mathbf{J}_F.
\label{eq:Jt_chainrule}
\end{align}

\begin{proposition}[Coordinate-correct guidance requires denoiser pullback]
\label{prop:pullback}
Let $\tilde\Phi(\mathbf{z}) := \tfrac{1}{2\sigma_n^2}\|A(\mathbf{z})-\mathbf{y}\|_2^2$ and $\Phi_t(\mathbf{x})=\tilde\Phi(F_\theta(\mathbf{x};\sigma_t))$.
Then the first-order expansion at $\mathbf{x}_t$ is
\begin{equation}
\Phi_t(\mathbf{x}_t+\Delta)
=
\Phi_t(\mathbf{x}_t)
+
\left\langle \mathbf{J}_F^{\!\top}\nabla_{\mathbf{z}}\tilde\Phi(\mathbf{z})\big|_{\mathbf{z}=\mathbf{z}_t},\Delta\right\rangle
+ o(\|\Delta\|_2),
\end{equation}
so coordinate-correct steepest descent in $\mathbf{x}_t$-space is proportional to $-\mathbf{J}_F^\top\nabla_{\mathbf{z}}\tilde\Phi(\mathbf{z}_t)$.
Moreover, under a Gauss--Newton/Fisher approximation (dropping second-order residual terms), the diffusion-state curvature is
$
\frac{1}{\sigma_n^2}\mathbf{J}_F^{\!\top}\mathbf{J}_A^{\!\top}\mathbf{J}_A\mathbf{J}_F.
$
\end{proposition}

Proposition~\ref{prop:pullback} motivates differentiating through the denoiser (via $\mathbf{J}_F^\top$): likelihood derivatives in clean coordinates do not, in general, provide correctly scaled directions for updating $\mathbf{x}_t$ when $\mathbf{J}_F$ is anisotropic.

\subsection{Penalized Gauss--Newton correction and explicit approximations}
\label{subsec:curvature_system}

Building on Sec.~\ref{subsec:likelihood_pullback}, we compute the likelihood correction $\Delta_t$ by minimizing a quadratic surrogate of
$\mathcal{E}_t(\mathbf{x}_t+\Delta)$ in \eqref{eq:pb_energy}, using the composed residual $\mathbf{r}_t(\mathbf{x})$ and Jacobian
$\mathbf{J}_t=\mathbf{J}_A\mathbf{J}_F$ from \eqref{eq:Jt_chainrule}.

\paragraph{(A1) Local linearization and penalized Gauss--Newton (standard).}
Using the first-order residual model
\begin{equation}
\mathbf{r}_t(\mathbf{x}_t+\Delta)\approx \mathbf{r}_t + \mathbf{J}_t\Delta,
\label{eq:linearize_rt}
\end{equation}
substituting into \eqref{eq:pb_energy} yields the quadratic approximation
\begin{equation}
\mathcal{E}_t(\mathbf{x}_t+\Delta)
\;\approx\;
\frac{1}{2\sigma_n^2}\|\mathbf{r}_t+\mathbf{J}_t\Delta\|_2^2
+
\frac{\alpha_t}{2}\Delta^\top \mathbf{H}_t \Delta.
\label{eq:quad_energy}
\end{equation}
Its stationary point solves the ideal penalized Gauss--Newton system
\begin{equation}
\left(\frac{1}{\sigma_n^2}\mathbf{J}_t^\top \mathbf{J}_t + \alpha_t\mathbf{H}_t\right)\Delta
= -\frac{1}{\sigma_n^2}\mathbf{J}_t^\top\mathbf{r}_t.
\label{eq:ideal_normal_eq}
\end{equation}
With $\mathbf{J}_t=\mathbf{J}_A\mathbf{J}_F$, the curvature term is $\mathbf{J}_t^\top \mathbf{J}_t = \mathbf{J}_F^\top \mathbf{J}_A^\top \mathbf{J}_A \mathbf{J}_F$, which requires forward denoiser-Jacobian products
of the form $\mathbf{J}_F\mathbf{v}$ inside curvature matvecs.

\paragraph{(A2) One-sided curvature approximation (\textsc{clamp} design choice).}
To avoid forward denoiser-Jacobian products inside the curvature, \textsc{clamp} adopts a one-sided Gauss--Newton/Fisher model that
approximates the forward denoiser perturbation as $\mathbf{J}_F\mathbf{v}\approx \mathbf{v}$ only when forming the curvature operator, while
retaining the pullback $\mathbf{J}_F^\top$ in the adjoint pass so that update directions remain coordinate-correct under anisotropic $\mathbf{J}_F$
(Prop.~\ref{prop:pullback}). This approximation is motivated by the local translation heuristic
$F_\theta(\mathbf{x}_t+\Delta;\sigma_t)\approx F_\theta(\mathbf{x}_t;\sigma_t)+\Delta$ for small $\Delta$.

A useful sanity check is that the diffusion-state gradient of the stepwise likelihood is exact under this construction:
\[
\nabla_{\mathbf{x}}\Phi_t(\mathbf{x}_t)
=
\frac{1}{\sigma_n^2}\,\mathbf{J}_F^\top \mathbf{J}_A^\top \mathbf{r}_t,
\]
whereas only the curvature is approximated. In particular, if we define the exact GN/Fisher curvature operator
$\mathbf{B}_t^{\mathrm{exact}}(\mathbf{v}) := \sigma_n^{-2}\mathbf{J}_F^\top \mathbf{J}_A^\top \mathbf{J}_A \mathbf{J}_F\mathbf{v}$, then the one-sided operator $\mathbf{B}_t$
defined below satisfies
\[
\big(\mathbf{B}_t^{\mathrm{exact}}-\mathbf{B}_t\big)(\mathbf{v})
=
\frac{1}{\sigma_n^2}\,\mathbf{J}_F^\top \mathbf{J}_A^\top \mathbf{J}_A (\mathbf{J}_F-\mathbf{I})\mathbf{v},
\]
so the curvature discrepancy is controlled by the deviation of the forward action from identity (i.e., $\|\mathbf{J}_F-I\|$), while eliminating $\mathbf{J}_F\mathbf{v}$
products in curvature matvecs.

\paragraph{Asymmetric Jacobian actions.}
Concretely, we use the asymmetric Jacobian actions
\begin{equation}
\tilde{\mathbf{J}}_t(\mathbf{v}) := \mathbf{J}_A\mathbf{v},
\qquad
\tilde{\mathbf{J}}_t^{\!\top}(\mathbf{u}) := \mathbf{J}_F^{\!\top}\mathbf{J}_A^{\!\top}\mathbf{u},
\label{eq:one_sided_Jt}
\end{equation}
which induce the matrix-free operators
\begin{align}
\mathbf{B}_t(\mathbf{v})
&:=
\frac{1}{\sigma_n^2}\,\tilde{\mathbf{J}}_t^{\!\top}\!\big(\tilde{\mathbf{J}}_t(\mathbf{v})\big)
=
\frac{1}{\sigma_n^2}\,\mathbf{J}_F^{\!\top}\mathbf{J}_A^{\!\top}\!\big(\mathbf{J}_A\mathbf{v}\big),
\label{eq:pb_Bv}\\
\mathbf{c}_t
&:=
\frac{1}{\sigma_n^2}\,\tilde{\mathbf{J}}_t^{\!\top}\mathbf{r}_t
=
\frac{1}{\sigma_n^2}\,\mathbf{J}_F^{\!\top}\mathbf{J}_A^{\!\top}\mathbf{r}_t.
\label{eq:pb_ct}
\end{align}
Compared to the exact GN/Fisher curvature $\sigma_n^{-2}\mathbf{J}_F^\top \mathbf{J}_A^\top \mathbf{J}_A \mathbf{J}_F$, $\mathbf{B}_t$ drops the rightmost $\mathbf{J}_F$ factor.
Since $\tilde{\mathbf{J}}_t(\cdot)$ and $\tilde{\mathbf{J}}_t^\top(\cdot)$ are not adjoints of a single linear map, $\mathbf{B}_t$ (and thus $\mathbf{M}_t$ below)
need not be self-adjoint, motivating a nonsymmetric Krylov solver (GMRES).

\paragraph{Linear system to solve.}
The correction $\Delta_t$ is defined as the solution of the damped system
\begin{equation}
\mathbf{M}_t(\Delta_t) = -\mathbf{c}_t,
\qquad
\mathbf{M}_t(\mathbf{v}) := \mathbf{B}_t(\mathbf{v}) + \alpha_t \mathbf{H}_t(\mathbf{v}),
\label{eq:gmres_system}
\end{equation}
which we solve with GMRES (Sec.~\ref{subsec:gmres}) using only matrix--vector products.

\subsection{Prior-aligned anisotropic damping and diffusion-calibrated regularization}
\label{subsec:metric_alpha}
\paragraph{Rank-one anisotropic metric.}
Define the denoiser residual and its unit direction
\begin{equation}
\mathbf{d}_t := \mathbf{x}_t - \hat{\mathbf{x}}_{0|t},
~
\mathbf{u}_t := \frac{\mathbf{d}_t}{\|\mathbf{d}_t\|_2}.
\label{eq:pb_ui}
\end{equation}
We set
\begin{equation}
\mathbf{H}_t := \lambda_{\mathrm{id}}\mathbf{I} + \mathbf{u}_t\mathbf{u}_t^\top,
~
\mathbf{H}_t\mathbf{v} = \lambda_{\mathrm{id}}\mathbf{v} + \mathbf{u}_t(\mathbf{u}_t^\top \mathbf{v}),
\label{eq:pb_H}
\end{equation}
where $\lambda_{\mathrm{id}}\ge 0$ provides baseline isotropic damping and prevents degeneracy.
This metric penalizes motion along the denoiser residual direction $\mathbf{u}_t$ more strongly than orthogonal directions, yielding explicit step control aligned with a diffusion-prior-indicated direction.
\paragraph{Diffusion-calibrated regularization strength.}
Define the (unscaled) backprojected residual in diffusion coordinates
\begin{equation}
\mathbf{a}_t := \sigma_n^2\,\mathbf{c}_t = \mathbf{J}_F^{\!\top}\mathbf{J}_A^{\!\top}\mathbf{r}_t.
\label{eq:CLAMP_a_def}
\end{equation}
We set
\begin{equation}
\alpha_t
:=
\frac{\|\mathbf{d}_t\|_2}{\|\mathbf{a}_t\|_2}
\cdot \frac{1}{\sigma_{\mathrm{prior}}^2}
\label{eq:CLAMP_alpha}
\end{equation}
where $\sigma_{\mathrm{prior}}=\max({\sigma_{t-1},\sigma_{\min}^{+}})$. The ratio $\|\mathbf{d}_t\|_2/\|\mathbf{a}_t\|_2$ provides an automatic scale match between the prior-indicated magnitude
(the denoiser residual $\mathbf{d}_t=\mathbf{x}_t-\hat{\mathbf{x}}_{0|t}$) and the likelihood signal seen in diffusion coordinates (the backprojection $\mathbf{a}_t$):
when the likelihood pullback is strong, $\alpha_t$ decreases to avoid over-damping; when it is weak, $\alpha_t$ increases to keep corrections conservative.
The factor $\sigma_{\mathrm{prior}}^{-2}$ ties the damping to the remaining diffusion noise level, yielding comparable step control across $t$.
Finally, defining $\mathbf{a}_t=\sigma_n^2\mathbf{c}_t$ removes explicit $\sigma_n^{-2}$ dependence from $\alpha_t$, so increasing $\sigma_n$
naturally weakens likelihood corrections through $\mathbf{B}_t$ and $\mathbf{c}_t$ without retuning.

\subsection{Matrix-free GMRES and Jacobian-action interface}
\label{subsec:gmres}

We never form $\mathbf{J}_A$, $\mathbf{J}_F$, $\mathbf{B}_t$, or $\mathbf{H}_t$ explicitly. Instead, we assume access to matrix-free Jacobian actions: a JVP through $A$ to compute $\mathbf{J}_A\mathbf{v}$ (for linear $A$, this equals $A\mathbf{v}$ exactly), a VJP through $A$ (or an explicit adjoint when available) to compute $\mathbf{J}_A^\top \mathbf{u}$, and a VJP through the denoiser to compute $\mathbf{J}_F^\top \mathbf{u}$. With these primitives, we evaluate
\begin{align}
&\mathbf{B}_t(\mathbf{v})
=
\sigma_n^{-2}\,\mathbf{J}_F^\top\!\Big(\mathbf{J}_A^\top\!\big(\mathbf{J}_A\mathbf{v}\big)\Big),
~
\\
&\mathbf{c}_t
=
\sigma_n^{-2}\,\mathbf{J}_F^\top\!\big(\mathbf{J}_A^\top \mathbf{r}_t\big),~\mathbf{H}_t(\mathbf{v})=\lambda_{\mathrm{id}}\mathbf{v}+\mathbf{u}_t(\mathbf{u}_t^\top\mathbf{v}).  
\end{align}
For differentiable nonlinear $A$, these Jacobian actions can be obtained via automatic differentiation \cite{baydin2018autodiff}; if automatic differentiation is unavailable, one may approximate $\mathbf{J}_A\mathbf{v}$ using standard numerical directional-derivative schemes (e.g., finite differences)~\cite{knoll2004jfnk}.

We solve \eqref{eq:gmres_system} using GMRES \cite{saad1986gmres}, which applies to nonsymmetric operators and is fully matrix-free (see Algorithm~\ref{alg:CLAMP_gmres} in Appendix~\ref{sec:appendix_algo} for details).
We run GMRES for a fixed iteration budget $K$ \cite{dembo1982inexact_newton,knoll2004jfnk}.
A fixed Krylov budget is appropriate because our operator is derived from a local linearization and a one-sided curvature approximation: over-solving can amplify model mismatch, whereas moderate $K$ yields stable updates.

\paragraph{Per-step computational cost.}
Each GMRES iteration requires one application of $\mathbf{M}_t(\cdot)$, i.e., one JVP through $A$, one VJP through $A$, one VJP through $F_\theta$, and one rank-one metric application.

\subsection{Variance-preserving Langevin propagation}
\label{subsec:vp_langevin_step}

After computing $\Delta_t$, we propagate from $\sigma_t$ to $\sigma_{t-1}$.
A deterministic EDM-style probability-flow update would be
\begin{equation}
\mathbf{x}_{t-1}^{\mathrm{ODE}}
=
\mathbf{x}_t + \Delta_t + (\sigma_{t-1}-\sigma_t)\,\boldsymbol{\epsilon}_\theta(\mathbf{x}_t,\sigma_t).
\label{eq:edm_ode_baseline}
\end{equation}
In posterior sampling, the additional drift $\Delta_t$ can make purely deterministic propagation sensitive to discretization and model mismatch.
We instead use a stochastic transition that preserves the schedule variance while maintaining stochastic exploration:
\begin{equation}
\mathbf{x}_{t-1}
=
\mathbf{x}_t + \Delta_t
+(\sigma_{\mathrm{ret},t}-\sigma_t)\,\boldsymbol{\epsilon}_\theta(\mathbf{x}_t,\sigma_t)
+\sigma_{\mathrm{inj},t}\,\boldsymbol{\xi},
\label{eq:vp_step}
\end{equation}
where $\boldsymbol{\xi}\sim\mathcal{N}(\mathbf{0},\mathbf{I})$. We enforce variance matching with the schedule:
\begin{equation}
\sigma_{t-1}^2 = \sigma_{\mathrm{ret},t}^2+\sigma_{\mathrm{inj},t}^2.
\label{eq:vp_variance}
\end{equation}
Using $\boldsymbol{\epsilon}_\theta=(\mathbf{x}_t-\hat{\mathbf{x}}_{0|t})/\sigma_t$, the drift can be rewritten as
\begin{align}
\label{eq:vp_drift}
(\sigma_{\mathrm{ret},t}-\sigma_t)\,\boldsymbol{\epsilon}_\theta
&=
\lambda_t\big(F_\theta(\mathbf{x}_t;\sigma_t)-\mathbf{x}_t\big), \\
\lambda_t&:=1-\frac{\sigma_{\mathrm{ret},t}}{\sigma_t},
\end{align}
so \eqref{eq:vp_step} is a single annealed-Langevin (unadjusted Langevin algorithm; ULA) step~\cite{durmus2017ula}:
move a fraction $\lambda_t$ toward the denoiser output and inject Gaussian noise.

Many stochastic diffusion samplers introduce an explicit stochasticity parameter to set the injected-noise term (with the complementary coefficient fixed by the schedule constraint), e.g., $\eta$ in DDIM \cite{song2021ddim} and $(\eta,\eta_b)$ in DDRM \cite{kawar2022ddrm}.
\textsc{clamp} instead determines $(\sigma_{\mathrm{ret},t},\sigma_{\mathrm{inj},t})$ from the schedule by imposing a fluctuation--dissipation coupling at the geometric reference
$\sigma_{\mathrm{ref}}^2:=\sigma_t\sigma_{\mathrm{ret},t}$, yielding a closed form.

\begin{proposition}[Drift/noise split under variance preservation]
\label{prop:vp_closed_form}
Assume (i) variance matching \eqref{eq:vp_variance} and (ii) a Langevin fluctuation--dissipation coupling at
$\sigma_{\mathrm{ref}}^2:=\sigma_t\sigma_{\mathrm{ret},t}$, i.e.
$\sigma_{\mathrm{inj},t}^2 = 2\lambda_t\sigma_{\mathrm{ref}}^2$ with $\lambda_t$ from \eqref{eq:vp_drift}.
Then $(\sigma_{\mathrm{ret},t},\sigma_{\mathrm{inj},t})$ is uniquely determined:
\begin{equation}
\sigma_{\mathrm{ret},t}
=
\sigma_t-\sqrt{\sigma_t^2-\sigma_{t-1}^2},
~
\sigma_{\mathrm{inj},t}
=
\sqrt{\sigma_{t-1}^2-\sigma_{\mathrm{ret},t}^2}.
\label{eq:vp_closed_form}
\end{equation}
\end{proposition}

\subsection{Latent \textsc{clamp}}
\label{subsec:latent}

\begin{table*}[!t]
\centering
\setlength{\tabcolsep}{2.5pt}
\renewcommand{\arraystretch}{0.92}
\caption{Quantitative comparison on FFHQ and ImageNet across inverse-problem tasks in pixel setting. We report average PSNR/SSIM (higher is better), LPIPS (lower is better), and run-time (lower is better) over 100 validation images. Best and second-best scores are highlighted in \textbf{bold} and \underline{underlined}, respectively.}
\label{tab:main_table_pixel}
\begin{tabular}{l!{\vrule}l!{\vrule}cccc!{\vrule}cccc}
\toprule
\multirow{2}{*}{\textbf{Task}} & \multirow{2}{*}{\textbf{Method}} &
\multicolumn{4}{c|}{\textbf{FFHQ}} & \multicolumn{4}{c}{\textbf{ImageNet}} \\
& &
\textbf{PSNR $\uparrow$} & \textbf{SSIM $\uparrow$} & \textbf{LPIPS $\downarrow$} & \textbf{Run-time (s)} &
\textbf{PSNR $\uparrow$} & \textbf{SSIM $\uparrow$} & \textbf{LPIPS $\downarrow$} & \textbf{Run-time (s)} \\
\midrule

\multirow{5}{*}{Super resolution 4$\times$} &
Ours   & \textbf{29.515} & \textbf{0.841} & \textbf{0.219} & \underline{6.743} & \textbf{26.981} & \textbf{0.742} & \textbf{0.290} & \underline{23.587} \\
& DAPS   & 28.619 & 0.764 & 0.262 & 28.257 & 25.512 & 0.636 & 0.374 & 65.580 \\
& SITCOM & \underline{29.153} & \underline{0.826} & 0.231 & 15.755 & \underline{26.519} & \underline{0.716} & \underline{0.309} & 60.778 \\
& DMPlug & 28.637 & 0.797 & 0.253 & 118.255 & 25.091 & 0.660 & 0.347 & 238.433 \\
& DCDP   & 27.611 & 0.785 & \underline{0.225} & \textbf{4.719} & 24.358 & 0.640 & 0.369 & \textbf{9.190} \\
\midrule

\multirow{5}{*}{Box inpainting} &
Ours   & \textbf{25.495} & \textbf{0.882} & \textbf{0.125} & \textbf{6.175} & \textbf{21.547} & \textbf{0.825} & \textbf{0.188} & \underline{23.427} \\
& DAPS   & 24.285 & 0.746 & 0.222 & 22.352 & \underline{21.234} & 0.720 & 0.270 & 57.108 \\
& SITCOM & \underline{25.275} & \underline{0.831} & \underline{0.179} & 17.697 & 20.698 & \underline{0.726} & \underline{0.241} & 66.709 \\
& DMPlug & 23.526 & 0.790 & 0.281 & 103.76 & 19.455 & 0.667 & 0.384 & 253.946 \\
& DCDP   & 21.721 & 0.766 & 0.219 & \underline{8.147} & 17.741 & 0.662 & 0.301 & \textbf{17.214} \\
\midrule

\multirow{5}{*}{Random inpainting} &
Ours   & \textbf{32.955} & \textbf{0.918} & \textbf{0.140} & \textbf{6.171} & \textbf{30.215} & \textbf{0.866} & \textbf{0.159} & \underline{23.419} \\
& DAPS   & 30.078 & 0.791 & 0.214 & 21.576 & 27.350 & 0.721 & 0.251 & 56.932 \\
& SITCOM & \underline{31.431} & \underline{0.871} & \underline{0.173} & 25.382 & \underline{29.008} & \underline{0.821} & \underline{0.183} & 95.962 \\
& DMPlug & 30.825 & 0.865 & 0.231 & 122.143 & 27.023 & 0.739 & 0.317 & 248.146 \\
& DCDP   & 27.240 & 0.788 & 0.229 & \underline{8.130} & 23.284 & 0.606 & 0.355 & \textbf{17.518} \\
\midrule

\multirow{5}{*}{Gaussian deblurring} &
Ours   & \underline{29.010} & \textbf{0.827} & \textbf{0.234} & \underline{6.343} & \underline{26.388} & \underline{0.710} & \underline{0.333} & \underline{23.924} \\
& DAPS   & 28.762 & 0.767 & 0.254 & 37.730 & 25.932 & 0.655 & 0.355 & 74.984 \\
& SITCOM & \textbf{29.359} & \underline{0.826} & \underline{0.236} & 25.491 & \textbf{26.741} & \textbf{0.719} & \textbf{0.315} & 95.161 \\
& DMPlug & 27.95 & 0.770 & 0.278 & 181.695 & 24.082 & 0.613 & 0.389 & 313.680 \\
& DCDP   & 27.553 & 0.761 & 0.245 & \textbf{4.934} & 23.499 & 0.542 & 0.449 & \textbf{9.573} \\
\midrule

\multirow{5}{*}{Motion deblurring} &
Ours   & \textbf{31.641} & \textbf{0.880} & \underline{0.183} & \underline{6.322} & \underline{29.141} & \textbf{0.812} & \underline{0.243} & \underline{23.805} \\
& DAPS   & 30.954 & 0.823 & 0.202 & 38.606 & 28.623 & 0.764 & 0.246 & 78.000 \\
& SITCOM & \underline{31.501} & \underline{0.866} & \textbf{0.167} & 26.184 & \textbf{29.321} & \underline{0.811} & \textbf{0.213} & 95.529 \\
& DMPlug & 29.272 & 0.818 & 0.260 & 152.699 & 25.211 & 0.660 & 0.365 & 284.413 \\
& DCDP   & 25.468 & 0.566 & 0.365 & \textbf{5.007} & 17.771 & 0.248 & 0.574 & \textbf{9.488} \\
\midrule

\multirow{4}{*}{Phase retrieval} &
Ours   & \textbf{30.233} & \textbf{0.854} & \textbf{0.192} & \underline{28.660} & \underline{19.680} & \underline{0.459} & \underline{0.478} & \underline{104.628} \\
& DAPS   & \underline{30.103} & \underline{0.796} & \underline{0.208} & 98.054 & \textbf{22.354} & \textbf{0.519} & \textbf{0.402} & 241.113 \\
& SITCOM & 28.782 & 0.791 & 0.239 & \textbf{24.584} & 18.481 & 0.383 & 0.524 & \textbf{94.187} \\
& DCDP   & 24.036 & 0.677 & 0.310 & 102.386 & 15.953 & 0.283 & 0.595 & 195.193 \\
\midrule

\multirow{5}{*}{Nonlinear deblurring} &
Ours   & \textbf{29.961} & \textbf{0.856} & \textbf{0.166} & \underline{33.946} & \underline{28.036} & \textbf{0.788} & \textbf{0.212} & \textbf{82.427} \\
& DAPS   & 28.868 & 0.780 & 0.223 & 755.725 & 27.537 & 0.734 & 0.266 & 884.119 \\
& SITCOM & \underline{29.519} & \underline{0.812} & 0.207 & \textbf{27.167} & \textbf{28.097} & \underline{0.774} & \underline{0.221} & \underline{97.029} \\
& DMPlug & 28.298 & 0.811 & 0.249 & 291.221 & 25.086 & 0.679 & 0.317 & 668.641 \\
& DCDP   & 27.879 & 0.795 & \underline{0.204} & 269.786 & 25.726 & 0.664 & 0.280 & 367.695 \\
\midrule

\multirow{4}{*}{High dynamic range} &
Ours   & \textbf{29.488} & \textbf{0.891} & \textbf{0.133} & \textbf{26.610} & \textbf{28.804} & \textbf{0.886} & \textbf{0.136} & \textbf{104.38} \\
& DAPS   & 27.149 & \underline{0.834} & \underline{0.196} & 76.724 & \underline{26.568} & \underline{0.819} & \underline{0.198} & 216.701 \\
& SITCOM & \underline{27.205} & 0.777 & 0.226 & \underline{33.887} & 26.449 & 0.774 & 0.222 & \underline{129.681} \\
& DMPlug & 25.507 & 0.777 & 0.264 & 221.418 & 23.244 & 0.715 & 0.308 & 564.501 \\

\bottomrule
\end{tabular}
\end{table*}

To use a latent diffusion prior, we run \textsc{clamp} on latent states $\mathbf{z}_t$ with a differentiable decoder $D:\mathbb{R}^{d_z}\!\to\!\mathbb{R}^{d}$ that maps latents to pixels.
Since measurements are defined in pixel space, it is natural to write the forward model through the composition
\begin{equation}
\tilde A := A\circ D,
\qquad
\mathbf{y} = \tilde A(\mathbf{z}_0) + \mathbf{n},\;\; \mathbf{n}\sim\mathcal{N}(\mathbf{0},\beta_{\mathbf{y}}^2\mathbf{I}).
\end{equation}
At noise level $\sigma_t$, the latent denoiser produces $\hat{\mathbf{z}}_{0|t}=F_\theta(\mathbf{z}_t;\sigma_t)$ and the residual is evaluated as
$\mathbf{r}_t=\tilde A(\hat{\mathbf{z}}_{0|t})-\mathbf{y}$.
The latent-space algorithm is otherwise identical to the pixel-space derivation: replace $A$ by $\tilde A$ and $\mathbf{x}_t$ by $\mathbf{z}_t$ everywhere, retaining the denoiser pullback, anisotropic damping, matrix-free GMRES likelihood solve, and variance-preserving Langevin propagation.
Operationally, this substitution only changes the required Jacobian actions: each $\mathrm{JVP}_A/\mathrm{VJP}_A$ call becomes a $\mathrm{JVP}_{\tilde A}/\mathrm{VJP}_{\tilde A}$ call (and thus traverses $D$), while the denoiser pullback still uses $\mathrm{VJP}_F$.
Algorithm~\ref{alg:latent_CLAMP} in Appendix~\ref{sec:appendix_algo} summarizes the resulting procedure, with final output $\mathbf{x}_0=D(\mathbf{z}_0)$.

\section{Experiments}
\label{sec:Experiments}

\subsection{Experimental setup}
We largely follow the experimental setup of DAPS~\cite{Zhang_2025_CVPR} and evaluate on eight task settings, comprising five linear and three nonlinear inverse problems. 
Implementation details and hyperparameter configurations for all baselines and our method, as well as other experimental details are provided in Appendix~\ref{appendix:exp_detail}.


\subsection{Main Results}
Tables~\ref{tab:main_table_pixel} and~\ref{tab:main_table_latent} compare our method to baselines on FFHQ and ImageNet across eight inverse-problem tasks in both pixel and latent settings. Overall, \textsc{clamp} achieves competitive or leading quantitative metrics on most tasks, and Figure~\ref{fig:main} shows reconstructions with fewer artifacts across representative tasks.

Nonlinear tasks are challenging because a nonlinear forward model $A(\cdot)$ makes the data term nonconvex and can induce rapidly varying local curvature along the sampling trajectory, so likelihood-based corrections become highly sensitive to step size and direction. This is especially evident for nonlinear deblurring: small mis-calibration of the correction can lead to instability or pronounced artifacts. In contrast, our method remains stable in both pixel and latent settings. We attribute this behavior to combining (i) curvature-aware scaling of the likelihood correction with (ii) prior-aligned regularization that constrains updates in sensitive directions, which together stabilizes the conditioning step.

Beyond reconstruction quality, Table \ref{tab:main_table_pixel}, \ref{tab:main_table_latent} and Figure \ref{fig:time_plot} also highlight the computational efficiency of our approach: \textsc{clamp} provides a favorable quality--runtime trade-off, often matching or improving metrics while reducing runtime relative to strong competing methods. For instance, SITCOM is a strong and fast baseline, yet on pixel-space FFHQ motion deblurring we achieve up to a \(4.14\times\) speedup while attaining higher PSNR. In latent space, we are approximately \(2.4\times\) faster than Latent DAPS and about \(9\times\) faster than ReSample, while also delivering strong metrics overall.

We provide additional experiments in Appendix~\ref{appendix:addtional_exp}, including higher-resolution evaluations, latent-space experiments, sample results across a wider range of datasets and tasks, hyperparameter analyses, sample diversity under severe degradations, experiments with various measurement-noise settings, normalized residual checks, and diagnostic studies.

\subsection{Accelerated MRI reconstruction results}
To assess applicability beyond natural images, we evaluate on multi-coil accelerated MRI reconstruction using SKM-TEA dataset~\cite{Desai2022SKMTEAAD}. Table~\ref{tab:mri_table} reports PSNR/SSIM for $\times4$ and $\times8$ acceleration. Our method achieves the best PSNR/SSIM among the compared diffusion-based baselines. Appendix~\ref{appendix:mri} provides experimental details, including the baselines and qualitative results.

\begin{table}[t]
\centering
\caption{Quantitative comparison for MRI reconstruction under Poisson-disc undersampling at acceleration factors $\times4$ and $\times8$. We report average PSNR/SSIM over 100 test slices. Best scores are in \textbf{bold} and second-best are \underline{underlined}.}
\label{tab:mri_table}
\begin{tabular}{c c c c c}
\toprule
\multirow{2}{*}{Method} & \multicolumn{2}{c}{$\times 4$} & \multicolumn{2}{c}{$\times 8$} \\
\cmidrule(lr){2-3}\cmidrule(lr){4-5}
& PSNR$\uparrow$ & SSIM$\uparrow$ 
& PSNR$\uparrow$ & SSIM$\uparrow$  \\
\midrule
DPS         &25.46 &0.531    &25.40  &0.527   \\
DAPS        &30.20  &0.687    &28.19  &0.582   \\
DDS         &32.07  &0.763    &30.03  &0.669    \\
Score-Med   &\underline{32.21}  &\underline{0.776} &\underline{30.29} &\underline{0.694}    \\
Ours         &\textbf{34.05}  &\textbf{0.834}    &\textbf{32.27}  &\textbf{0.766} \\
\bottomrule
\end{tabular}
\end{table}

\subsection{Ablation studies}
We ablate four components under matched compute (same diffusion steps and linear-solver budgets): noise-level transition, prior-aligned damping, curvature operator, and Krylov solver. Table~\ref{tab:ablation} shows that (i) replacing the variance-preserving transition with an ODE-style update and (ii) removing damping both degrade quality, especially for nonlinear operators. In the curvature model, dropping the pullback $\mathbf{J}_F^\top$ speeds up inference but reduces accuracy, whereas using the full curvature (including forward $\mathbf{J}_F$) substantially increases runtime with little or no gain over the one-sided approximation, supporting our choice as the best accuracy--efficiency trade-off. Finally, since the resulting system is generally nonsymmetric, GMRES is more stable than CG in practice; additional results are in Appendix~\ref{appendix:addtional_exp}.

\section{Conclusion}
\label{sec:Conclusion}

\textsc{clamp} provides a training-free posterior sampler for inverse problems that works with pixel and latent diffusion priors and supports both linear and nonlinear forward models. Its core update replaces heuristic scalar guidance with a denoiser-pullback Gauss--Newton correction and a prior-aligned anisotropic damping metric, producing noise-level--dependent matrix-scaled steps. Although each correction uses additional Jacobian actions (JVP/VJP) beyond a single backprojected gradient, it is computed with a fixed-budget matrix-free solve and yields favorable end-to-end runtime in practice. \textsc{clamp} is not intended to be tuning-free; rather, it reduces non-compute tuning by replacing per-task guidance-weight schedules with simple scalar settings for data consistency and damping, while diffusion steps and Krylov iterations specify the runtime--quality trade-off and may be adjusted by task. Empirically, it achieves competitive or improved reconstruction quality over the compared baselines while often reducing inference time on FFHQ/ImageNet inverse problems and MRI reconstruction. Limitations and future directions are discussed in Appendix~\ref{appendix:future_limit}.

\section*{Impact Statement}
CLAMP makes diffusion posterior sampling for inverse problems more reliable and faster by replacing hand-tuned scalar guidance with a noise-level–specific, geometry-aware correction. It computes coordinate-correct likelihood updates by pulling measurement sensitivity back through the denoiser, then stabilizes them with manifold-aligned anisotropic damping—yielding direction-wise step control without per-task schedules. A matrix-free GMRES implementation and a variance-preserving Langevin transition deliver a strong quality–runtime trade-off across diverse linear/nonlinear operators, including medical imaging (accelerated MRI).

\section*{Acknowledgement}
This work was supported in part by the National Research Foundation of Korea (NRF) grant funded by the Korea government (MSIT) (RS-2025-24683103, RS-2026-25476632), in part by Korea Basic Science Institute (National research Facilities and Equipment Center) grant funded by the Ministry of Science and ICT (No. RS-2024-00401899), and in part by Institute of Information \& communications Technology Planning \& Evaluation (IITP) under the Leading Generative AI Human Resources Development (IITP-2026-RS-2024-00360227) grant funded by the Korea government (MSIT).

\newpage
\bibliography{main}
\bibliographystyle{icml2026}

\appendix
\onecolumn

\FloatBarrier 
\clearpage 
\section{Algorithms}
\label{sec:appendix_algo}

\begin{algorithm}[H]
\caption{\textsc{clamp}}
\label{alg:CLAMP}
\begin{algorithmic}[1]
\Require Measurement $\mathbf{y}$; forward model $A$; data-consistency scale $\sigma_n$.
\Statex \hspace{\algorithmicindent} Denoiser $F_\theta(\cdot;\sigma)$; schedule $\{\sigma_t\}_{t=T}^{0}$.
\Statex \hspace{\algorithmicindent} $\lambda_{\mathrm{id}}\ge 0$; GMRES iterations $K$.
\Statex \hspace{\algorithmicindent} \(\triangleright\) All Jacobian actions are matrix-free: use JVP for $\mathbf{J}_A\mathbf{v}$ and VJP for $\mathbf{J}_A^\top\mathbf{u}$, $\mathbf{J}_F^\top\mathbf{u}$.
\State $\sigma_{\min}^{+}\gets \min\{\sigma_t:\sigma_t>0\}$.
\State Draw $\mathbf{x}_T\sim\mathcal{N}(\mathbf{0},\sigma_T^2\mathbf{I})$.
\For{$t=T,\dots,1$}
  \State $\hat{\mathbf{x}}_{0|t} \gets F_\theta(\mathbf{x}_t;\sigma_t)$
  \State $\mathbf{r}_t \gets A(\hat{\mathbf{x}}_{0|t})-\mathbf{y}$
  \State $\mathbf{d}_t \gets \mathbf{x}_t-\hat{\mathbf{x}}_{0|t}$;\;\;
         $\hat{\boldsymbol{\epsilon}}_\theta \gets \mathbf{d}_t/\sigma_t$
  \State $\sigma_{\mathrm{prior}}\gets \max(\sigma_{t-1},\sigma_{\min}^{+})$
  \State $\mathbf{u}_t \gets \mathbf{d}_t/\|\mathbf{d}_t\|_2$
  \State $\mathbf{a}_t \gets \mathrm{VJP}_F\!\big(\mathrm{VJP}_A(\mathbf{r}_t)\big)$
  \State $\mathbf{c}_t \gets \mathbf{a}_t/\sigma_n^2$
  \State $\alpha_t \gets (\|\mathbf{d}_t\|_2/\|\mathbf{a}_t\|_2)\cdot \sigma_{\mathrm{prior}}^{-2}$
  \State Define $\mathbf{H}_t(\mathbf{v}) \gets \lambda_{\mathrm{id}}\mathbf{v}+\mathbf{u}_t(\mathbf{u}_t^\top\mathbf{v})$
  \State Define $\mathbf{B}_t(\mathbf{v}) \gets \sigma_n^{-2}\,
         \mathrm{VJP}_F\!\big(\mathrm{VJP}_A(\mathrm{JVP}_A(\mathbf{v}))\big)$
  \State Define $\mathbf{M}_t(\mathbf{v}) \gets \mathbf{B}_t(\mathbf{v})+\alpha_t\mathbf{H}_t(\mathbf{v})$
  \State $\Delta_t \gets \mathrm{GMRES}(\mathbf{M}_t,\,-\mathbf{c}_t;\,K)$ (Alg.~\ref{alg:CLAMP_gmres})
  \State $\sigma_{\mathrm{ret},t}\gets \sigma_t-\sqrt{\sigma_t^2-\sigma_{t-1}^2}$
  \State $\sigma_{\mathrm{inj},t}\gets \sqrt{\sigma_{t-1}^2-\sigma_{\mathrm{ret},t}^2}$
  \State $\boldsymbol{\xi}\sim\mathcal{N}(\mathbf{0},\mathbf{I})$
  \State $\mathbf{x}_{t-1}\gets \mathbf{x}_t + \Delta_t
     + (\sigma_{\mathrm{ret},t}-\sigma_t)\hat{\boldsymbol{\epsilon}}_\theta
     + \sigma_{\mathrm{inj},t}\boldsymbol{\xi}$
\EndFor
\State \Return $\mathbf{x}_0$
\end{algorithmic}
\end{algorithm}

\clearpage 

\begin{algorithm}[H]
\caption{GMRES for $(\mathbf{B}_t+\alpha_t\mathbf{H}_t)\Delta=-\mathbf{c}_t$}
\label{alg:CLAMP_gmres}
\begin{algorithmic}[1]
\Require Matvec $\mathbf{M}(\mathbf{v}) := \mathbf{B}_t(\mathbf{v})+\alpha_t\mathbf{H}_t(\mathbf{v})$; RHS $\mathbf{b}:=-\mathbf{c}_t$;
Krylov budget $K$; breakdown tolerance $\tau$ (e.g., $10^{-12}$).
\Ensure Approximate solution $\Delta$.

\State $\Delta_0\gets \mathbf{0}$
\State $\mathbf{r}_0\gets \mathbf{b}-\mathbf{M}(\Delta_0)$ 
\State $\beta\gets \|\mathbf{r}_0\|_2$
\If{$\beta<\tau$} \State \Return $\mathbf{0}$ \EndIf
\State $\mathbf{v}_1\gets \mathbf{r}_0/\beta$
\State $m\gets K$ 

\For{$j=1,2,\dots,K$}
    \State $\mathbf{w}\gets \mathbf{M}(\mathbf{v}_j)$
    \For{$\ell=1,2,\dots,j$}
        \State $h_{\ell j}\gets \mathbf{v}_\ell^\top \mathbf{w}$
        \State $\mathbf{w}\gets \mathbf{w}-h_{\ell j}\mathbf{v}_\ell$
    \EndFor
    \State $h_{j+1,j}\gets \|\mathbf{w}\|_2$
    \If{$h_{j+1,j}<\tau$}
        \State $m\gets j$ 
        \State \textbf{break}
    \EndIf
    \State $\mathbf{v}_{j+1}\gets \mathbf{w}/h_{j+1,j}$
\EndFor

\State Form the (upper) Hessenberg matrix $\bar{\mathcal{H}}_m\in\mathbb{R}^{(m+1)\times m}$
from $\{h_{\ell j}\}_{1\le \ell\le j+1,\;1\le j\le m}$.
\State Set $\mathbf{e}_1 \gets [\beta,0,\dots,0]^\top \in \mathbb{R}^{m+1}$
\State Solve $\mathbf{y}=\arg\min_{\mathbf{y}\in\mathbb{R}^m}\|\bar{\mathcal{H}}_m\mathbf{y}-\mathbf{e}_1\|_2$
(least squares; e.g., QR/SVD)
\State $\Delta \gets \sum_{j=1}^{m} y_j\,\mathbf{v}_j$
\State \Return $\Delta$
\end{algorithmic}
\end{algorithm}

\clearpage 

\begin{algorithm}[H]
\caption{Latent-\textsc{clamp}}
\label{alg:latent_CLAMP}
\begin{algorithmic}[1]
\Require Measurement $\mathbf{y}$; forward model $A$; data-consistency scale $\sigma_n$.
\Statex \hspace{\algorithmicindent} Latent denoiser $F_\theta(\cdot;\sigma)$; decoder $D(\cdot)$; schedule $\{\sigma_t\}_{t=T}^{0}$.
\Statex \hspace{\algorithmicindent} $\lambda_{\mathrm{id}}\ge 0$; GMRES iterations $K$.
\Statex \hspace{\algorithmicindent} \(\triangleright\) All Jacobian actions are matrix-free: use JVP/VJP for $A\!\circ\!D$ and VJP for $F$.
\State $\sigma_{\min}^{+}\gets \min\{\sigma_t:\sigma_t>0\}$.
\State Draw $\mathbf{z}_T\sim\mathcal{N}(\mathbf{0},\sigma_T^2\mathbf{I})$.
\For{$t=T,\dots,1$}
  \State $\hat{\mathbf{z}}_{0|t} \gets F_\theta(\mathbf{z}_t;\sigma_t)$
  \State $\mathbf{r}_t \gets A\!\big(D(\hat{\mathbf{z}}_{0|t})\big)-\mathbf{y}$
  \State $\mathbf{d}_t \gets \mathbf{z}_t-\hat{\mathbf{z}}_{0|t}$;\;\;
         $\hat{\boldsymbol{\epsilon}}_\theta \gets \mathbf{d}_t/\sigma_t$
  \State $\sigma_{\mathrm{prior}}\gets \max(\sigma_{t-1},\sigma_{\min}^{+})$
  \State $\mathbf{u}_t \gets \mathbf{d}_t/\|\mathbf{d}_t\|_2$
  \State $\mathbf{a}_t \gets \mathrm{VJP}_F\!\Big(\mathrm{VJP}_{A\circ D}(\mathbf{r}_t)\Big)$
  \State $\mathbf{c}_t \gets \mathbf{a}_t/\sigma_n^2$
  \State $\alpha_t \gets (\|\mathbf{d}_t\|_2/\|\mathbf{a}_t\|_2)\cdot \sigma_{\mathrm{prior}}^{-2}$
  \State Define $\mathbf{H}_t(\mathbf{v}) \gets \lambda_{\mathrm{id}}\mathbf{v}+\mathbf{u}_t(\mathbf{u}_t^\top\mathbf{v})$
  \State Define $\mathbf{B}_t(\mathbf{v}) \gets \sigma_n^{-2}\,
         \mathrm{VJP}_F\!\Big(\mathrm{VJP}_{A\circ D}(\mathrm{JVP}_{A\circ D}(\mathbf{v}))\Big)$
  \State Define $\mathbf{M}_t(\mathbf{v}) \gets \mathbf{B}_t(\mathbf{v})+\alpha_t\mathbf{H}_t(\mathbf{v})$
  \State $\Delta_t \gets \mathrm{GMRES}(\mathbf{M}_t,\,-\mathbf{c}_t;\,K)$ (Alg.~\ref{alg:CLAMP_gmres})
  \State $\sigma_{\mathrm{ret},t}\gets \sigma_t-\sqrt{\sigma_t^2-\sigma_{t-1}^2}$
  \State $\sigma_{\mathrm{inj},t}\gets \sqrt{\sigma_{t-1}^2-\sigma_{\mathrm{ret},t}^2}$
  \State $\boldsymbol{\xi}\sim\mathcal{N}(\mathbf{0},\mathbf{I})$
  \State $\mathbf{z}_{t-1}\gets \mathbf{z}_t + \Delta_t
     + (\sigma_{\mathrm{ret},t}-\sigma_t)\hat{\boldsymbol{\epsilon}}_\theta
     + \sigma_{\mathrm{inj},t}\boldsymbol{\xi}$
\EndFor
\State \Return $\mathbf{x}_0 \gets D(\mathbf{z}_0)$
\end{algorithmic}
\end{algorithm}

\FloatBarrier 

\section{Theoretical details}
\label{sec:appendix_theory}

This appendix collects proofs for the propositions stated in Sec.~\ref{sec:Method}.
We use the Euclidean inner product $\langle \mathbf{a},\mathbf{b}\rangle := \mathbf{a}^\top\mathbf{b}$
and norm $\|\mathbf{a}\|_2 := \sqrt{\mathbf{a}^\top\mathbf{a}}$.
For a symmetric positive definite matrix $\mathbf{Q}$, we denote the largest and smallest eigenvalues by
$\lambda_{\max}(\mathbf{Q})$ and $\lambda_{\min}(\mathbf{Q})$, and the condition number by
$\kappa(\mathbf{Q}) := \lambda_{\max}(\mathbf{Q})/\lambda_{\min}(\mathbf{Q})$.

\subsection{Proof of Proposition~\ref{prop:scalar_barrier}}
\label{app:proof_scalar_barrier}

\textbf{Proposition~\ref{prop:scalar_barrier}.}
Fix $t$ and consider a local quadratic model in the increment $\Delta$:
\begin{equation}
\mathcal{J}(\Delta)
=
\mathcal{J}(\mathbf{0})+\langle \mathbf{g}_t,\Delta\rangle+\frac{1}{2}\Delta^\top\mathbf{Q}_t\Delta,
\qquad
\mathbf{Q}_t\succ 0.
\label{eq:app_quad_model}
\end{equation}
Consider the scalar-guided update $\Delta_{\rm sc}=-\eta\mathbf{g}_t$ with $\eta\ge 0$.
Then $\mathcal{J}(\Delta_{\rm sc})\le \mathcal{J}(\mathbf{0})$ for all $\mathbf{g}_t$ if and only if
$\eta\le 2/\lambda_{\max}(\mathbf{Q}_t)$.

Let $\Delta^\star=-\mathbf{Q}_t^{-1}\mathbf{g}_t$ be the minimizer of \eqref{eq:app_quad_model}.
In an eigenbasis $\mathbf{Q}_t\mathbf{v}_i=\lambda_i\mathbf{v}_i$, for each index $i$ with
$\langle \mathbf{v}_i,\mathbf{g}_t\rangle\neq 0$,
\begin{equation}
\frac{|\langle \mathbf{v}_i,\Delta_{\rm sc}\rangle|}{|\langle \mathbf{v}_i,\Delta^\star\rangle|}
= \eta\lambda_i
\le \frac{2\lambda_i}{\lambda_{\max}(\mathbf{Q}_t)},
\label{eq:app_ratio_bound_stmt}
\end{equation}
and in particular for $\lambda_i=\lambda_{\min}(\mathbf{Q}_t)$,
$
\frac{|\langle \mathbf{v}_{\min},\Delta_{\rm sc}\rangle|}{|\langle \mathbf{v}_{\min},\Delta^\star\rangle|}
\le \frac{2}{\kappa(\mathbf{Q}_t)}.
$

\begin{proof}
We prove three claims:
(i) the objective change formula under $\Delta_{\rm sc}$;
(ii) the (necessary and sufficient) descent condition $\eta\le 2/\lambda_{\max}$;
(iii) the eigen-coordinate ratio \eqref{eq:app_ratio_bound_stmt}.

\medskip
\noindent
\textbf{1) Objective change under a scalar step.}
Substitute $\Delta_{\rm sc}=-\eta\mathbf{g}_t$ into \eqref{eq:app_quad_model}:
\begin{align}
\mathcal{J}(\Delta_{\rm sc})-\mathcal{J}(\mathbf{0})
&=
\left\langle \mathbf{g}_t,\,-\eta\mathbf{g}_t\right\rangle
+\frac{1}{2}(-\eta\mathbf{g}_t)^\top\mathbf{Q}_t(-\eta\mathbf{g}_t)
\nonumber\\
&=
-\eta\,\mathbf{g}_t^\top\mathbf{g}_t
+\frac{\eta^2}{2}\,\mathbf{g}_t^\top\mathbf{Q}_t\mathbf{g}_t
\nonumber\\
&=
-\eta\,\|\mathbf{g}_t\|_2^2
+\frac{\eta^2}{2}\,\mathbf{g}_t^\top\mathbf{Q}_t\mathbf{g}_t.
\label{eq:app_deltaJ_exact}
\end{align}

\medskip
\noindent
\textbf{2) Rayleigh quotient bound.}
Because $\mathbf{Q}_t$ is symmetric positive definite, there exists an orthonormal eigenbasis
$\{\mathbf{v}_i\}_{i=1}^d$ and eigenvalues $\{\lambda_i\}_{i=1}^d$ with $\lambda_i>0$ and
$\mathbf{Q}_t\mathbf{v}_i=\lambda_i\mathbf{v}_i$.
Expand $\mathbf{g}_t$ as
\[
\mathbf{g}_t=\sum_{i=1}^d a_i\mathbf{v}_i,
\qquad
a_i:=\langle \mathbf{v}_i,\mathbf{g}_t\rangle.
\]
Then, using orthonormality $\mathbf{v}_i^\top\mathbf{v}_j=\delta_{ij}$,
\[
\|\mathbf{g}_t\|_2^2
=
\mathbf{g}_t^\top\mathbf{g}_t
=
\sum_{i=1}^d a_i^2,
\]
and
\[
\mathbf{g}_t^\top\mathbf{Q}_t\mathbf{g}_t
=
\left(\sum_i a_i\mathbf{v}_i\right)^\top
\left(\sum_j a_j\mathbf{Q}_t\mathbf{v}_j\right)
=
\sum_{j=1}^d \lambda_j a_j^2.
\]
Since $\lambda_j\le \lambda_{\max}(\mathbf{Q}_t)$ for all $j$, we obtain
\begin{equation}
\mathbf{g}_t^\top\mathbf{Q}_t\mathbf{g}_t
=
\sum_j \lambda_j a_j^2
\le
\lambda_{\max}(\mathbf{Q}_t)\sum_j a_j^2
=
\lambda_{\max}(\mathbf{Q}_t)\|\mathbf{g}_t\|_2^2.
\label{eq:app_rayleigh}
\end{equation}

\medskip
\noindent
\textbf{3) Sufficient condition for descent.}
Combine \eqref{eq:app_deltaJ_exact} and \eqref{eq:app_rayleigh}:
\begin{align}
\mathcal{J}(\Delta_{\rm sc})-\mathcal{J}(\mathbf{0})
&\le
-\eta\,\|\mathbf{g}_t\|_2^2
+\frac{\eta^2}{2}\lambda_{\max}(\mathbf{Q}_t)\|\mathbf{g}_t\|_2^2
\nonumber\\
&=
-\eta\,\|\mathbf{g}_t\|_2^2
\left(1-\frac{\eta}{2}\lambda_{\max}(\mathbf{Q}_t)\right).
\label{eq:app_deltaJ_bound}
\end{align}
If $0<\eta\le 2/\lambda_{\max}(\mathbf{Q}_t)$, then
$1-\frac{\eta}{2}\lambda_{\max}(\mathbf{Q}_t)\ge 0$ and therefore
$\mathcal{J}(\Delta_{\rm sc})-\mathcal{J}(\mathbf{0})\le 0$.
Hence $\mathcal{J}(\Delta_{\rm sc})\le \mathcal{J}(\mathbf{0})$.

\medskip
\noindent
\textbf{4) Necessity of $\eta\le 2/\lambda_{\max}(\mathbf{Q}_t)$ for uniform descent over all $\mathbf{g}_t$.}
Assume $\eta>2/\lambda_{\max}(\mathbf{Q}_t)$.
Let $\mathbf{v}_{\max}$ be a unit eigenvector associated with $\lambda_{\max}(\mathbf{Q}_t)$, and choose $\mathbf{g}_t=\mathbf{v}_{\max}$.
Then $\|\mathbf{g}_t\|_2^2=1$ and $\mathbf{g}_t^\top\mathbf{Q}_t\mathbf{g}_t=\lambda_{\max}(\mathbf{Q}_t)$.
From \eqref{eq:app_deltaJ_exact},
\[
\mathcal{J}(\Delta_{\rm sc})-\mathcal{J}(\mathbf{0})
=
-\eta + \frac{\eta^2}{2}\lambda_{\max}(\mathbf{Q}_t)
=
\eta\left(\frac{\eta}{2}\lambda_{\max}(\mathbf{Q}_t)-1\right)
>0,
\]
since $\frac{\eta}{2}\lambda_{\max}(\mathbf{Q}_t)>1$ by assumption.
Thus the objective increases for this choice of $\mathbf{g}_t$, so $\eta>2/\lambda_{\max}$ cannot guarantee uniform descent.
This shows the condition is also necessary for a guarantee that holds for all $\mathbf{g}_t$.

\medskip
\noindent
\textbf{5) Closed form of the quadratic minimizer $\Delta^\star$.}
Differentiate \eqref{eq:app_quad_model} with respect to $\Delta$.
The derivative of $\langle \mathbf{g}_t,\Delta\rangle=\mathbf{g}_t^\top\Delta$ is $\mathbf{g}_t$.
For symmetric $\mathbf{Q}_t$, the derivative of $\frac12\Delta^\top\mathbf{Q}_t\Delta$ is $\mathbf{Q}_t\Delta$.
Therefore,
\[
\nabla_\Delta \mathcal{J}(\Delta) = \mathbf{g}_t + \mathbf{Q}_t\Delta.
\]
Setting $\nabla_\Delta \mathcal{J}(\Delta^\star)=\mathbf{0}$ gives
$\mathbf{Q}_t\Delta^\star=-\mathbf{g}_t$, hence
\[
\Delta^\star=-\mathbf{Q}_t^{-1}\mathbf{g}_t.
\]

\medskip
\noindent
\textbf{6) Eigen-coordinate ratio.}
Using the eigen-expansion $\mathbf{g}_t=\sum_i a_i\mathbf{v}_i$,
\[
\Delta_{\rm sc}=-\eta\mathbf{g}_t = \sum_i (-\eta a_i)\mathbf{v}_i.
\]
Also, because $\mathbf{Q}_t^{-1}\mathbf{v}_i=\frac{1}{\lambda_i}\mathbf{v}_i$,
\[
\Delta^\star
=
-\mathbf{Q}_t^{-1}\mathbf{g}_t
=
-\sum_i a_i\,\mathbf{Q}_t^{-1}\mathbf{v}_i
=
-\sum_i \frac{a_i}{\lambda_i}\mathbf{v}_i.
\]
Therefore
\[
\langle \mathbf{v}_i,\Delta_{\rm sc}\rangle = -\eta a_i,
\qquad
\langle \mathbf{v}_i,\Delta^\star\rangle = -\frac{a_i}{\lambda_i}.
\]
If $a_i=\langle \mathbf{v}_i,\mathbf{g}_t\rangle=0$, then both components are zero and the ratio is not informative.
If $a_i\neq 0$, then
\[
\frac{|\langle \mathbf{v}_i,\Delta_{\rm sc}\rangle|}{|\langle \mathbf{v}_i,\Delta^\star\rangle|}
=
\frac{|\eta a_i|}{|a_i|/\lambda_i}
=
\eta\lambda_i.
\]
Under the stability condition $\eta\le 2/\lambda_{\max}(\mathbf{Q}_t)$, this yields
\[
\eta\lambda_i \le \frac{2\lambda_i}{\lambda_{\max}(\mathbf{Q}_t)}.
\]
For $\lambda_i=\lambda_{\min}(\mathbf{Q}_t)$ the bound becomes
$
\eta\lambda_{\min}(\mathbf{Q}_t)\le 2/\kappa(\mathbf{Q}_t)
$.
\end{proof}

\subsection{Proof of Proposition~\ref{prop:pullback}}
\label{app:proof_pullback}

\textbf{Proposition~\ref{prop:pullback}.}
Let
\[
\tilde\Phi(\mathbf{z}) := \frac{1}{2\sigma_n^2}\|A(\mathbf{z})-\mathbf{y}\|_2^2,
\qquad
\Phi_t(\mathbf{x}) := \tilde\Phi(F_\theta(\mathbf{x};\sigma_t)).
\]
Let $\mathbf{z}_t := F_\theta(\mathbf{x}_t;\sigma_t)$,
$\mathbf{g}_z := \nabla_{\mathbf{z}}\tilde\Phi(\mathbf{z})|_{\mathbf{z}=\mathbf{z}_t}$,
and $\mathbf{J}_F := \nabla_{\mathbf{x}}F_\theta(\mathbf{x};\sigma_t)|_{\mathbf{x}=\mathbf{x}_t}$.
Then, for any increment $\Delta$,
\[
\Phi_t(\mathbf{x}_t+\Delta)
=
\Phi_t(\mathbf{x}_t)
+
\langle \mathbf{J}_F^{\!\top}\mathbf{g}_z,\Delta\rangle
+ o(\|\Delta\|_2).
\]
Consequently, among all increments with $\|\Delta\|_2\le \rho$, the direction that maximizes first-order decrease of $\Phi_t$ is proportional to
$-\mathbf{J}_F^{\!\top}\mathbf{g}_z$.
Under first-order residual linearization, the Gauss--Newton/Fisher curvature of $\Phi_t$ in diffusion-state coordinates is
$\frac{1}{\sigma_n^2}\mathbf{J}_F^{\!\top}\mathbf{J}_A^{\!\top}\mathbf{J}_A\mathbf{J}_F$, where
$\mathbf{J}_A := \nabla_{\mathbf{z}}A(\mathbf{z})|_{\mathbf{z}=\mathbf{z}_t}$.

\begin{proof}
We prove the first-order expansion, then the steepest descent direction, and finally the GN curvature form.

\medskip
\noindent
\textbf{1) First-order expansion of $\Phi_t(\mathbf{x}_t+\Delta)$.}

\emph{Step 1: Expand the denoiser.}
Because $F_\theta(\cdot;\sigma_t)$ is differentiable at $\mathbf{x}_t$, by the definition of the Jacobian there exists a remainder term
$\mathbf{r}_F(\Delta)$ such that
\begin{equation}
F_\theta(\mathbf{x}_t+\Delta;\sigma_t)
=
F_\theta(\mathbf{x}_t;\sigma_t) + \mathbf{J}_F\Delta + \mathbf{r}_F(\Delta),
\label{eq:app_F_expand}
\end{equation}
where $\|\mathbf{r}_F(\Delta)\|_2 = o(\|\Delta\|_2)$ as $\Delta\to \mathbf{0}$.
Define $\mathbf{z}_t:=F_\theta(\mathbf{x}_t;\sigma_t)$, so \eqref{eq:app_F_expand} becomes
\[
F_\theta(\mathbf{x}_t+\Delta;\sigma_t)
=
\mathbf{z}_t + \mathbf{J}_F\Delta + o(\|\Delta\|_2).
\]

\emph{Step 2: Expand the clean-space likelihood.}
Because $\tilde\Phi$ is differentiable at $\mathbf{z}_t$, there exists a remainder $r_{\tilde\Phi}(\mathbf{u})$ such that
\begin{equation}
\tilde\Phi(\mathbf{z}_t+\mathbf{u})
=
\tilde\Phi(\mathbf{z}_t) + \langle \mathbf{g}_z,\mathbf{u}\rangle + r_{\tilde\Phi}(\mathbf{u}),
\label{eq:app_tildePhi_expand}
\end{equation}
where $\mathbf{g}_z := \nabla_{\mathbf{z}}\tilde\Phi(\mathbf{z})|_{\mathbf{z}=\mathbf{z}_t}$ and
$r_{\tilde\Phi}(\mathbf{u})=o(\|\mathbf{u}\|_2)$ as $\mathbf{u}\to \mathbf{0}$.

\emph{Step 3: Compose the expansions.}
By definition $\Phi_t(\mathbf{x})=\tilde\Phi(F_\theta(\mathbf{x};\sigma_t))$.
Let
\[
\mathbf{u}(\Delta):=F_\theta(\mathbf{x}_t+\Delta;\sigma_t)-\mathbf{z}_t.
\]
From Step 1, $\mathbf{u}(\Delta)=\mathbf{J}_F\Delta+o(\|\Delta\|_2)$, hence $\|\mathbf{u}(\Delta)\|_2=O(\|\Delta\|_2)$.
Apply \eqref{eq:app_tildePhi_expand} with $\mathbf{u}=\mathbf{u}(\Delta)$:
\begin{align*}
\Phi_t(\mathbf{x}_t+\Delta)
&=
\tilde\Phi\!\big(\mathbf{z}_t+\mathbf{u}(\Delta)\big)\\
&=
\tilde\Phi(\mathbf{z}_t)+\langle \mathbf{g}_z,\mathbf{u}(\Delta)\rangle + o(\|\mathbf{u}(\Delta)\|_2).
\end{align*}
Substitute $\mathbf{u}(\Delta)=\mathbf{J}_F\Delta+o(\|\Delta\|_2)$:
\[
\Phi_t(\mathbf{x}_t+\Delta)
=
\tilde\Phi(\mathbf{z}_t) + \langle \mathbf{g}_z,\mathbf{J}_F\Delta\rangle + o(\|\Delta\|_2),
\]
because $\langle \mathbf{g}_z,o(\|\Delta\|_2)\rangle=o(\|\Delta\|_2)$ and $o(\|\mathbf{u}(\Delta)\|_2)=o(\|\Delta\|_2)$.
Finally, use $\langle \mathbf{g}_z,\mathbf{J}_F\Delta\rangle=\langle \mathbf{J}_F^\top\mathbf{g}_z,\Delta\rangle$ and note
$\tilde\Phi(\mathbf{z}_t)=\Phi_t(\mathbf{x}_t)$ to obtain
\[
\Phi_t(\mathbf{x}_t+\Delta)
=
\Phi_t(\mathbf{x}_t)
+
\langle \mathbf{J}_F^\top\mathbf{g}_z,\Delta\rangle
+
o(\|\Delta\|_2).
\]

\medskip
\noindent
\textbf{2) Steepest first-order decrease under $\|\Delta\|_2\le \rho$.}
Let $\mathbf{w}:=\mathbf{J}_F^\top\mathbf{g}_z$.
For small $\Delta$, the first-order change is governed by $\langle \mathbf{w},\Delta\rangle$.
To maximize first-order decrease we minimize $\langle \mathbf{w},\Delta\rangle$ subject to $\|\Delta\|_2\le \rho$.

By Cauchy--Schwarz,
\[
\langle \mathbf{w},\Delta\rangle \ge -\|\mathbf{w}\|_2\,\|\Delta\|_2 \ge -\|\mathbf{w}\|_2\,\rho.
\]
Equality is achieved by choosing $\Delta$ colinear with $-\mathbf{w}$ and with maximal norm:
\[
\Delta^\star
=
-\rho\,\frac{\mathbf{w}}{\|\mathbf{w}\|_2}
=
-\rho\,\frac{\mathbf{J}_F^\top\mathbf{g}_z}{\|\mathbf{J}_F^\top\mathbf{g}_z\|_2},
\qquad
(\mathbf{w}\neq \mathbf{0}).
\]
Thus the steepest descent direction is proportional to $-\mathbf{J}_F^\top\mathbf{g}_z$.

\medskip
\noindent
\textbf{3) Gauss--Newton/Fisher curvature in diffusion-state coordinates.}
Define the clean residual $\mathbf{r}(\mathbf{z}) := A(\mathbf{z})-\mathbf{y}$, so
$\tilde\Phi(\mathbf{z})=\frac{1}{2\sigma_n^2}\|\mathbf{r}(\mathbf{z})\|_2^2$.
Let $\mathbf{r}_t:=\mathbf{r}(\mathbf{z}_t)=A(\mathbf{z}_t)-\mathbf{y}$ and
$\mathbf{J}_A:=\nabla_{\mathbf{z}}A(\mathbf{z})|_{\mathbf{z}=\mathbf{z}_t}$.

First-order linearization of $A$ around $\mathbf{z}_t$ gives
\[
A(\mathbf{z}_t+\mathbf{u}) \approx A(\mathbf{z}_t) + \mathbf{J}_A\mathbf{u}
\quad\Longrightarrow\quad
\mathbf{r}(\mathbf{z}_t+\mathbf{u}) \approx \mathbf{r}_t + \mathbf{J}_A\mathbf{u}.
\]
From Step 1, the clean-space perturbation induced by $\Delta$ is $\mathbf{u}\approx \mathbf{J}_F\Delta$ (up to higher-order terms).
Therefore
\[
\mathbf{r}\!\left(F_\theta(\mathbf{x}_t+\Delta;\sigma_t)\right)
=
\mathbf{r}(\mathbf{z}_t+\mathbf{u})
\approx
\mathbf{r}_t + \mathbf{J}_A(\mathbf{J}_F\Delta).
\]
Substitute into $\Phi_t(\mathbf{x}_t+\Delta)=\frac{1}{2\sigma_n^2}\|\cdot\|_2^2$:
\begin{align*}
\Phi_t(\mathbf{x}_t+\Delta)
&\approx
\frac{1}{2\sigma_n^2}\left\|\mathbf{r}_t + \mathbf{J}_A\mathbf{J}_F\Delta\right\|_2^2\\
&=
\frac{1}{2\sigma_n^2}\Big(
\|\mathbf{r}_t\|_2^2
+2\langle \mathbf{r}_t,\;\mathbf{J}_A\mathbf{J}_F\Delta\rangle
+\|\mathbf{J}_A\mathbf{J}_F\Delta\|_2^2
\Big).
\end{align*}
Rewrite the cross term:
\[
\langle \mathbf{r}_t,\;\mathbf{J}_A\mathbf{J}_F\Delta\rangle
=
\langle \mathbf{J}_F^\top \mathbf{J}_A^\top\mathbf{r}_t,\;\Delta\rangle.
\]
Rewrite the quadratic term:
\[
\|\mathbf{J}_A\mathbf{J}_F\Delta\|_2^2
=
(\mathbf{J}_A\mathbf{J}_F\Delta)^\top(\mathbf{J}_A\mathbf{J}_F\Delta)
=
\Delta^\top \mathbf{J}_F^\top \mathbf{J}_A^\top \mathbf{J}_A\mathbf{J}_F\,\Delta.
\]
Hence the quadratic form in $\Delta$ is
$
\frac{1}{2}\Delta^\top\left(\frac{1}{\sigma_n^2}\mathbf{J}_F^\top \mathbf{J}_A^\top \mathbf{J}_A\mathbf{J}_F\right)\Delta,
$
which is precisely the Gauss--Newton/Fisher curvature $\frac{1}{\sigma_n^2}\mathbf{J}^\top \mathbf{J}$ for $\mathbf{J}=\mathbf{J}_A\mathbf{J}_F$.
\end{proof}

\subsection{Proof of Proposition~\ref{prop:vp_closed_form}}
\label{app:proof_vp_closed_form}

\textbf{Proposition~\ref{prop:vp_closed_form}.}
Consider the stochastic prior propagation step
\begin{equation}
\mathbf{x}_{t-1}
=
\mathbf{x}_t + \Delta_t
+(\sigma_{\mathrm{ret},t}-\sigma_t)\,\epsilon_\theta(\mathbf{x}_t,\sigma_t)
+\sigma_{\mathrm{inj},t}\,\epsilon,
\qquad
\epsilon\sim\mathcal{N}(\mathbf{0},\mathbf{I}),
\label{eq:app_vp_step}
\end{equation}
and impose schedule variance matching
\begin{equation}
\sigma_{t-1}^2 = \sigma_{\mathrm{ret},t}^2+\sigma_{\mathrm{inj},t}^2.
\label{eq:app_vp_match}
\end{equation}
Define
\[
\lambda_t := 1-\frac{\sigma_{\mathrm{ret},t}}{\sigma_t},
\qquad
\sigma_{\mathrm{ref}}^2 := \sigma_t\sigma_{\mathrm{ret},t},
\]
and assume the fluctuation--dissipation coupling
\begin{equation}
\sigma_{\mathrm{inj},t}^2 = 2\lambda_t\,\sigma_{\mathrm{ref}}^2.
\label{eq:app_fd}
\end{equation}
Then, under a decreasing nonnegative schedule $0\le \sigma_{t-1}\le \sigma_t$ with $\sigma_t>0$,
the pair $(\sigma_{\mathrm{ret},t},\sigma_{\mathrm{inj},t})$ is uniquely determined by
\[
\sigma_{\mathrm{ret},t}
=
\sigma_t-\sqrt{\sigma_t^2-\sigma_{t-1}^2},
\qquad
\sigma_{\mathrm{inj},t}
=
\sqrt{\sigma_{t-1}^2-\sigma_{\mathrm{ret},t}^2}.
\]

\begin{proof}
The proof is algebraic: we eliminate $\sigma_{\mathrm{inj},t}$ using \eqref{eq:app_vp_match} and \eqref{eq:app_fd}, then solve for
$\sigma_{\mathrm{ret},t}$.

\medskip
\noindent
\textbf{1) Expand the coupling \eqref{eq:app_fd}.}
Substitute the definitions of $\lambda_t$ and $\sigma_{\mathrm{ref}}^2$ into \eqref{eq:app_fd}:
\begin{align}
\sigma_{\mathrm{inj},t}^2
&=
2\left(1-\frac{\sigma_{\mathrm{ret},t}}{\sigma_t}\right)\,(\sigma_t\sigma_{\mathrm{ret},t})
\nonumber\\
&=
2\left(\sigma_t\sigma_{\mathrm{ret},t}-\sigma_{\mathrm{ret},t}^2\right)
\nonumber\\
&=
2(\sigma_t-\sigma_{\mathrm{ret},t})\,\sigma_{\mathrm{ret},t}.
\label{eq:app_inj_sq}
\end{align}

\medskip
\noindent
\textbf{2) Eliminate $\sigma_{\mathrm{inj},t}$ using variance matching.}
From \eqref{eq:app_vp_match},
\[
\sigma_{t-1}^2 = \sigma_{\mathrm{ret},t}^2+\sigma_{\mathrm{inj},t}^2.
\]
Substitute \eqref{eq:app_inj_sq}:
\begin{align}
\sigma_{t-1}^2
&=
\sigma_{\mathrm{ret},t}^2 + 2(\sigma_t-\sigma_{\mathrm{ret},t})\,\sigma_{\mathrm{ret},t}
\nonumber\\
&=
\sigma_{\mathrm{ret},t}^2 + 2\sigma_t\sigma_{\mathrm{ret},t} - 2\sigma_{\mathrm{ret},t}^2
\nonumber\\
&=
2\sigma_t\sigma_{\mathrm{ret},t}-\sigma_{\mathrm{ret},t}^2.
\label{eq:app_quadratic_start}
\end{align}
Rearrange to obtain a quadratic equation in $\sigma_{\mathrm{ret},t}$:
\begin{equation}
\sigma_{\mathrm{ret},t}^2 - 2\sigma_t\sigma_{\mathrm{ret},t} + \sigma_{t-1}^2 = 0.
\label{eq:app_quadratic}
\end{equation}

\medskip
\noindent
\textbf{3) Solve the quadratic.}
Apply the quadratic formula to \eqref{eq:app_quadratic}:
\[
\sigma_{\mathrm{ret},t}
=
\frac{2\sigma_t \pm \sqrt{(2\sigma_t)^2-4\sigma_{t-1}^2}}{2}
=
\sigma_t \pm \sqrt{\sigma_t^2-\sigma_{t-1}^2}.
\]
Thus the two algebraic candidates are
\begin{equation}
\sigma_{\mathrm{ret},t}^{(\pm)} = \sigma_t \pm \sqrt{\sigma_t^2-\sigma_{t-1}^2}.
\label{eq:app_roots}
\end{equation}
The square root is real because $\sigma_t^2-\sigma_{t-1}^2\ge 0$ under $\sigma_{t-1}\le \sigma_t$.

\medskip
\noindent
\textbf{4) Select the admissible root.}
Since $\sigma_{\mathrm{ret},t}$ and $\sigma_{\mathrm{inj},t}$ are standard deviations, they must satisfy
\[
\sigma_{\mathrm{ret},t}\ge 0,
\qquad
\sigma_{\mathrm{inj},t}^2=\sigma_{t-1}^2-\sigma_{\mathrm{ret},t}^2\ge 0
\quad\Longrightarrow\quad
\sigma_{\mathrm{ret},t}\le \sigma_{t-1}.
\]
We check the two roots in \eqref{eq:app_roots}.

\emph{(a) The ``plus'' root is inadmissible when $\sigma_{t-1}<\sigma_t$.}
If $\sigma_{t-1}<\sigma_t$, then $\sqrt{\sigma_t^2-\sigma_{t-1}^2}>0$ and
\[
\sigma_{\mathrm{ret},t}^{(+)}
=
\sigma_t + \sqrt{\sigma_t^2-\sigma_{t-1}^2}
>
\sigma_t
>
\sigma_{t-1},
\]
so it violates $\sigma_{\mathrm{ret},t}\le \sigma_{t-1}$ and cannot satisfy variance matching with a real $\sigma_{\mathrm{inj},t}$.

\emph{(b) The ``minus'' root is admissible.}
Define
\begin{equation}
\sigma_{\mathrm{ret},t}
:=
\sigma_{\mathrm{ret},t}^{(-)}
=
\sigma_t - \sqrt{\sigma_t^2-\sigma_{t-1}^2}.
\label{eq:app_ret_final}
\end{equation}

\underline{Nonnegativity.}
Because $\sigma_t^2-\sigma_{t-1}^2 \le \sigma_t^2$ and $\sigma_t\ge 0$, we have
$\sqrt{\sigma_t^2-\sigma_{t-1}^2}\le \sigma_t$, hence
$\sigma_{\mathrm{ret},t}\ge 0$.

\underline{Upper bound $\sigma_{\mathrm{ret},t}\le \sigma_{t-1}$.}
The inequality $\sigma_{\mathrm{ret},t}\le \sigma_{t-1}$ is equivalent to
\[
\sigma_t - \sqrt{\sigma_t^2-\sigma_{t-1}^2}\le \sigma_{t-1}
\quad\Longleftrightarrow\quad
\sqrt{\sigma_t^2-\sigma_{t-1}^2}\ge \sigma_t-\sigma_{t-1}.
\]
Both sides are nonnegative because $\sigma_t\ge \sigma_{t-1}\ge 0$, so we can square without changing the inequality:
\begin{align*}
\sigma_t^2-\sigma_{t-1}^2
&\ge
(\sigma_t-\sigma_{t-1})^2\\
&=
\sigma_t^2 - 2\sigma_t\sigma_{t-1} + \sigma_{t-1}^2.
\end{align*}
Cancel $\sigma_t^2$ on both sides and rearrange:
\[
-\sigma_{t-1}^2
\ge
-2\sigma_t\sigma_{t-1} + \sigma_{t-1}^2
\quad\Longleftrightarrow\quad
2\sigma_t\sigma_{t-1}\ge 2\sigma_{t-1}^2
\quad\Longleftrightarrow\quad
\sigma_{t-1}(\sigma_t-\sigma_{t-1})\ge 0,
\]
which holds because $\sigma_{t-1}\ge 0$ and $\sigma_t-\sigma_{t-1}\ge 0$.
Therefore $\sigma_{\mathrm{ret},t}\le \sigma_{t-1}$.

Thus the minus root is admissible. Under a strict decrease $\sigma_{t-1}<\sigma_t$, it is the only admissible root, hence
$\sigma_{\mathrm{ret},t}$ is uniquely determined.

\medskip
\noindent
\textbf{5) Recover $\sigma_{\mathrm{inj},t}$.}
Once $\sigma_{\mathrm{ret},t}$ is fixed, variance matching \eqref{eq:app_vp_match} gives
\[
\sigma_{\mathrm{inj},t}^2 = \sigma_{t-1}^2-\sigma_{\mathrm{ret},t}^2.
\]
The right-hand side is nonnegative because $\sigma_{\mathrm{ret},t}\le \sigma_{t-1}$, so define
\[
\sigma_{\mathrm{inj},t} := \sqrt{\sigma_{t-1}^2-\sigma_{\mathrm{ret},t}^2},
\]
taking the principal (nonnegative) square root.
\end{proof}

\clearpage

\begin{table*}[!t]
\centering
\setlength{\tabcolsep}{2.5pt}
\renewcommand{\arraystretch}{0.92}
\caption{Quantitative comparison on FFHQ and ImageNet across inverse-problem tasks in latent setting. We report average PSNR/SSIM (higher is better), LPIPS (lower is better), and run-time (lower is better) over 100 validation images. Best and second-best scores are highlighted in \textbf{bold} and \underline{underlined}, respectively.}
\label{tab:main_table_latent}
\begin{adjustbox}{max totalsize={\textwidth}{0.8\textheight},center}
\begin{tabular}{l!{\vrule}l!{\vrule}cccc!{\vrule}cccc}
\toprule
\multirow{2}{*}{\textbf{Task}} & \multirow{2}{*}{\textbf{Method}} &
\multicolumn{4}{c|}{\textbf{FFHQ}} & \multicolumn{4}{c}{\textbf{ImageNet}} \\
& &
\textbf{PSNR $\uparrow$} & \textbf{SSIM $\uparrow$} & \textbf{LPIPS $\downarrow$} & \textbf{Run-time (s)} &
\textbf{PSNR $\uparrow$} & \textbf{SSIM $\uparrow$} & \textbf{LPIPS $\downarrow$} & \textbf{Run-time (s)} \\
\midrule

\multirow{4}{*}{Super resolution 4$\times$} &
Ours       & \underline{28.933} & \textbf{0.829} & \textbf{0.233} & \textbf{34.255} & \textbf{26.465} & \textbf{0.726} & \textbf{0.335} & \textbf{37.353} \\
& PSLD       & 24.627 & 0.628 & 0.377 & 69.290 & 16.656 & 0.291 & 0.541 & 102.757 \\
& ReSample   & 23.317 & 0.456 & 0.507 & 300.061 & 22.152 & 0.423 & 0.470 & 269.078 \\
& LatentDAPS & \textbf{29.204} & \underline{0.825} & \underline{0.272} & \underline{84.460} & \underline{26.189} & \underline{0.702} & \underline{0.388} & \underline{86.315} \\
\midrule

\multirow{4}{*}{Box inpainting} &
Ours       & \textbf{25.151} & \textbf{0.837} & \textbf{0.236} & \textbf{35.369} & \underline{20.836} & \textbf{0.731} & \underline{0.342} & \textbf{37.393} \\
& PSLD       & 21.847 & 0.612 & 0.370 & 69.045 & 18.762 & 0.424 & 0.549 & 99.832 \\
& ReSample   & 19.978 & \underline{0.796} & \underline{0.247} & 296.326 & 18.087 & \underline{0.713} & \textbf{0.309} & 265.139 \\
& LatentDAPS & \underline{23.474} & 0.742 & 0.369 & \underline{85.823} & \textbf{22.818} & 0.561 & 0.543 & \underline{88.725} \\
\midrule

\multirow{4}{*}{Random inpainting} &
Ours       & \textbf{31.433} & \textbf{0.894} & \textbf{0.193} & \textbf{35.263} & \textbf{28.181} & \textbf{0.806} & \underline{0.262} & \textbf{37.396} \\
& PSLD       & 24.280 & 0.635 & 0.346 & 68.915 & 20.690 & 0.436 & 0.550 & 98.449 \\
& ReSample   & \underline{29.950} & \underline{0.842} & \underline{0.201} & 307.874 & \underline{26.916} & \underline{0.756} & \textbf{0.255} & 323.283 \\
& LatentDAPS & 26.036 & 0.743 & 0.385 & \underline{85.840} & 19.630 & 0.588 & 0.522 & \underline{86.092} \\
\midrule

\multirow{4}{*}{Gaussian deblurring} &
Ours       & \textbf{28.269} & \textbf{0.801} & \textbf{0.272} & \textbf{35.011} & \textbf{25.443} & \textbf{0.662} & \textbf{0.411} & \textbf{38.110} \\
& PSLD       & 22.015 & 0.503 & 0.444 & \underline{70.098} & 19.591 & 0.329 & 0.555 & 109.263 \\
& ReSample   & \underline{26.357} & 0.662 & 0.329 & 355.885 & \underline{23.530} & 0.497 & \underline{0.439} & 338.281 \\
& LatentDAPS & 25.717 & \underline{0.732} & 0.384 & 87.407 & 22.695 & \underline{0.567} & 0.549 & \underline{86.016} \\
\midrule

\multirow{4}{*}{Motion deblurring} &
Ours       & \textbf{29.959} & \textbf{0.840} & \textbf{0.244} & \textbf{36.185} & \textbf{27.119} & \textbf{0.730} & \textbf{0.348} & \textbf{38.118} \\
& PSLD       & 24.416 & 0.603 & 0.346 & \underline{70.041} & 20.761 & 0.376 & 0.518 & 98.400 \\
& ReSample   & \underline{28.744} & 0.754 & \underline{0.262} & 347.756 & \underline{24.845} & 0.579 & \underline{0.404} & 347.756 \\
& LatentDAPS & 26.646 & \underline{0.757} & 0.361 & 87.409 & 23.557 & \underline{0.592} & 0.513 & \underline{89.050} \\
\midrule

\multirow{3}{*}{Phase retrieval} &
Ours       & \textbf{28.194} & \textbf{0.802} & \textbf{0.271} & \textbf{133.941} & \textbf{20.133} & \textbf{0.483} & \textbf{0.458} & \textbf{119.078} \\
& ReSample   & \underline{24.676} & 0.606 & \underline{0.412} & 320.911 & 16.913 & 0.320 & \underline{0.608} & 319.601 \\
& LatentDAPS & 23.199 & \underline{0.692} & 0.421 & \underline{177.730} & \underline{17.067} & \underline{0.446} & 0.624 & \underline{192.787} \\
\midrule

\multirow{3}{*}{Nonlinear deblurring} &
Ours       & \textbf{29.243} & \textbf{0.836} & \underline{0.243} & \textbf{146.718} & \underline{25.889} & \textbf{0.716} & \underline{0.325} & \textbf{158.452} \\
& ReSample   & \underline{28.748} & \underline{0.797} & \textbf{0.236} & 843.212 & \textbf{26.047} & \underline{0.697} & \textbf{0.301} & 686.128 \\
& LatentDAPS & 25.152 & 0.726 & 0.387 & \underline{194.014} & 22.516 & 0.568 & 0.530 & \underline{198.074} \\
\midrule

\multirow{3}{*}{High dynamic range} &
Ours       & \textbf{26.245} & \underline{0.816} & \underline{0.279} & \textbf{134.060} & \textbf{28.885} & \textbf{0.887} & \textbf{0.130} & \textbf{117.544} \\
& ReSample   & \underline{25.038} & \textbf{0.822} & \textbf{0.239} & 291.372 & \underline{24.950} & \underline{0.783} & \underline{0.257} & 273.695 \\
& LatentDAPS & 20.789 & 0.630 & 0.512 & \underline{174.976} & 19.394 & 0.469 & 0.641 & \underline{180.929} \\

\bottomrule
\end{tabular}
\end{adjustbox}
\end{table*}

\begin{table}[t]
\caption{Diffusion steps ($T$) and GMRES iterations ($K$) used for each task in pixel and latent spaces.}
\label{tab:impl_details_TK}
\centering
\begin{tabular}{cccccc}
\toprule
 & Parameter &  Linear tasks & Phase retrieval & Nonlinear deblurring & High dynamic range \\
\midrule
 \multirow{2}{*}{Pixel} & $T$ & 50  & 250 & 50  & 250 \\
       & $K$ & 5   & 4   & 20  & 4   \\
\midrule
\multirow{2}{*}{Latent}  &  $T$ & 25  & 250 & 75  & 250 \\
       & $K$ & 20  & 5   & 20  & 5   \\
\bottomrule
\end{tabular}
\end{table}
\section{Experimental Details.}
\label{appendix:exp_detail}
\subsection{Implementation details.}
We use natural-image pre-trained diffusion priors in both pixel and latent spaces. Pixel-space priors are a FFHQ model \cite{chung2023dps} and an ImageNet model \cite{dhariwal2021diffusion}. Latent-space priors use an unconditional LDM-VQ4 model \cite{rombach2022high} for both FFHQ and ImageNet. For a fair comparison, we apply the same priors to all baselines. We follow \cite{karras2022edm} for the time-step discretization and noise schedule.
\paragraph{Datasets and metrics.}
We evaluate on FFHQ ($256\times256$) \cite{Karras_2019_CVPR} and ImageNet ($256\times256$) \cite{deng2009imagenet}, using 100 images from the validation set for each dataset. We report standard distortion and perceptual metrics, including PSNR, SSIM~\cite{wang2004ssim}, and LPIPS~\cite{Zhang_2018_CVPR}, along with qualitative reconstructions.
\paragraph{Baselines.}
We evaluate our method against representative state-of-the-art baselines capable of handling noisy measurement settings in both pixel and latent domains. For pixel-space, we compare our approach with DAPS \cite{Zhang_2025_CVPR}, SITCOM \cite{pmlr-v267-alkhouri25b}, DMPlug \cite{wang2024dmplug}, and DCDP \cite{li2024decoupled}. For latent-space, we benchmark against LatentDAPS, ReSample \cite{song2024solving}, and PSLD \cite{rout2023solving}.

All experiments were conducted on a single NVIDIA RTX 6000 Ada Generation GPU with batch size 1.
Unless stated otherwise, we set the data-consistency scale to $\sigma_n=0.01$ and the identity-damping parameter to $\lambda_{\mathrm{id}}=2.0$ for all tasks. Following the DAPS setting, we assume that the true measurement noise level is unknown at test time; therefore, while measurements are generated with a larger physical noise level (e.g., $\beta_{\mathbf{y}}=0.05$), we use $\sigma_n=0.01$ as a data-consistency calibration scale rather than as an estimate of the physical measurement noise. We set the number of diffusion steps $T$ and the number of GMRES iterations $K$ as summarized in Table~\ref{tab:impl_details_TK}.

\subsection{Baseline Details.}
\label{appendix:baselines}



\paragraph{DAPS~\cite{Zhang_2025_CVPR}}
We used the provided code in the official DAPS repository (\href{https://github.com/zhangbingliang2019/DAPS}{code}).
For both \texttt{DAPS} and \texttt{LatentDAPS}, we followed the DAPS paper setting and performed reconstruction via Langevin Dynamics.

\paragraph{SITCOM~\cite{pmlr-v267-alkhouri25b}}
We used the default configuration provided in the official SITCOM repository (\href{https://github.com/sjames40/SITCOM}{code}).

\paragraph{DMPlug~\cite{wang2024dmplug}}
We used the default configuration of DMPlug (\href{https://github.com/sun-umn/DMPlug}{code}).
Specifically, we employed DDIM with $T=3$ reverse steps and $\eta=0$ for all tasks, and optimized the latent seed using Adam with a learning rate of $0.01$.
The maximum number of optimization iterations was set to $5000$ for linear tasks and $10000$ for nonlinear tasks.
For early stopping, we adopted the same ES-WMV~\cite{wangearly} criterion as DMPlug, using a window size of $W=10$ and patience $P=100$ for linear tasks, and $W=50$ and $P=300$ for nonlinear tasks.

\paragraph{DCDP~\cite{li2024decoupled}.}
We used the official implementation (\href{https://github.com/morefre/decoupled-data-consistency-with-diffusion-purification-for-image-restoration}{code}). Following the DCDP paper, linear tasks were evaluated in the noiseless setting.
For noisy linear experiments, we followed the paper's recommendation to reduce the learning rate by a factor of $10$ and to increase the ending timestep $t_k$.
For inpainting, keeping the original learning rate performed better; under noise level $\beta_\mathbf{y}=0.05$, $t_k=100$ gave the best performance and was used.
For nonlinear tasks, we used the default configuration.
\begin{table}[!htp]
\centering
\caption{\textbf{DCDP hyperparameters for noisy linear tasks.}
SR: super-resolution, BIP: box inpainting, RIP: random inpainting, GB: Gaussian blur, MB: motion blur.}
\label{tab:dcdp_hparams}
\setlength{\tabcolsep}{5pt}
\begin{tabular}{lccccc}
\toprule
\textbf{Param (DCDP)} & \textbf{SR} & \textbf{BIP} & \textbf{RIP} & \textbf{GB} & \textbf{MB} \\
\midrule
$t_k$   & 100 & 100 & 100 & 100 & 100 \\
$\alpha$ & $10^{2}$ & $10^{3}$ & $10^{3}$ & $10^{5}$ & $10^{5}$ \\
\bottomrule
\end{tabular}
\end{table}

\paragraph{ReSample~\cite{song2024solving}}
We used the default configuration from the official ReSample repository (\href{https://github.com/soominkwon/resample}{code}).

\paragraph{PSLD~\cite{rout2023solving}}
We used the official PSLD implementation (\href{https://github.com/LituRout/PSLD}{code}).
However, the provided configuration was unstable in our setup; therefore, we tuned $(\eta,\gamma)$ based on FFHQ and applied the task-specific settings below.
\begin{table}[H]
\caption{\textbf{PSLD hyperparameters used in our experiments (FFHQ-based tuning).}
SR: super resolution 4$\times$, BIP: box inpainting, RIP: random inpainting, GB: Gaussian blurring, MB: motion blurring.}
\label{tab:psld_hparams}
\centering
\setlength{\tabcolsep}{5pt}
\begin{tabular}{lccccc}
\toprule
\textbf{Param (PSLD)} & \textbf{SR} & \textbf{BIP} & \textbf{RIP} & \textbf{GB} & \textbf{MB} \\
\midrule
$\eta$   & $1.0$ & $0.1$ & $0.1$ & $0.1$ & $0.1$ \\
$\gamma$ & $0.1$ & $10^{-3}$ & $10^{-2}$ & $10^{-3}$ & $10^{-2}$ \\
\bottomrule
\end{tabular}
\end{table}

\section{Additional Experiments}

\paragraph{Latent-space experiments.}
Table~\ref{tab:main_table_latent} summarizes latent-space inverse-problem results on FFHQ and ImageNet.
Overall, our method consistently achieves strong reconstruction quality while being markedly faster than prior latent-space baselines, indicating that the proposed acceleration does not sacrifice fidelity.
Notably, we remain robust on blur families and also handle the notoriously ill-posed phase retrieval task reliably in the latent setting.
We additionally observe particularly large gains on high dynamic range reconstruction, where our method shows a clear margin over competing approaches across datasets. These results suggest that our latent-space solver generalizes well across diverse degradations, including highly ill-conditioned regimes, while maintaining practical runtimes.

\begin{table}[t]
\centering
\caption{
We evaluate four design choices---noise-level transition (variance-preserving Langevin vs.\ ODE-style update), prior-aligned damping, curvature operator (one-sided, without \(\mathbf{J}_F^{\!\top}\), and full curvature with \(\mathbf{J}_F\)), and linear solver (GMRES vs.\ BiCG/CG)---on super-resolution and nonlinear deblurring. 
Compute is matched across variants by fixing the number of diffusion steps and the number of Krylov/solver iterations.}
\label{tab:ablation}
\setlength{\tabcolsep}{4.5pt}
\renewcommand{\arraystretch}{1.15}
\begin{tabular}{l cccc cccc}
\toprule
\multirow{2}{*}{\textbf{Variant}} &
\multicolumn{4}{c}{\textbf{Super resolution 4$\times$}} &
\multicolumn{4}{c}{\textbf{Nonlinear deblurring}} \\
\cmidrule(lr){2-9}
 &
{PSNR} & {SSIM} & {LPIPS} & {Time} &
{PSNR} & {SSIM} & {LPIPS} & {Time} \\
\midrule
Vanilla
& 22.982 & 0.390 & 0.496 & 1.696
& 22.644 & 0.532 & 0.364 & 18.294 \\

w/ODE sampling
& 28.767 & 0.813 & 0.230 & 6.346
& 28.997 & 0.818 & 0.194 & 30.106 \\

w/o damping
& 26.391 & 0.628 & 0.414 & 6.291
& 29.976 & 0.823 & 0.177 & 31.053 \\

w/o $\mathbf{J}_F^\top$
& 28.501 & 0.758 & 0.334 & 1.669
& 24.848 & 0.643 & 0.339 & 18.127 \\

w/ $\mathbf{J}_F$
& 29.798 & 0.841 & 0.227 & 31.985
& 30.420 & 0.843 & 0.170 & 132.596 \\
GMRES$\rightarrow$BiCG
& 29.806 & 0.841 & 0.228 & 11.107
& 26.796 & 0.740 & 0.284 & 74.102 \\
\midrule

Proposed
& 29.807 & 0.841 & 0.229 & 6.370
& 30.524 & 0.859 & 0.170 & 30.286\\
\bottomrule
\end{tabular}
\end{table}

\label{appendix:addtional_exp}
\paragraph{Ablation Studies.}
To justify our design choices, we conduct a targeted study of four components: (i) how we transition to the next noise level, (ii) whether prior-aligned damping is applied, (iii) the effect of the one-sided curvature approximation, and (iv) the choice of linear solver. 
For the noise-level transition, we compare our default variance-preserving Langevin propagation against an EDM-style ODE update, where no additional stochastic noise is injected and the iterate is propagated solely using the model prediction. 
To evaluate the one-sided curvature approximation, we consider three curvature operators. Our default uses the one-sided form \(\mathbf{J}_F^{\!\top} \mathbf{J}_A^{\!\top}(\mathbf{J}_A\,\cdot)\). Removing the one-sided approximation yields \(\mathbf{J}_A^{\!\top}(\mathbf{J}_A\,\cdot)\), while the full form corresponds to \(\mathbf{J}_F^{\!\top} \mathbf{J}_A^{\!\top}(\mathbf{J}_A(\mathbf{J}_F\,\cdot))\). We compare these three variants in a controlled setting. 
For solving the resulting linear systems, our default employs GMRES. As a variant, we compare against BiCG \cite{fletcher1976conjugate}, a Krylov alternative for nonsymmetric systems---relevant here because the one-sided approximation makes the system matrix \(B\) nonsymmetric. We additionally include a plain CG solver as a further baseline. 
Finally, we compare against a ``vanilla'' configuration that uses CG with \(\mathbf{J}_A^{\!\top}(\mathbf{J}_A\,\cdot)\), removes damping, and adopts the ODE update. 
For a fair comparison, we equalize compute across methods by matching both the number of solver steps and the number of diffusion steps.

\begin{figure}
\centering
  \makebox[\textwidth][c]{%
    \includegraphics[width=\linewidth]{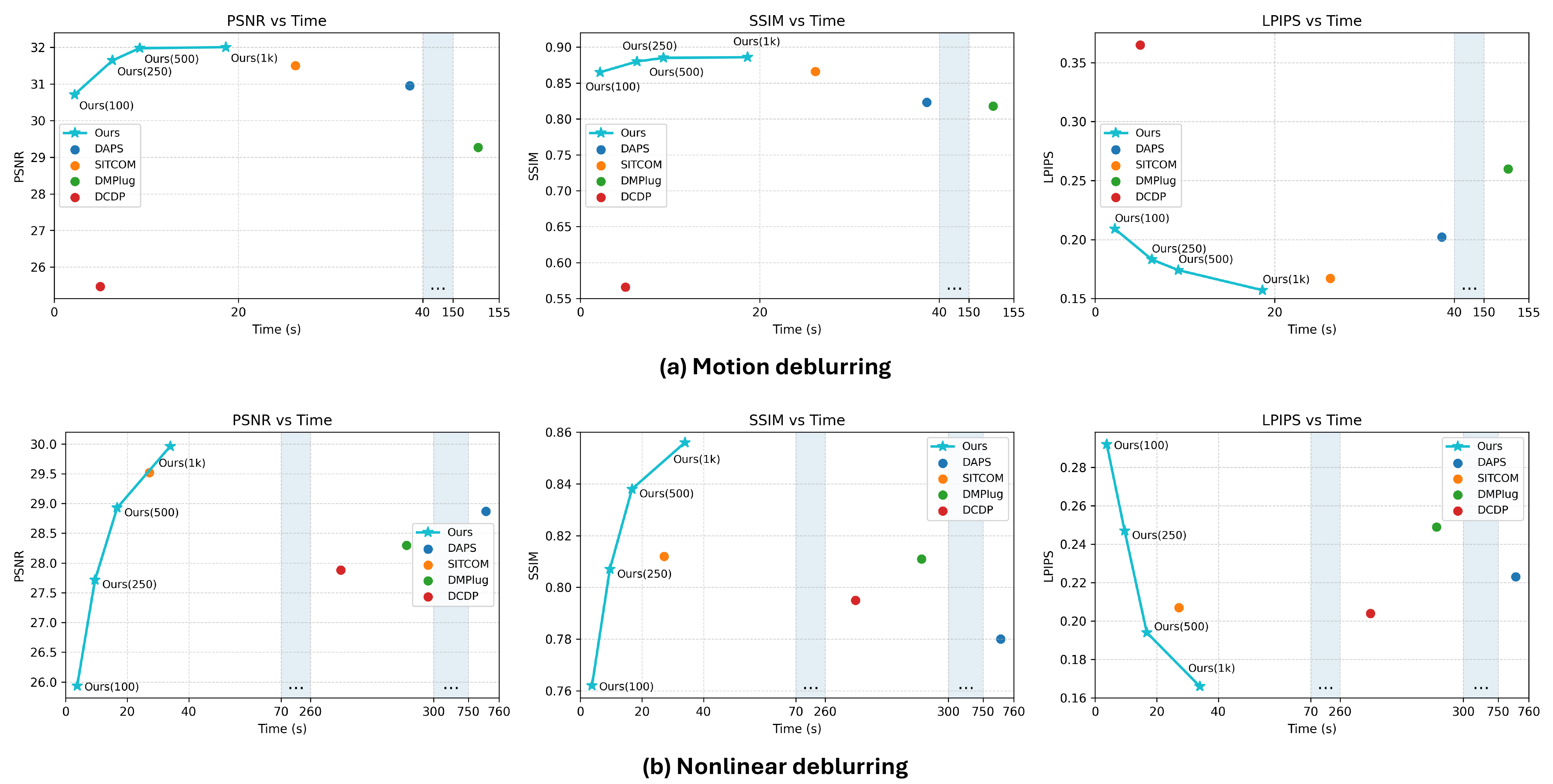}
  }
\caption{PSNR/SSIM/LPIPS versus wall-clock time for (a) motion deblurring and (b) nonlinear deblurring.
Our method traces a Pareto curve as we vary the compute budget, where the budget for ours is measured by the product of the number of diffusion time steps ($T$) and the number of GMRES iterations ($K$); the evaluated $(T,K)$ configurations are listed in Table~\ref{tab:time_plot_table}.
}
\label{fig:time_plot}
\end{figure}

Across a representative linear task (super-resolution) and a nonlinear task (nonlinear deblurring), we observe consistent trends.
First, replacing the proposed sampling transition with an ODE-style update degrades performance.
Second, removing prior-aligned damping underperforms in both settings, indicating that the damping is important beyond the specific task type.
Third, dropping the \(\mathbf{J}_F^{\!\top}\) factor reduces runtime by approximately \(3.8\times\) on super-resolution, but incurs a clear performance trade-off, most notably on nonlinear deblurring.
Conversely, using the exact curvature by including \(\mathbf{J}_F\) increases runtime by about \(5.0\times\) on super-resolution and \(4.4\times\) on nonlinear deblurring, while providing nearly identical or only marginally different performance.
These results indicate that the one-sided approximation achieves a more favorable accuracy--efficiency trade-off, preserving performance at significantly lower computational cost.

Regarding the linear solver choice, CG is fundamentally mismatched to our setting since it assumes a symmetric (typically SPD) system, whereas the one-sided curvature leads to a nonsymmetric matrix \(\mathbf{B}\).
This motivates BiCG as a more appropriate Krylov alternative for nonsymmetric systems.
However, BiCG (and related CG-style variants) can still exhibit numerical instability in practice when the effective system becomes poorly conditioned, especially on challenging nonlinear tasks such as nonlinear deblurring.
In particular, increasing the number of iterations may amplify nonnormal effects and round-off errors, causing the iterates to become unstable and the reconstruction quality to deteriorate.
These observations further support GMRES as a robust default for our nonsymmetric curvature operator.

\paragraph{Evaluation on runtime and quality.}
In Figure~\ref{fig:time_plot}, we evaluate the compute--quality trade-off on (a) motion deblur and (b) nonlinear deblur by plotting PSNR/SSIM/LPIPS against wall-clock time.
Overall, \textsc{clamp} exhibits a favorable Pareto frontier, achieving improvements across all metrics (PSNR/SSIM/LPIPS) at comparable or lower runtime.

\begin{table}[t]
    \caption{In Figure \ref{fig:time_plot}, we report the diffusion step count $T$ and the number of GMRES iterations per step $K$ used for \textsc{clamp} at four compute budgets (Steps $=T\times K\in\{100,250,500,1\text{k}\}$) on motion deblurring and nonlinear deblurring.}
    \label{tab:time_plot_table}
    \centering
    \begin{tabular}{l c cccc}
    \toprule
         \textbf{Task}&\textbf{Step}& 100& 250 & 500 & 1k  \\
         \midrule
         \multirow{2}{*}{Motion deblurring}&$T$&20&50&50&100\\
         &$K$&5&5&10&10\\
         \midrule
         \multirow{2}{*}{Nonlinear deblurring}&$T$&10&25&25&50\\
         &$K$&10&10&20&20\\
         \bottomrule
    \end{tabular}

\end{table}

\paragraph{High-Resolution Experiments (512$\times$512).}
\begin{figure}[t]
\centering
  \makebox[\textwidth][c]{%
    \includegraphics[width=0.6\linewidth]{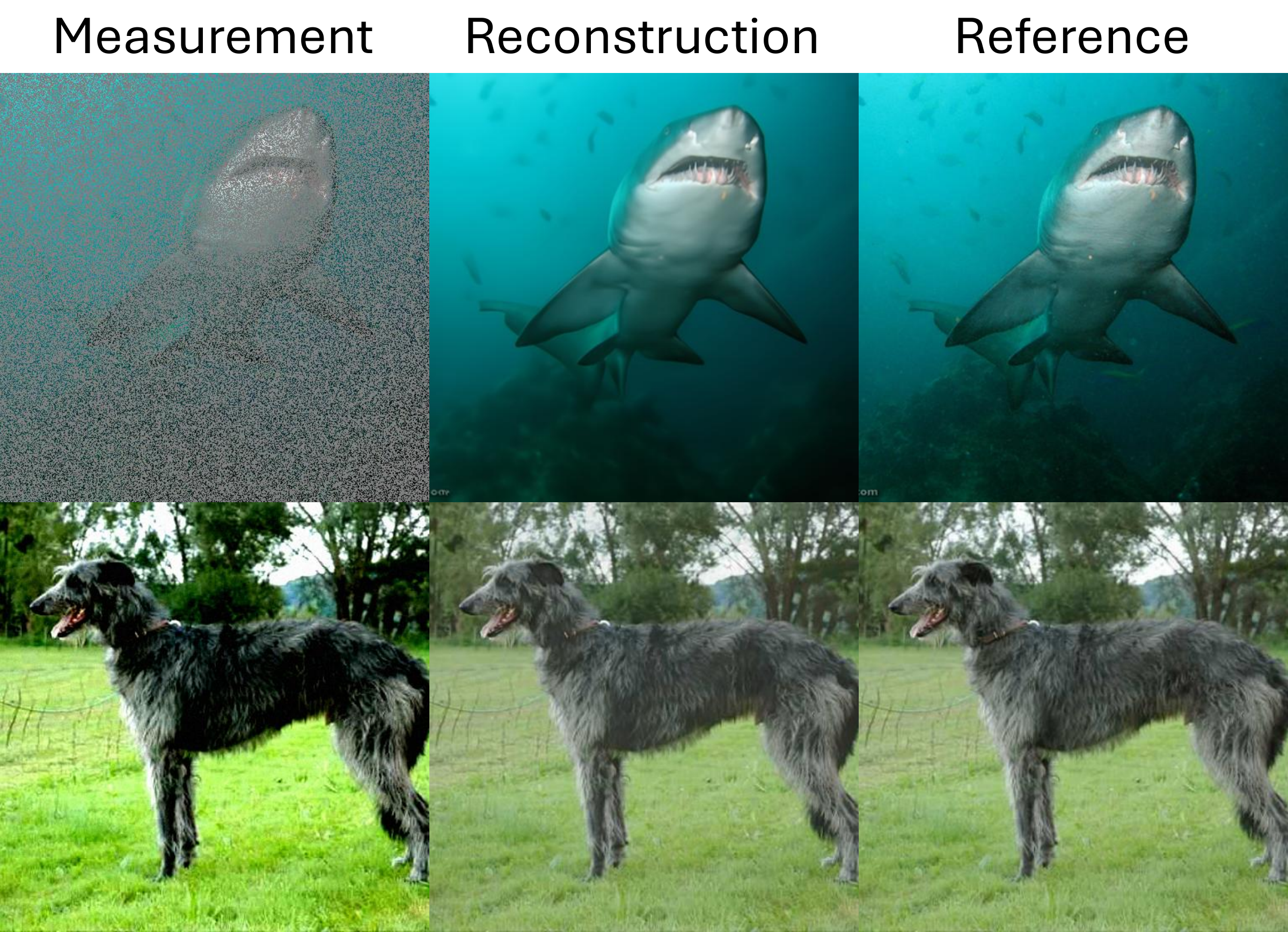}
  }
\caption{Visualizations at $512\times512$ resolution on ImageNet.
\textbf{Top:} random inpainting. \textbf{Bottom:} high dynamic range.}
\label{fig:512}
\end{figure}
\begin{table}[t]
\caption{$512\times512$ ImageNet results with a class-conditional diffusion prior (offline predicted labels; no classifier guidance).
Random Inpainting (linear) and HDR (nonlinear): PSNR/SSIM/LPIPS and run-time. Best/second-best are \textbf{bold}/\underline{underlined} (10 samples, batch size 1).}
\label{tab:high_resolution}
\centering
\begin{tabular}{ccccccccc}
    \toprule
    \multirow{2}{*}{\textbf{Method}} & \multicolumn{4}{c}{\textbf{Random Inpainting}} & \multicolumn{4}{c}{\textbf{High dynamic range}} \\
    \cmidrule(lr){2-5}\cmidrule(lr){6-9}
    & \textbf{PSNR} $\uparrow$ & SSIM $\uparrow$ & \textbf{LPIPS} $\downarrow$ & \textbf{Run-time (s)} & \textbf{PSNR} $\uparrow$ & \textbf{SSIM} $\uparrow$ & \textbf{LPIPS} $\downarrow$ & \textbf{Run-time (s)} \\
    \midrule
    Ours  & \textbf{35.646} & \textbf{0.920} & \textbf{0.212} & \textbf{40.897}  & \textbf{28.845} & \textbf{0.858} & \textbf{0.192} & \textbf{176.777} \\
    DAPS   & 31.736 & 0.761 & 0.313 & \underline{119.487}   & \underline{28.003} & \underline{0.815} & \underline{0.255} & 469.332 \\
    SITCOM & \underline{34.047} & \underline{0.868} & \underline{0.230} & 187.028  & 25.361 & 0.712 & 0.322 & \underline{254.260} \\
    \bottomrule
\end{tabular}

\end{table}
To assess the scalability of our method to higher resolutions, we conduct additional experiments at 512$\times$512 using the publicly available ImageNet class-conditional diffusion prior~\cite{dhariwal2021diffusion}. 
Since this prior requires a class label, we obtain a label for each validation image by running a fixed pretrained ImageNet classifier once offline and treating the predicted label as a constant condition thereafter.
Importantly, the classifier is not used during sampling (i.e., we do not employ classifier guidance); the sampling procedure only queries the class-conditional diffusion model with the precomputed label for each image.
Under this protocol, we compare our method against DAPS and SITCOM on one representative linear task and one representative nonlinear task, using the same 512$\times$512 diffusion prior and the same precomputed labels for all methods. Table~\ref{tab:high_resolution} shows that our method remains competitive in reconstruction quality at higher resolution while providing substantial speedups over baselines: it is 4.57$\times$ faster than SITCOM on random inpainting and 2.65$\times$ faster than DAPS on HDR, while maintaining higher PSNR/SSIM and lower LPIPS compared to competing methods. Reconstruction samples are provided in Figure \ref{fig:512}, demonstrating that our method reconstructs images close to the originals while preserving fine details and naturalness even at higher resolutions.

\paragraph{Hyperparameter analysis.}
\begin{table}[t]
\centering
\caption{Effect of the identity-damping parameter $\lambda_{\mathrm{id}}$: we sweep $\lambda_{\mathrm{id}}$ and report PSNR/SSIM/LPIPS on random inpainting and nonlinear blur.}
\label{apdix:lambda_sweep}
\setlength{\tabcolsep}{6pt}
\begin{tabular}{lcccccc}
\toprule
\multirow{2}{*}{$\lambda_{\mathrm{id}}$}
& \multicolumn{3}{c}{\textbf{Random inpainting}}
& \multicolumn{3}{c}{\textbf{Nonlinear deblurring}} \\
\cmidrule(lr){2-4}\cmidrule(lr){5-7}
& \textbf{PSNR $\uparrow$} & \textbf{SSIM $\uparrow$} &
\textbf{LPIPS $\downarrow$} & \textbf{PSNR $\uparrow$} & \textbf{SSIM $\uparrow$} &
\textbf{LPIPS $\downarrow$} \\
\midrule
0   & 33.081 & 0.897 & 0.144 & 30.012 & 0.825 & 0.175 \\
0.1 & 33.165 & 0.899 & 0.144 & 30.048 & 0.828 & 0.173 \\
2   & 33.249 & 0.901 & 0.202 & 30.524 & 0.859 & 0.170 \\
10  & 31.873 & 0.863 & 0.285 & 29.654 & 0.842 & 0.218 \\
100 & 27.482 & 0.754 & 0.415 & 26.765 & 0.776 & 0.294 \\
\bottomrule
\end{tabular}
\end{table}

We analyze the identity-damping parameter $\lambda_{\mathrm{id}}$ in the prior-aligned metric
$\mathbf{H}_t=\lambda_{\mathrm{id}}\mathbf{I}+\mathbf{u}_t\mathbf{u}_t^\top$ used by the correction system
$(\mathbf{B}_t+\alpha_t \mathbf{H}_t){\Delta}_t=-\mathbf{c}_t$.
Here $\mathbf{B}_t$ denotes the one-sided GN/Fisher curvature operator and $\alpha_t$ is the diffusion-calibrated scaling that adjusts the overall damping strength across noise levels.
The isotropic term $\lambda_{\mathrm{id}}\mathbf{I}$ provides baseline regularization in all directions (and is the only regularization on $\mathbf{u}_t^\perp$, where the rank-one term vanishes), while $\mathbf{u}_t\mathbf{u}_t^\top$ adds additional damping along the denoiser-residual direction.

Table~\ref{apdix:lambda_sweep} shows that the effect of $\lambda_{\mathrm{id}}$ is task dependent.
For random inpainting, which is intrinsically one-to-many, smaller $\lambda_{\mathrm{id}}$ yields substantially better perceptual quality (lower LPIPS), consistent with allowing more freedom in $u_t^\perp$ while maintaining measurement consistency; larger $\lambda_{\mathrm{id}}$ progressively over-regularizes the correction and degrades all metrics.
For nonlinear blur, moderate $\lambda_{\mathrm{id}}$ improves stability and achieves the best overall metrics in this sweep (we use $\lambda_{\mathrm{id}}=2$), whereas very large $\lambda_{\mathrm{id}}$ again over-regularizes the update and leads to under-correction and worse performance.

\paragraph{Severe degradation settings.}
In Figure \ref{fig:severe}, we visualize sample diversity under severe degradation on FFHQ $256{\times}256$.
We consider two highly ill-posed settings: $16{\times}$ super-resolution (top two rows) and $192{\times}192$ box inpainting (bottom two rows).
For each measurement, we draw three independent runs (Sample1--3).
Despite the measurement being heavily corrupted or largely missing, our method produces natural-looking, measurement-consistent reconstructions while exhibiting meaningful diversity across runs (e.g., facial attributes and expressions), suggesting effective sampling of a multi-modal solution space.

\begin{figure}[t]
\centering
  \makebox[\textwidth][c]{%
    \includegraphics[width=0.6\linewidth]{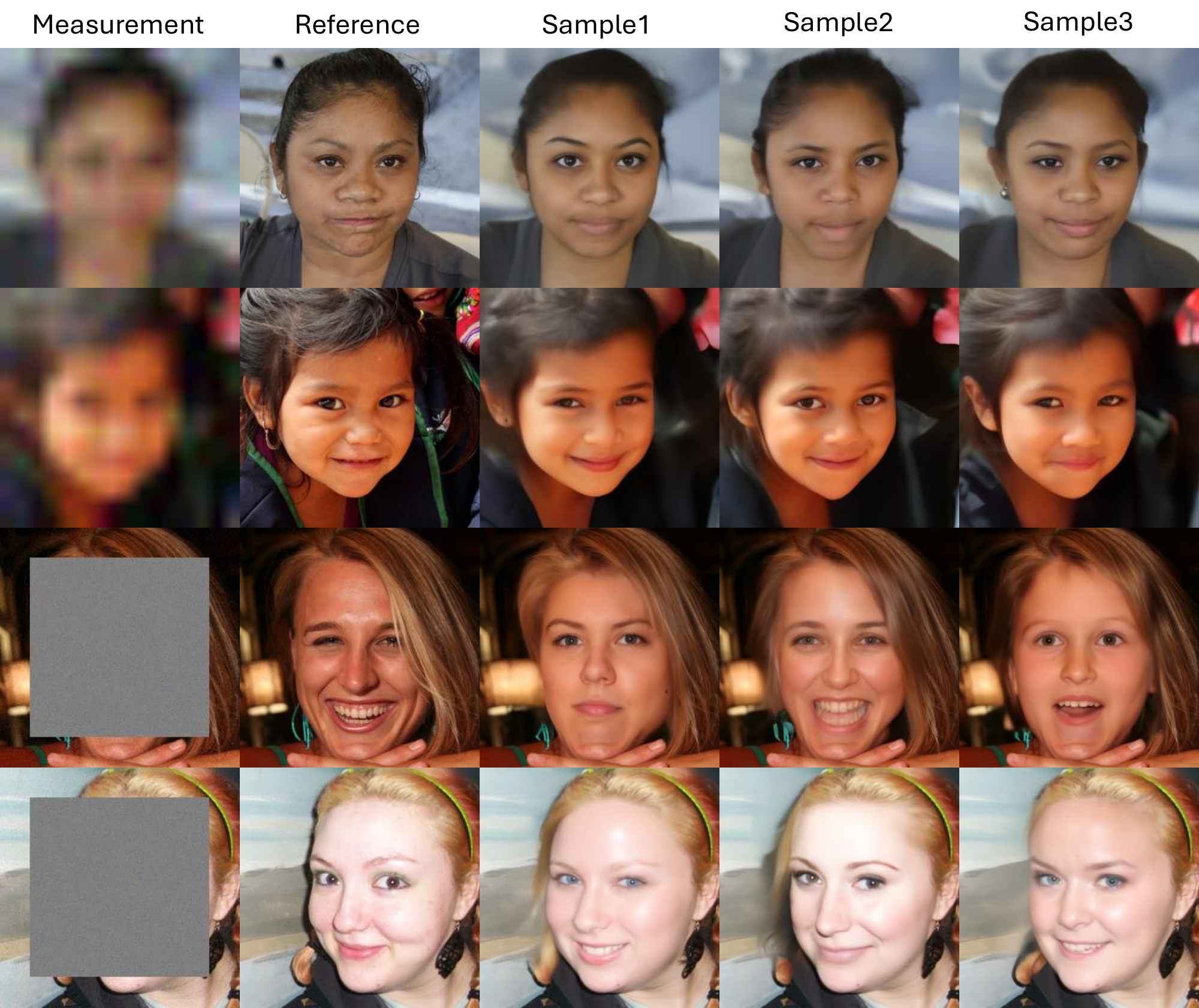}
  }
\caption{\textbf{Sample diversity under severe degradation.} \textit{Top two rows:} 16$\times$ super-resolution. \textit{Bottom two rows:} 192$\times$192 box inpainting. Our method yields diverse, measurement-consistent reconstructions by effectively sampling from a multi-modal posterior.}
\label{fig:severe}
\end{figure}

\paragraph{Various measurement noise levels.}
\begin{table}[t]
    \centering
    \tiny
    \caption{\textbf{Effect of measurement noise.}
    We evaluate robustness to different measurement noise levels on motion deblurring in pixel space and nonlinear deblurring in latent space.
    For \textbf{Gaussian} noise, we add i.i.d.\ noise to the measurement, $\mathbf{y}=A(\mathbf{x})+\mathbf{n}$ with $\mathbf{n}\sim\mathcal{N}(\mathbf{0},\beta_{\mathbf{y}}^2\mathbf{I})$, sweeping $\beta_{\mathbf{y}}\in\{0.01,\,0.05,\,0.1\}$.
    For \textbf{Poisson} noise, the Poisson level $\lambda$ controls the effective photon budget (larger $\lambda$ indicates higher expected counts and thus higher SNR).}

    \label{tab:various_noise}
    \begin{tabular}{lcccccc|lcccccc}
    \toprule
         \multirow{2}{*}{\textbf{Gaussian} ($\beta_{\mathbf{y}}$)} & \multicolumn{3}{c}{\textbf{Motion deblurring (pixel)}} & \multicolumn{3}{c|}{\textbf{Nonlinear deblurring (latent)}}  & \multirow{2}{*}{\textbf{Poisson} ($\lambda$)} &  \multicolumn{3}{c}{\textbf{Motion deblurring (pixel)}} & \multicolumn{3}{c}{\textbf{Nonlinear deblurring (latent)}} \\
             \cmidrule(lr){2-7}\cmidrule(lr){9-14}
         & \textbf{PSNR $\uparrow$} & \textbf{SSIM $\uparrow$} &
\textbf{LPIPS $\downarrow$} & \textbf{PSNR $\uparrow$} & \textbf{SSIM $\uparrow$} &
\textbf{LPIPS $\downarrow$} & & \textbf{PSNR $\uparrow$} & \textbf{SSIM $\uparrow$} &
\textbf{LPIPS $\downarrow$} & \textbf{PSNR $\uparrow$} & \textbf{SSIM $\uparrow$} &
\textbf{LPIPS $\downarrow$} \\
         \midrule
         0.01 & 32.217 & 0.885& 0.200 & 30.201 & 0.848 & 0.259 & 3.0 & 31.945 & 0.881 & 0.195 & 29.737 & 0.841 & 0.254 \\
         0.05 (base) &  31.913 & 0.880 & 0.194 & 29.793 & 0.842 & 0.252 & 1.0 & 31.210 & 0.867 & 0.193 & 28.820 & 0.822 & 0.250 \\
         0.1 &   30.647 & 0.853 & 0.196 & 28.108 & 0.799 & 0.251 & 0.5 & 29.724 & 0.823 & 0.227 & 27.241 & 0.766 & 0.295 \\
    \bottomrule
    \end{tabular}
\end{table}
We further test robustness by varying only the measurement noise model/level while keeping the forward operator, compute budget, and data-consistency scale $\sigma_n$ fixed.
For Gaussian corruption, we add i.i.d.\ additive noise in the measurement domain with physical standard deviation $\beta_{\mathbf{y}}\in\{0.01,\,0.05,\,0.1\}$, where $\beta_{\mathbf{y}}=0.05$ is used as our default setting.
For Poisson corruption, we follow the Poisson simulation protocol of DPS \cite{chung2023dps}: measurements are interpreted as photon counts proportional to discretized intensity, noisy counts are sampled from a Poisson distribution, and the result is rescaled back to the normalized range.
Here, the reported Poisson level $\lambda\in\{3.0,\,1.0,\,0.5\}$ acts as a scale factor (photon-rate / gain) that modulates the expected counts, producing signal-dependent shot noise.

Across both noise families, our method remains stable without qualitative degradation and shows graceful performance decay as the noise increases, indicating that the proposed update is not overly sensitive to the measurement noise level and is reliable under a broad range of measurement conditions.

\paragraph{Normalized residual check.}
A core challenge in training-free diffusion posterior sampling is to ensure that the generated samples are not only visually plausible
but also statistically consistent with the assumed measurement model.
To this end, we evaluate a simple yet informative calibration diagnostic in measurement space.

We define the noise-normalized residual statistic as
\begin{equation}
  r_s \;:=\; \frac{1}{m}\,\frac{\| \mathbf{y} - A(\mathbf{x})\|_2^2}{\beta_{\mathbf{y}}^{2}}
  \label{eq:rs}
\end{equation}
where \(m\) is the number of measurement entries. It measures the average squared misfit in units of the physical measurement-noise standard deviation $\beta_{\mathbf{y}}$
(assuming i.i.d.\ Gaussian noise across measurements).

If samples are well-calibrated with respect to the likelihood under the measurement noise level \(\beta_\mathbf{y}\), then
\(\|\mathbf{y}-A(\mathbf{x})\|_2^2/\beta_{\mathbf{y}}^2\) is on the order of a \(\chi^2_m\) random variable, implying that \(r_s\) should concentrate around \(1\).
Deviations reveal systematic mis-calibration:
\(r_s \gg 1\) indicates under-enforcing measurement consistency (insufficient likelihood correction),
whereas \(r_s \ll 1\) indicates over-enforcing data consistency (overconfident and potentially over-smoothed samples).

Figure~\ref{fig:rs_histo} reports the empirical distribution of \(r_s\) for the box inpainting operator,
and Table~\ref{tab:nom_res} summarizes the mean \(r_s\) across tasks on FFHQ dataset.
\textsc{clamp} yields residual distributions most tightly centered near the ideal value \(1\) and achieves mean \(r_s\) closest to \(1\) on the majority of tasks,
providing evidence that \textsc{clamp} produces samples that are better calibrated in measurement space under the same physical measurement-noise level.
While this diagnostic does not fully characterize posterior uncertainty, it serves as a practical sanity check that complements reconstruction metrics by directly probing sample-level consistency with the observation model.

\begin{figure}[t]
\centering
  \makebox[\textwidth][c]{%
    \includegraphics[width=\linewidth]{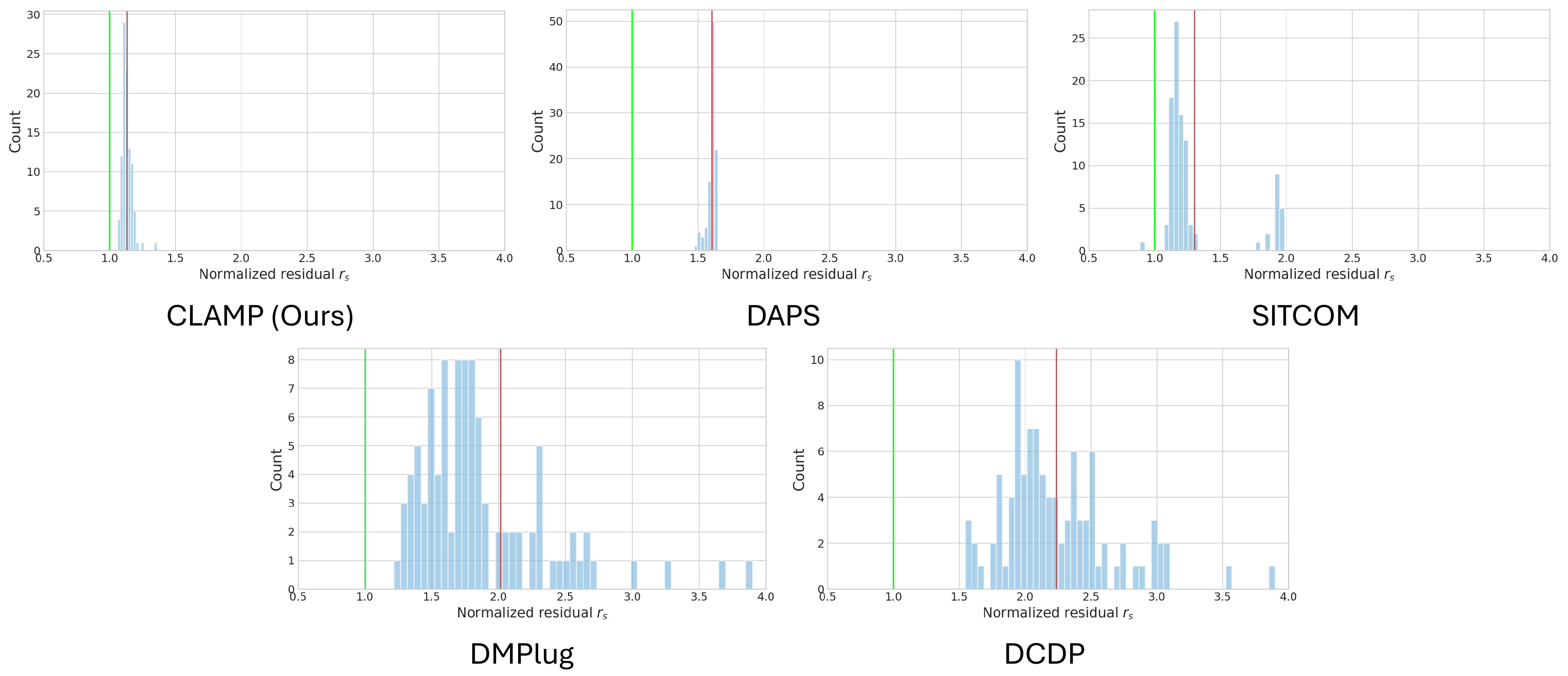}
  }
  \caption{\textbf{Measurement-space calibration on Box Inpainting.}
  Histogram of the normalized measurement residual
  \(
    r_s := \|\mathbf{y} - A(\mathbf{x})\|_2^2 / (\beta_\mathbf{y}^2\, m)
  \)
  for the \emph{box inpainting} operator, where \(\mathbf{y}\) is the masked observation, \(A(\cdot)\) applies the same box mask,
  $\beta_\mathbf{y}$ is the measurement-noise standard deviation, and \(m\) is the number of measurement entries.
  The red line denotes the empirical mean and the green line indicates the ideal value \(r_s=1\).}
\label{fig:rs_histo}
\end{figure}
\begin{table}[!htp]
\small
\centering
\caption{\textbf{Measurement-space calibration (normalized residual).} We report the mean measurement residual normalized by the measurement noise level for posterior samples across inverse problems. Values closer to $1$ indicate better calibration/coverage in measurement space; systematic deviations imply over- or under-enforcing data consistency. Columns correspond to different operators/tasks (SR, BIP, RIP, GB, MB, PR, NB, HDR).}
\label{tab:nom_res}
\setlength{\tabcolsep}{5pt}
\begin{tabular}{lcccccccc}
\toprule
\textbf{Method} & \textbf{SR} & \textbf{BIP} & \textbf{RIP} & \textbf{GB} & \textbf{MB} & \textbf{PR} & \textbf{NB} & \textbf{HDR} \\
\midrule
Ours   & \underline{1.272} & \textbf{1.131} & \textbf{1.078} & \textbf{1.030} & \textbf{1.052} & \textbf{1.045} & \textbf{1.118} & \textbf{1.412} \\
DAPS   & 1.361 & 1.605 & 1.285 & \underline{1.031} & 1.071 & \underline{1.054} & 1.290 & 2.391 \\
SITCOM & 1.364 & \underline{1.302} & \underline{1.114} & 1.041 & 1.059 & 1.056 & 1.243 & \underline{1.528} \\
DMPlug & \textbf{1.117} & 2.014 & 1.344 & 1.072 & 1.216 &   -  & \underline{1.190}  & 3.651 \\
DCDP   & 1.676 & 2.237 & 1.955 & 1.174 & 1.760  & 1.082 & 1.528 &  -    \\
\bottomrule
\end{tabular}
\end{table}

\paragraph{Comparison for linear diffusion solvers.}
\begin{table*}[t]
\centering

\begin{minipage}[t]{0.64\textwidth}
\centering
\caption{
Comparison with specialized diffusion solvers on FFHQ linear pixel-space inverse problems.
We report average PSNR/SSIM, LPIPS, and wall-clock runtime under the same noisy-measurement protocol. 
\textsc{clamp} uses the task-specific settings in Table~14 to provide a comparable runtime--quality trade-off against specialized linear solvers.
}
\label{tab:linear_solvers}
\small
\begin{tabular}{cccccc}
\toprule
\textbf{Task} & \textbf{Method} &
\textbf{PSNR $\uparrow$} & \textbf{SSIM $\uparrow$} &
\textbf{LPIPS $\downarrow$} & \textbf{Run-time (s)} \\
\midrule

\multirow{3}{*}{Super resolution 4$\times$} &
Ours       & \textbf{28.37} & \textbf{0.813} & \textbf{0.251} & \underline{1.7} \\
& DDRM     & 28.01 & 0.792 & \underline{0.257}& \textbf{0.5} \\
& DDPG     & \underline{28.22} & \underline{0.809} & 0.265 & 1.9 \\
\midrule

\multirow{4}{*}{Box inpainting} &
Ours       & \textbf{25.12} &  \textbf{0.853} & \textbf{0.173} & \underline{1.7} \\
& DDRM     & 22.31 &  0.800 & 0.206 & \textbf{0.6}\\
& DDNM     & 24.05 &  \underline{0.849} & \underline{0.175} & \underline{1.7} \\
& DDPG     & \underline{24.88} &  0.820 & 0.229 & 1.9 \\
\midrule

\multirow{4}{*}{Random inpainting} &
Ours       & \textbf{31.46} & \textbf{0.884} & \textbf{0.186} & \underline{1.1} \\
& DDRM     & 25.79 & 0.760 & 0.268 & \textbf{0.5} \\
& DDNM     & 29.80 & 0.861 & \underline{0.192} & 1.7 \\
& DDPG     & \underline{30.43} & \underline{0.870} & 0.225 & 1.9 \\
\midrule

\multirow{3}{*}{Gaussian deblurring} &
Ours       & \underline{29.96} & \textbf{0.845} & \underline{0.221}& 2.5 \\
& DDRM     & \textbf{30.08} & 0.836 & 0.228 & \textbf{0.5} \\
& DDPG     & 29.75 & \underline{0.841} & \textbf{0.203} & 1.9 \\
\midrule

\multirow{2}{*}{Motion deblurring} &
Ours       & \underline{27.89} & \textbf{0.802} & \underline{0.255} & \underline{4.2} \\
& DDPG     & \textbf{27.98} & \textbf{0.802} & \textbf{0.236} & \textbf{2.6} \\
\bottomrule
\end{tabular}
\end{minipage}
\hfill
\begin{minipage}[t]{0.34\textwidth}
\vspace{5em}
\centering
\caption{
\textsc{clamp} settings used in Table~13.
$T$ is the number of diffusion steps, $K$ is the number of GMRES iterations per step, and $\lambda_{\mathrm{id}}$ is the identity-damping coefficient.
}
\label{tab:linear_clamp_settings}
\begin{tabular}{c|ccc}
\toprule
\textbf{Task} & $T$ & $K$ & $\lambda_\mathrm{id}$ \\
\midrule
Super resolution 4$\times$ & 25 & 2 & 3.0 \\
Box inpainting             & 15 & 5 & 2.0 \\
Random inpainting          & 12 & 4 & 2.0 \\
Gaussian deblurring        & 25 & 4 & 1.0 \\
Motion deblurring          & 34 & 6 & 2.0 \\
\bottomrule
\end{tabular}
\end{minipage}

\end{table*}

We additionally compare \textsc{clamp} with diffusion solvers specialized for linear pixel-space inverse problems, including DDRM, DDNM, and DDPG, on FFHQ linear tasks. 
For \textsc{clamp}, we use the task-specific compute settings reported in Table~\ref{tab:linear_clamp_settings}, with relatively small diffusion-step and Krylov budgets chosen to yield a runtime range comparable to these specialized solvers.

As shown in Table~\ref{tab:linear_solvers}, \textsc{clamp} is competitive across the linear tasks. 
On super-resolution and both inpainting tasks, \textsc{clamp} achieves the best reconstruction quality among the compared methods, improving PSNR, SSIM, and LPIPS while maintaining comparable runtime. 
For example, \textsc{clamp} obtains the strongest scores on box inpainting and random inpainting, where it outperforms DDRM, DDNM, and DDPG in all three quality metrics. 
On Gaussian deblurring, the result is mixed: DDRM attains slightly higher PSNR and DDPG obtains lower LPIPS, while \textsc{clamp} gives the best SSIM and remains close in the other metrics. 
On motion deblurring, DDPG is faster and slightly better in PSNR/LPIPS, while \textsc{clamp} matches its SSIM.

Overall, Table~\ref{tab:linear_solvers} shows that modest task-specific compute tuning can make \textsc{clamp} competitive with strong linear-specialized diffusion solvers. 
At the same time, the comparison also clarifies the scope of our claim: \textsc{clamp} is not intended to be the fastest or best solver for every linear pixel-space problem.
Its main advantage is that the same denoiser-pullback, curvature-guided correction applies beyond this specialized setting, including nonlinear tasks and latent-prior inverse problems, where SVD-based or pixel-space preconditioned linear solvers are not directly applicable.

\paragraph{Large-update frozen-step diagnostic.}
Equation~\ref{eq:linearize_rt} relies on a local residual linearization. The relevant failure mode
is therefore not the magnitude of $\|{\Delta}_t\|_2$ alone, but whether the
realized CLAMP update remains close to the best step for the same frozen local
problem in the regime where the linearization is most stressed. To test this
directly, we rank solver-produced updates by
\[
\eta_t =
\frac{\|\mathbf{J}_t{\Delta}_t\|_2}
{\|\mathbf{r}_t\|_2 + \varepsilon},
\]
keep the top \(25\%\), and compare the actual CLAMP step ${\Delta}_t$ to
${\Delta}_t^\star$, the optimizer of the exact frozen local subproblem in the
same GMRES Krylov subspace. This comparison includes both the residual
linearization and the one-sided curvature approximation, while isolating
local-model mismatch from solver-budget effects.

We report the relative objective gap
\[
\mathrm{objgap}_t =
\frac{
\left|\mathcal{E}^{\mathrm{frz}}_t({\Delta}_t)
-
\mathcal{E}^{\mathrm{frz}}_t({\Delta}_t^\star)\right|
}{
\left|\mathcal{E}^{\mathrm{frz}}_t(\mathbf{0})
-
\mathcal{E}^{\mathrm{frz}}_t({\Delta}_t^\star)\right| + \varepsilon
},
\]
the cosine alignment
\[
\mathrm{align}_t =
\frac{
\langle {\Delta}_t, {\Delta}_t^\star \rangle
}{
\|{\Delta}_t\|_2
\|{\Delta}_t^\star\|_2 + \varepsilon
},
\]
and the residual linearization error
\[
\mathrm{modelerr}_t =
\frac{
\left\|
\mathbf{r}_t(\mathbf{x}_t+{\Delta}_t)
-
\left(\mathbf{r}_t + \mathbf{J}_t{\Delta}_t\right)
\right\|_2
}{
\|\mathbf{J}_t{\Delta}_t\|_2 + \varepsilon
}.
\]
Here
\[
\mathcal{E}^{\mathrm{frz}}_t({\Delta})
:=
\mathcal{E}_t(\mathbf{x}_t+{\Delta})
\]
denotes the frozen local objective evaluated at
$\mathbf{x}_t + {\Delta}$.

\begin{table}[t]
\centering
\small
\setlength{\tabcolsep}{5pt}
\caption{
Top-25\% large-update frozen-step diagnostic.
Steps are ranked by
$\eta_t =
\|\mathbf{J}_t{\Delta}_t\|_2/(\|\mathbf{r}_t\|_2+\varepsilon)$.
The CLAMP step ${\Delta}_t$ is compared with
${\Delta}_t^\star$, the optimizer of the exact frozen local
subproblem restricted to the same GMRES Krylov subspace.
}
\label{tab:large_update_frozen_step}
\begin{tabular}{lcccc}
\toprule
\textbf{Task}
& \(\mathrm{med.\;}\eta_t\)
& \(\mathrm{objgap}_t\)
& \(\mathrm{align}_t\)
& \(\mathrm{modelerr}_t\) \\
\midrule
Gaussian deblurring  & 0.66 & 0.14 & 0.92 & 0.28 \\
Nonlinear deblurring & 0.63 & 0.21 & 0.83 & 0.36 \\
\bottomrule
\end{tabular}
\end{table}

Table~\ref{tab:large_update_frozen_step} shows that, at similar realized update
sizes, Gaussian deblurring remains reasonably close to the exact frozen-local
optimum, with \(\mathrm{objgap}=0.14\) and \(\mathrm{align}=0.92\). Nonlinear
deblurring is clearly harder, with a larger objective gap, lower alignment, and
larger model error. Thus, sensitivity is not explained by step size alone and is
operator dependent. Nevertheless, even for nonlinear deblurring, the realized
step remains substantially aligned with the exact frozen-local optimizer. This
suggests that damping and repeated re-linearization keep CLAMP updates in a
reliable operating regime.

\paragraph{Direction of prior-aligned damping.}
We examine the direction used in the rank-one damping metric. Recall that
\[
\mathbf{d}_t = \mathbf{x}_t - \hat{\mathbf{x}}_{0|t},
\qquad
\mathbf{u}_t =
\frac{\mathbf{d}_t}{\|\mathbf{d}_t\|_2},
\]
and CLAMP uses
\[
\mathbf{H}_t = \lambda_{\mathrm{id}} \mathbf{I} + \mathbf{u}_t\mathbf{u}_t^\top .
\]
The vector $\mathbf{u}_t$ spans the denoiser-residual direction between the current
diffusion state and the denoised prediction. Extra damping along this direction
limits updates in the subspace most directly associated with the local prior
correction, while orthogonal directions remain controlled by the isotropic
floor.

To test whether the benefit is specific to this direction, we compare $\mathbf{u}_t\mathbf{u}_t^\top$ with an orthogonal rank-one metric
$\mathbf{q}_t\mathbf{q}_t^\top$, where
$\mathbf{q}_t^\top\mathbf{u}_t = 0$, and a random rank-one metric
$\boldsymbol{\rho}_t\boldsymbol{\rho}_t^\top$, where $\boldsymbol{\rho}_t$ is a
random unit vector.

\begin{table}[t]
\centering
\small
\setlength{\tabcolsep}{6pt}
\caption{Prior-aligned damping on latent phase retrieval.
We report averages over 10 samples. Here
$\mathbf{q}_t^\top\mathbf{u}_t = 0$, and $\boldsymbol{\rho}_t$ is a random unit
vector. Higher PSNR/SSIM and lower LPIPS are better.}
\label{tab:prior_aligned_damping}
\begin{tabular}{lccc}
\toprule
\textbf{Variant} & \textbf{PSNR $\uparrow$} & \textbf{SSIM $\uparrow$} &
\textbf{LPIPS $\downarrow$} \\
\midrule
$\mathbf{u}_t\mathbf{u}_t^\top$ & 30.74 & \textbf{0.858} & \textbf{0.241} \\
$\mathbf{q}_t\mathbf{q}_t^\top$ & 28.93 & 0.812 & 0.270 \\
$\boldsymbol{\rho}_t\boldsymbol{\rho}_t^\top$ & 28.92 & 0.814 & 0.270 \\
$\lambda_{\mathrm{id}} \mathbf{I} + \mathbf{u}_t\mathbf{u}_t^\top$ & \textbf{30.77} & 0.857 & 0.243 \\
$\lambda_{\mathrm{id}} \mathbf{I} + \mathbf{q}_t\mathbf{q}_t^\top$ & 28.95 & 0.814 & 0.271 \\
$\lambda_{\mathrm{id}} \mathbf{I} + \boldsymbol{\rho}_t\boldsymbol{\rho}_t^\top$ & 28.96 & 0.811 & 0.271 \\
\bottomrule
\end{tabular}
\end{table}

Table~\ref{tab:prior_aligned_damping} shows that the benefit is specific to the
prior-aligned direction. Orthogonal and random directions degrade performance
markedly, whereas the aligned rank-one metric remains competitive. The default
metric $\lambda_{\mathrm{id}}\mathbf{I} + \mathbf{u}_t\mathbf{u}_t^\top$ gives the best PSNR overall and
retains strong SSIM and LPIPS. The isotropic floor improves robustness, while
the directional gain comes from alignment with the denoiser-residual
direction.

\paragraph{Curvature variants and Jacobian-action costs.}
We compare curvature variants to justify the one-sided approximation
used by CLAMP. The exact denoiser-pullback Gauss--Newton curvature contains both
the left denoiser pullback and the right denoiser forward action,
\[
    \mathbf{J}_F^\top \mathbf{J}_A^\top \mathbf{J}_A\mathbf{J}_F .
\]
CLAMP keeps the left factor \(\mathbf{J}_F^\top\), which maps curvature information back
to diffusion-state coordinates, but drops the rightmost \(\mathbf{J}_F\) in the curvature
action, yielding the one-sided operator
\[
    \mathbf{J}_F^\top \mathbf{J}_A^\top \mathbf{J}_A .
\]
All variants in Table~\ref{tab:curvature_variants} use the same transition rule, damping, and Krylov budget.

\begin{table}[t]
\centering
\small
\setlength{\tabcolsep}{5pt}
\caption{
Curvature variants on super resolution $4\times$ and nonlinear
deblurring. We report PSNR/SSIM/LPIPS and wall-clock runtime in seconds.
}
\label{tab:curvature_variants}
\begin{tabular}{lcccccccc}
\toprule
\multirow{2}{*}{\textbf{Variant}}
& \multicolumn{4}{c}{\textbf{Super resolution 4$\times$}}
& \multicolumn{4}{c}{\textbf{Nonlinear deblurring}} \\
\cmidrule(lr){2-9}
& \textbf{PSNR $\uparrow$} & \textbf{SSIM $\uparrow$} & \textbf{LPIPS $\downarrow$} & \textbf{Run-time (s)}
& \textbf{PSNR $\uparrow$} & \textbf{SSIM $\uparrow$} & \textbf{LPIPS $\downarrow$} & \textbf{Run-time (s)} \\
\midrule
\(\mathbf{J}_A^\top \mathbf{J}_A\)
& 28.50 & 0.758 & 0.334 & 1.7
& 24.85 & 0.634 & 0.339 & 18.1 \\

\(\mathbf{J}_A^\top \mathbf{J}_A\mathbf{J}_F\)
& 28.18 & 0.737 & 0.356 & 28.6
& 22.06 & 0.342 & 0.537 & 116.9 \\

\(\mathbf{J}_F^\top \mathbf{J}_A^\top \mathbf{J}_A\)
& 29.81 & 0.841 & 0.229 & 6.4
& 30.52 & 0.859 & 0.170 & 30.3 \\

\(\mathbf{J}_F^\top \mathbf{J}_A^\top \mathbf{J}_A\mathbf{J}_F\)
& 29.81 & 0.841 & 0.227 & 32.0
& 30.42 & 0.843 & 0.170 & 132.6 \\
\bottomrule
\end{tabular}
\end{table}

\begin{table}[t]
\centering
\small
\setlength{\tabcolsep}{7pt}
\caption{Jacobian primitive costs in milliseconds.
The forward denoiser JVP $\mathbf{J}_F\mathbf{v}$ is the dominant extra cost required by
the full curvature operator.}
\label{tab:jacobian_primitive_costs}
\begin{tabular}{lcc}
\toprule
\textbf{Primitive} & \textbf{Super resolution 4$\times$} & \textbf{Nonlinear deblurring} \\
\midrule
$\mathbf{J}_A\mathbf{v}$        & 1.2   & 13.2 \\
$\mathbf{J}_A^\top\mathbf{v}$ & 1.2   & 5.2  \\
$\mathbf{J}_F\mathbf{v}$        & 112.9 & 108.5 \\
$\mathbf{J}_F^\top\mathbf{v}$ & 15.9  & 15.8 \\
\bottomrule
\end{tabular}
\end{table}

For efficiency, the dominant extra cost in the full curvature is the forward
denoiser JVP $\mathbf{J}_F\mathbf{v}$. Table~\ref{tab:jacobian_primitive_costs} shows that $\mathbf{J}_F\mathbf{v}$ costs
about 113 ms on super resolution and 109 ms on nonlinear deblurring, whereas
$\mathbf{J}_F^\top\mathbf{v}$ costs about 16 ms and the forward-operator Jacobian
actions, $\mathbf{J}_A\mathbf{v}$ and $\mathbf{J}_A^\top\mathbf{v}$, are much cheaper. Table~\ref{tab:curvature_variants}
shows that paying for $\mathbf{J}_F\mathbf{v}$ does not improve quality: the full
curvature changes PSNR little but increases runtime from 6.4 s to 32.0 s on
super resolution and from 30.3 s to 132.6 s on nonlinear deblurring. In
contrast, dropping the left factor $\mathbf{J}_F^\top$ loses the pullback that maps
curvature information back to diffusion-state coordinates and substantially
hurts performance, especially on nonlinear deblurring. Therefore, retaining
$\mathbf{J}_F^\top$ while omitting the rightmost $\mathbf{J}_F$ preserves the essential pullback
geometry and avoids the dominant computational bottleneck, which supports the
one-sided curvature operator as the best accuracy--efficiency trade-off.

\section{Qualitative MRI Results and Implementation Details}
\label{appendix:mri}

This section provides supplementary implementation details and qualitative comparisons for accelerated multi-coil MRI reconstruction. We adopt the standard linear forward model
\[
\mathbf{y}=\mathbf{A}\mathbf{x}_0+\mathbf{n},
\qquad
\mathbf{A} = \mathbf{M}\mathbf{F}\mathbf{S},
\]
where $\mathbf{M}$ is a $k$-space subsampling mask, $\mathbf{F}$ is the 2D Fourier transform,
and $\mathbf{S}$ denotes the coil-sensitivity operator applied coil-wise, with
measurements stacked across coils. For subsampling, we use Poisson-disc masks
at the target acceleration factors.

We conduct experiments on the SKM-TEA dataset~\cite{Desai2022SKMTEAAD}. We estimate sensitivity maps $S$ using JSENSE and use SENSE for coil combination as well as to generate fully-sampled reference reconstructions from fully-sampled data. We compare against baselines, including DPS~\cite{chung2023dps}, DAPS~\cite{Zhang_2025_CVPR}, DDS~\cite{Chung2023DecomposedDS}, and Score-Med~\cite{song2022inverse_med}.

Our diffusion-based prior is built upon a U-Net backbone consisting of approximately 167 million parameters, integrated within the EDM framework. The model was trained for 200,000 iterations using the RAdam~\cite{liu2019radam} optimizer with a learning rate of $1\times10^{-4}$ and a batch size of 4 on two NVIDIA RTX 4090 GPUs. The model processes complex-valued images as two-channel (real/imaginary) inputs at a resolution of 512 $\times$ 512.

For testing, we evaluate the model on the SKM-TEA test split, which comprises 39 3D volumes. From these volumes, we randomly sample 100 slices to conduct a robust comparative analysis of the reconstruction performance. During inference, we set the number of diffusion steps to $T=25$, the number of GMRES iterations to $K=10$, and the identity-damping parameter to $\lambda_{\mathrm{id}}=1.0$. To evaluate reconstruction fidelity, we measure PSNR and SSIM on magnitude images. These metrics are calculated per slice and averaged across the sampled test set to yield the quantitative scores. Table~\ref{tab:mri_table} presents the quantitative results and Figure~\ref{fig:mri_qual} shows qualitative comparisons including reconstructed magnitude images and zoomed-in crops of anatomical structures.

\begin{figure*}[t]
\centering

\setlength\tabcolsep{1pt}
\resizebox{\textwidth}{!}{
\begin{tabular}{cccccc}
\includegraphics[scale=0.095]{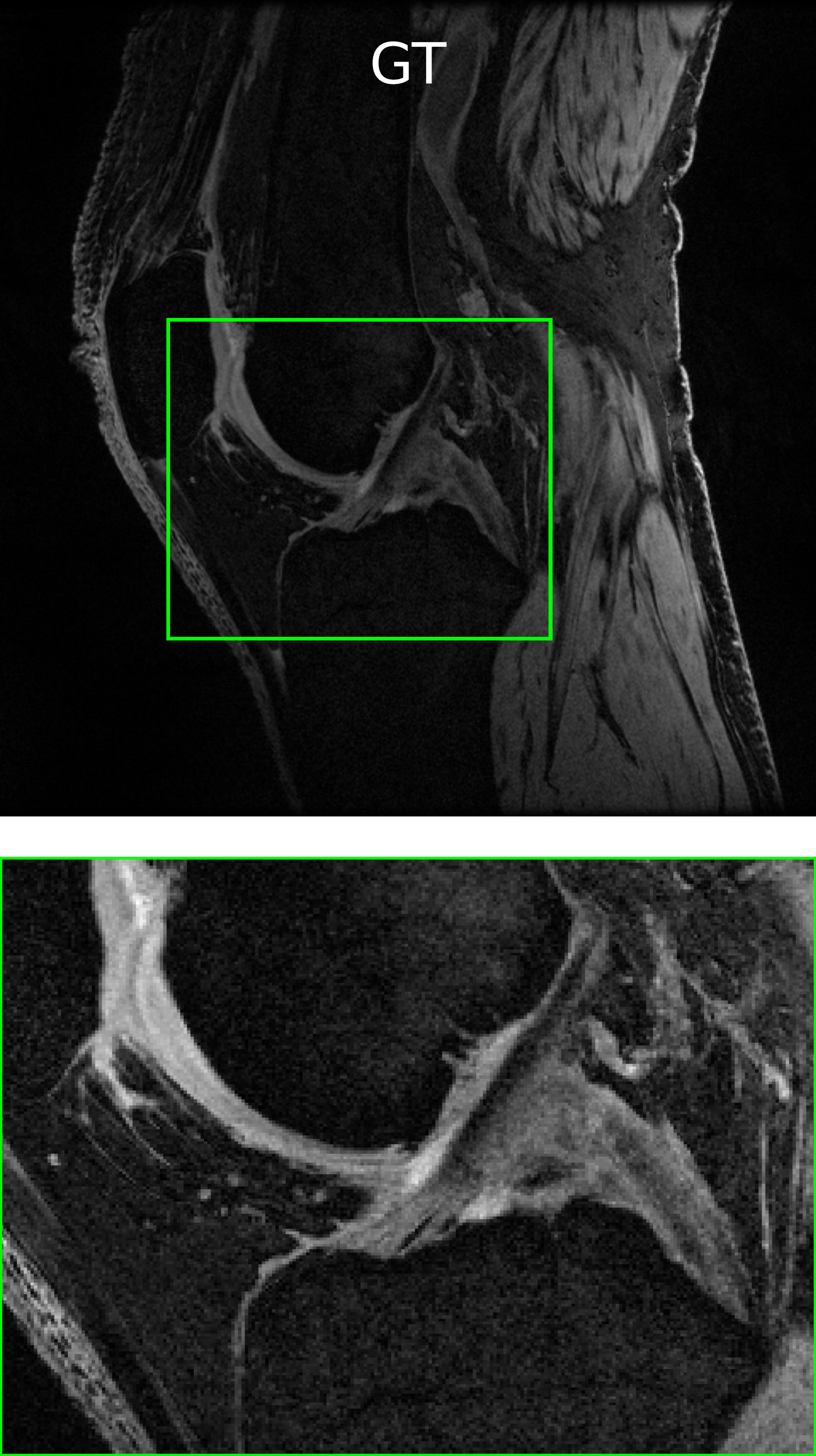} &
\includegraphics[scale=0.095]{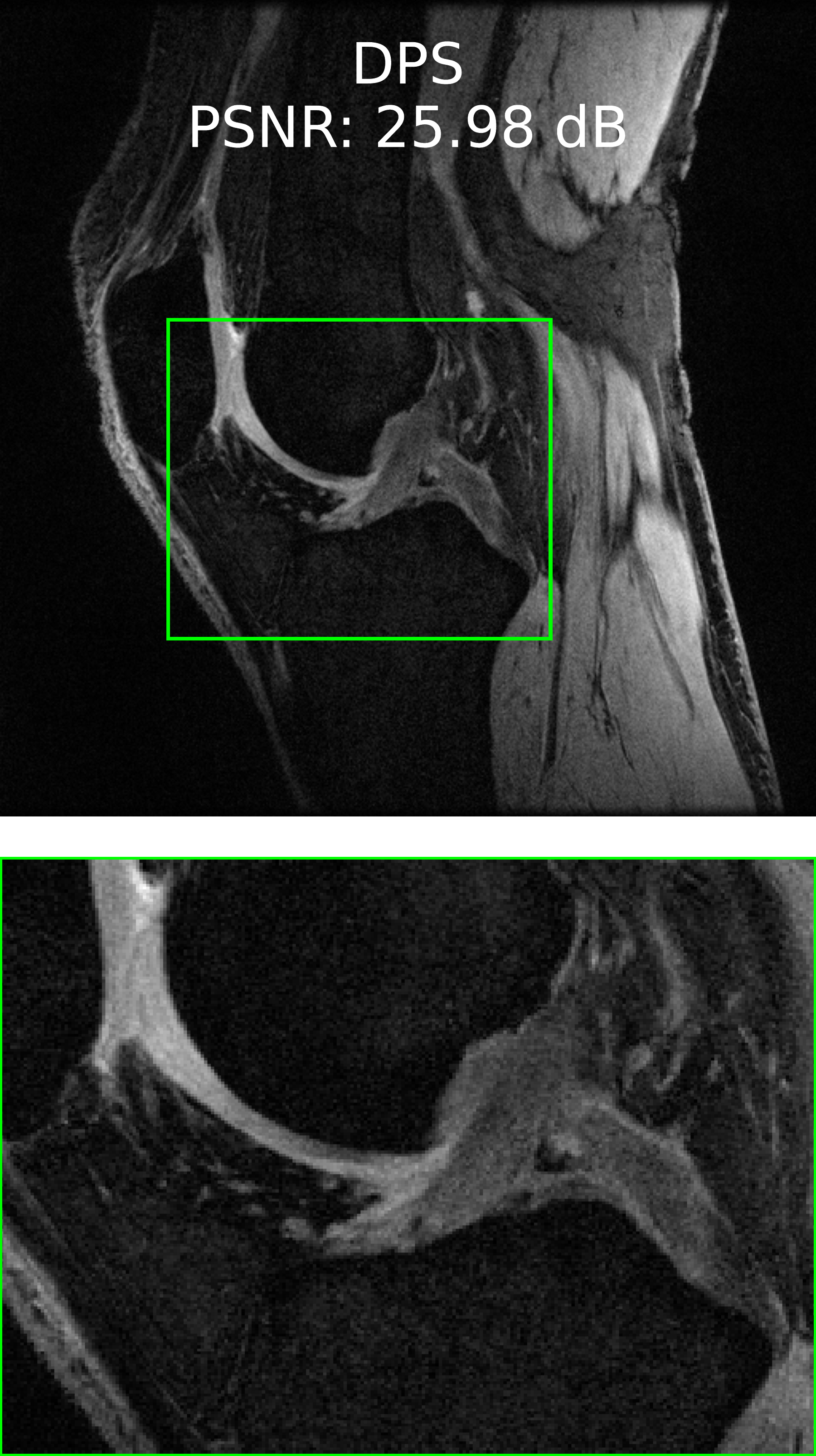} &
\includegraphics[scale=0.095]{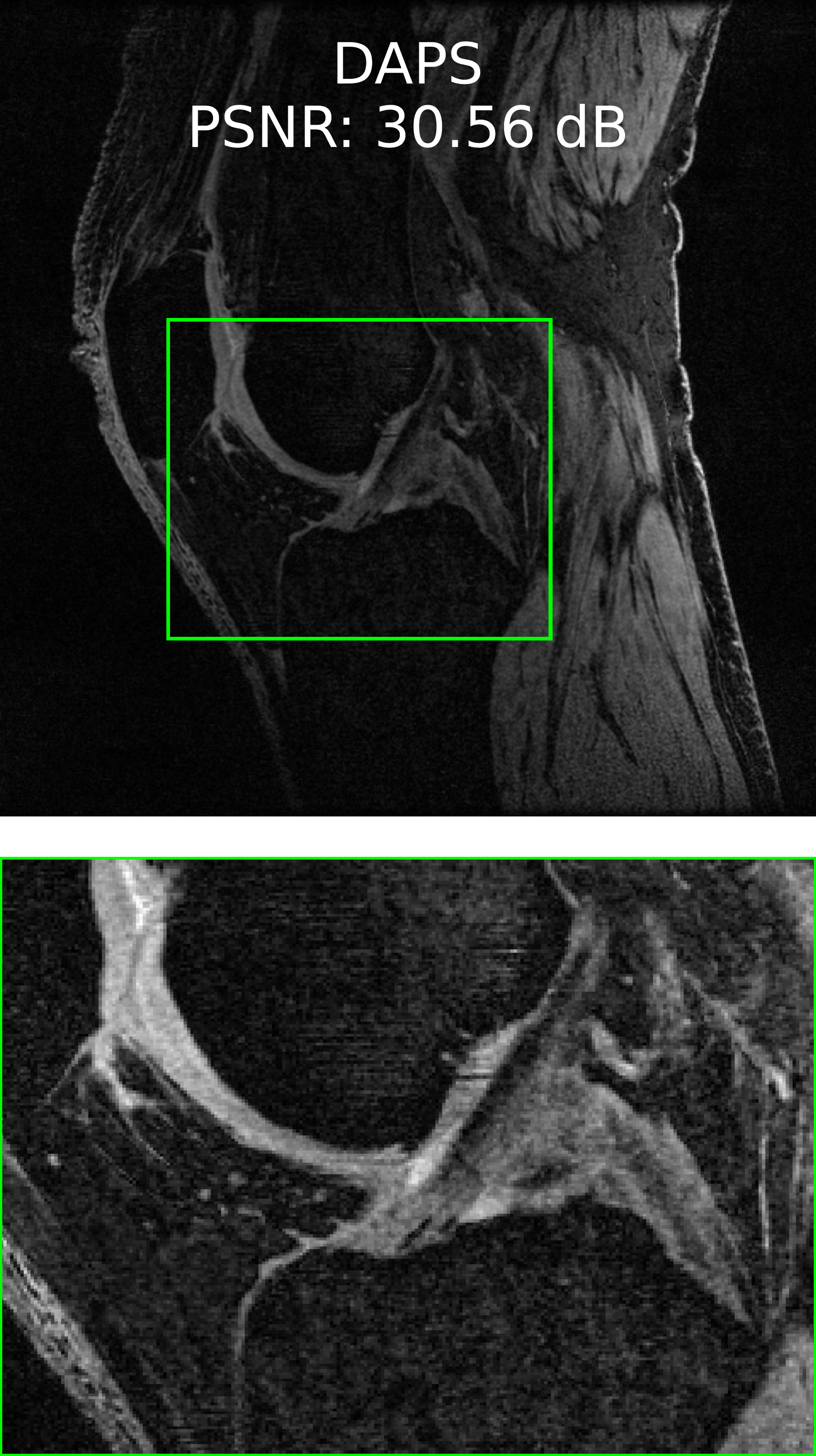} &
\includegraphics[scale=0.095]{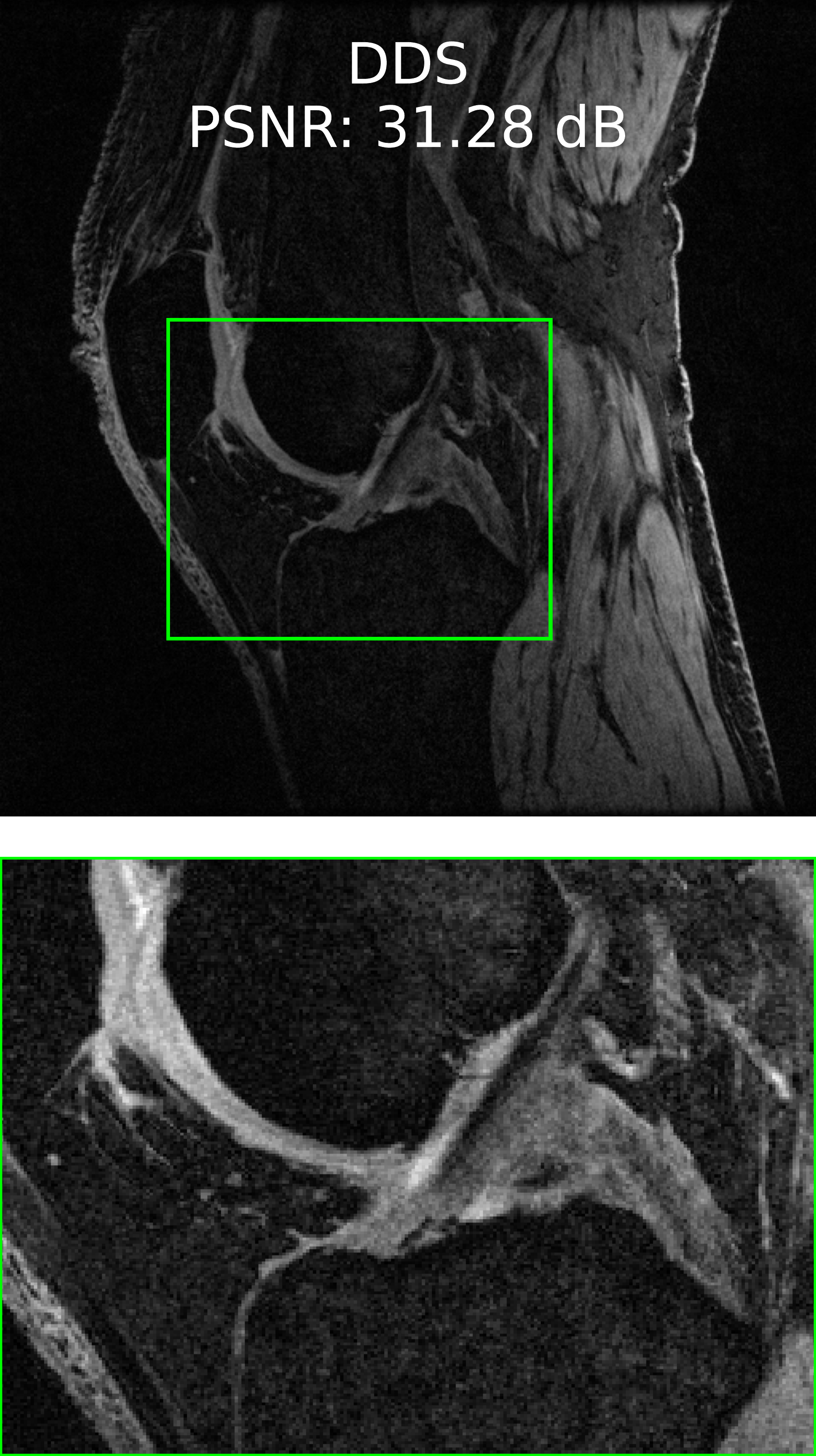} &
\includegraphics[scale=0.095]{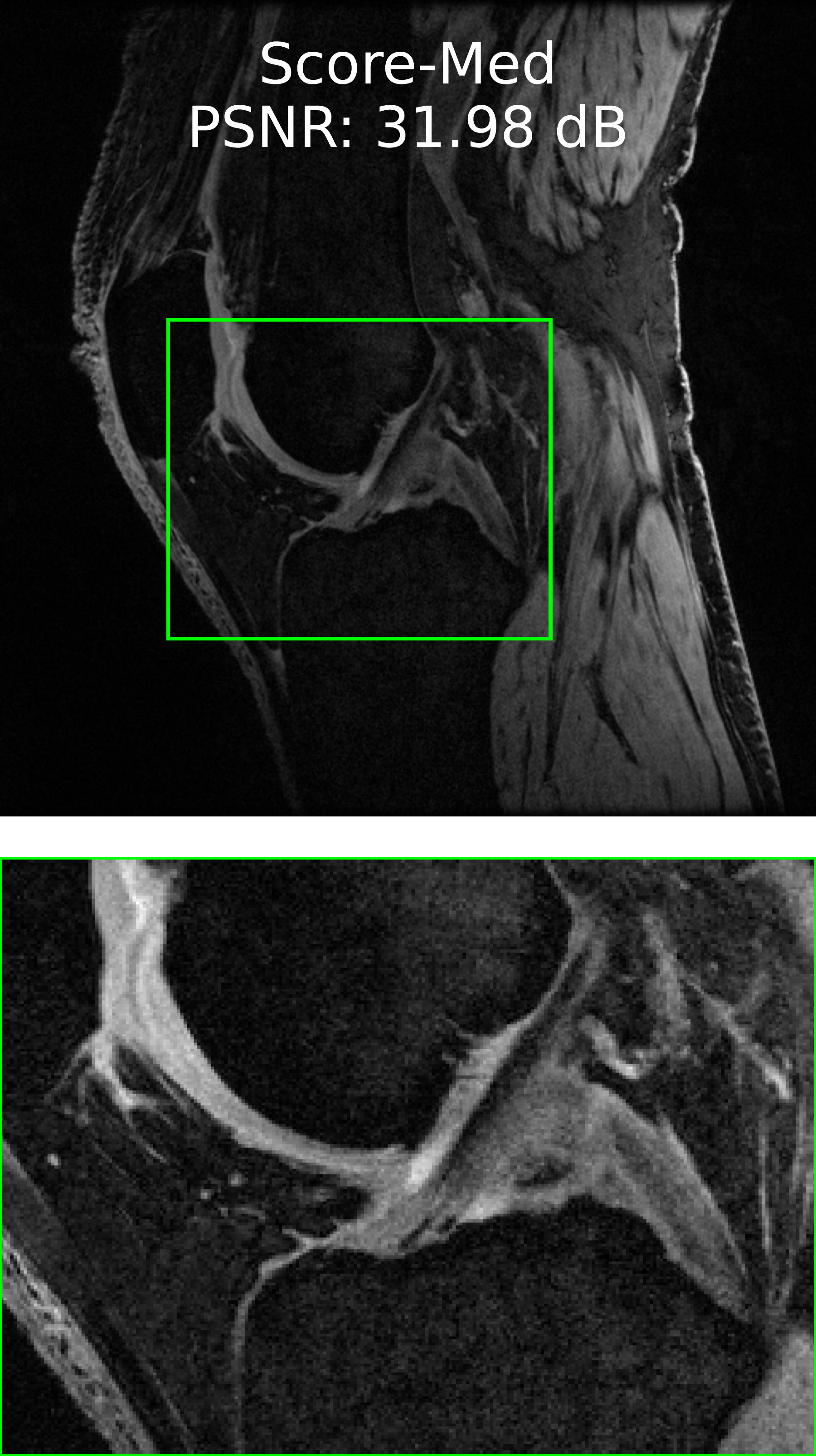} &
\includegraphics[scale=0.095]{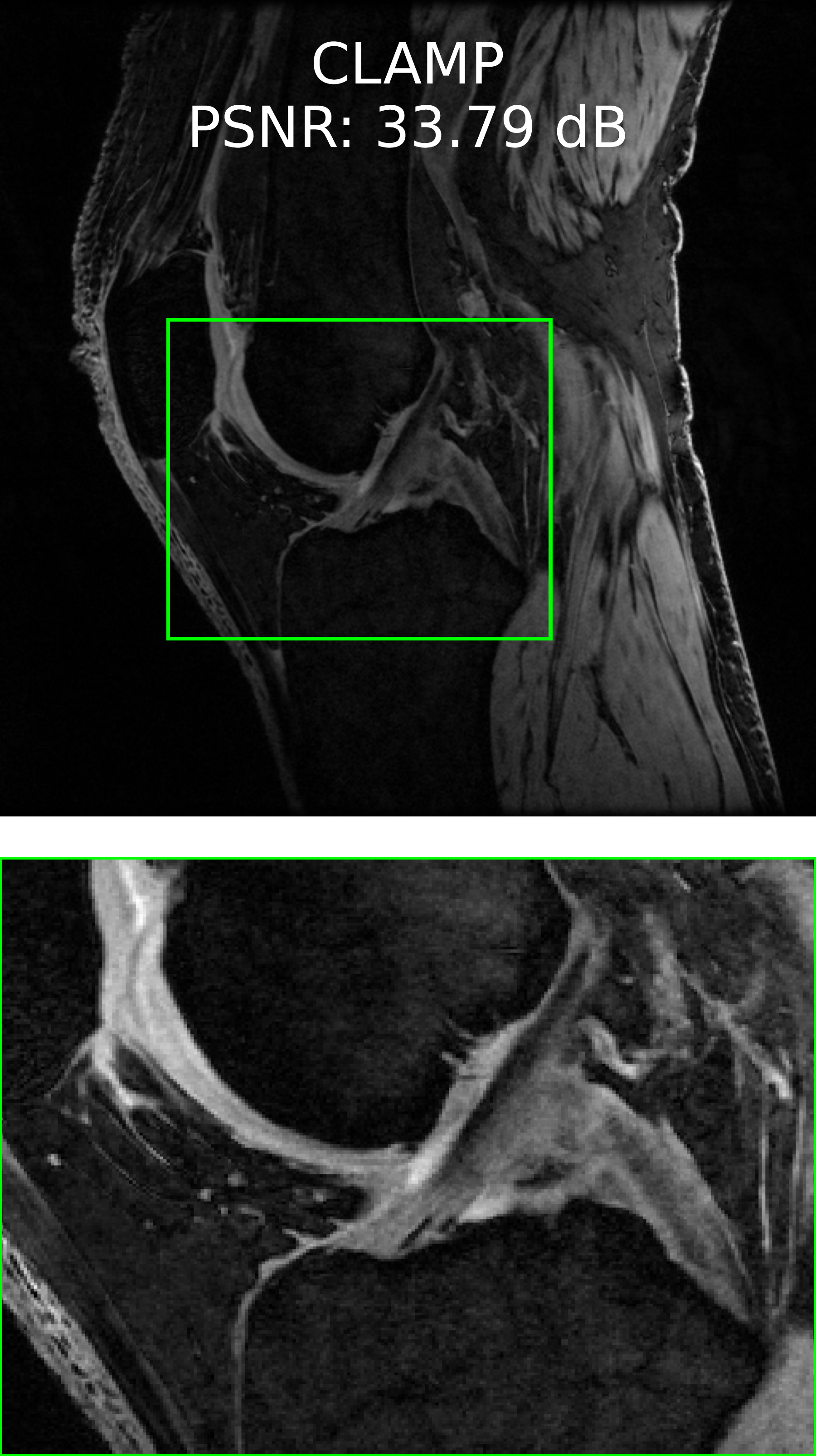} \\
\end{tabular}}
\caption{Reconstructions from $\times8$ undersampled measurements are shown for various diffusion-based baselines. The top row displays the full magnitude images with their respective PSNR (dB) values. The bottom row provides zoomed-in crops of the region highlighted by the green box. \textsc{clamp} demonstrates superior restoration of fine anatomical details and sharper edges compared to other baselines.}
\label{fig:mri_qual} 
\end{figure*}

\section{Future work and Limitations}
\label{appendix:future_limit}
Our method relies on automatic differentiation to compute Jacobian actions of the forward operator, specifically VJPs and JVPs. As a result, the framework may be difficult to apply directly when the measurement process contains non-differentiable components, such as hard quantization, clipping, or codec pipelines (e.g., JPEG). While replacing such operators with differentiable surrogates provides a practical workaround, this can introduce modeling mismatch between the true measurement process and its surrogate, which may adversely affect reconstruction quality. Developing principled and robust mechanisms to handle non-differentiable or black-box forward operators remains an important direction for future research.

Furthermore, our current experimental validation focuses on 2D image data. Extending the proposed framework to higher-dimensional modalities is a promising avenue for future work, including 3D reconstruction and video inverse problems. In these settings, scaling the curvature-guided corrections and associated linear solves to accommodate increased dimensionality—as well as exploiting spatiotemporal structure and temporal consistency—will be key challenges.

\clearpage
\begin{figure}
\centering
  \makebox[\textwidth][c]{%
    \includegraphics[width=0.55\linewidth]{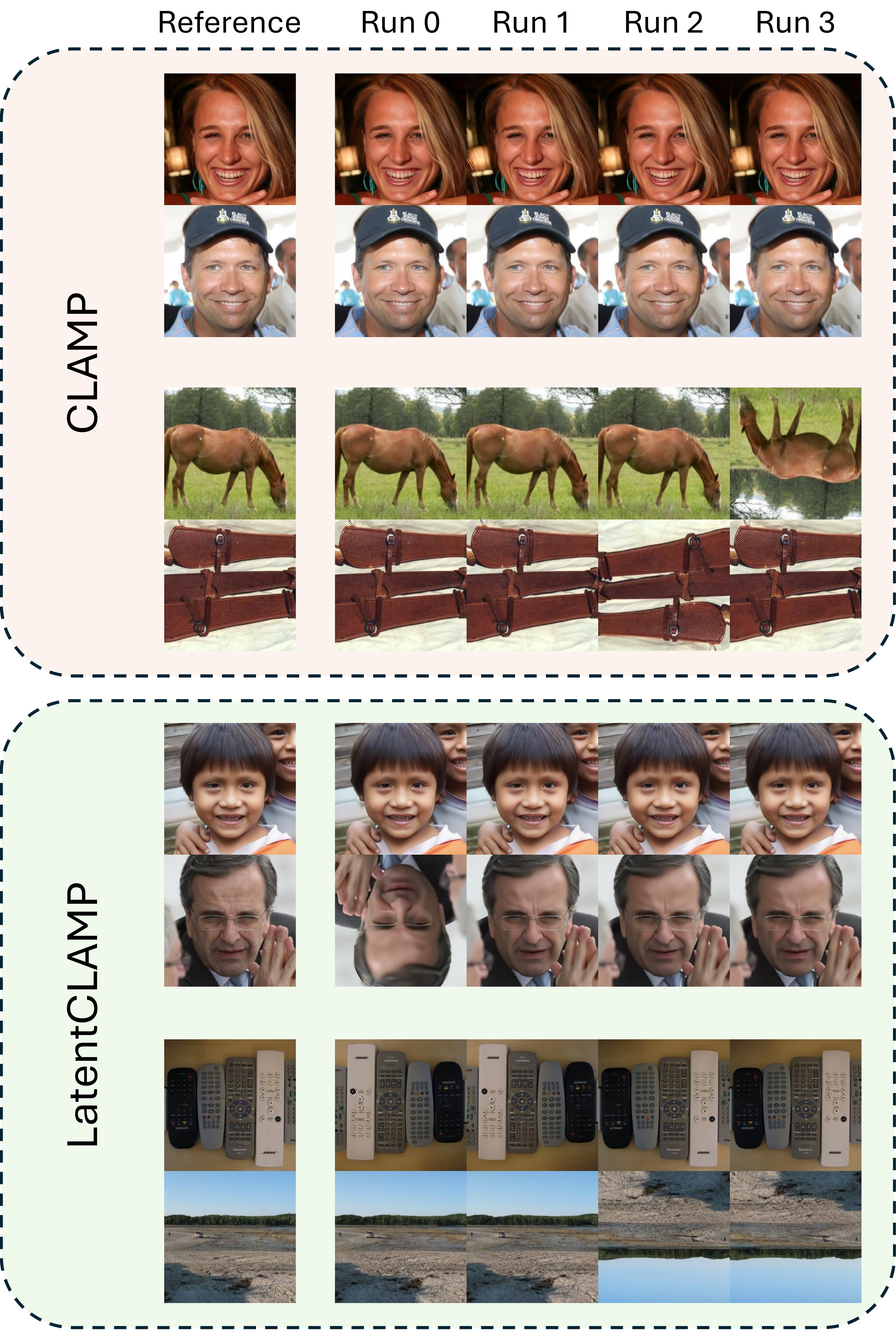}
  }
\caption{\textbf{Additional visualizations: phase retrieval (FFHQ and ImageNet).}
Additional reconstructions on FFHQ and ImageNet for phase retrieval; for each input, we visualize all reconstructions obtained from four independent trials.}

\end{figure}
\begin{figure}
\centering
  \makebox[\textwidth][c]{%
    \includegraphics[width=\linewidth]{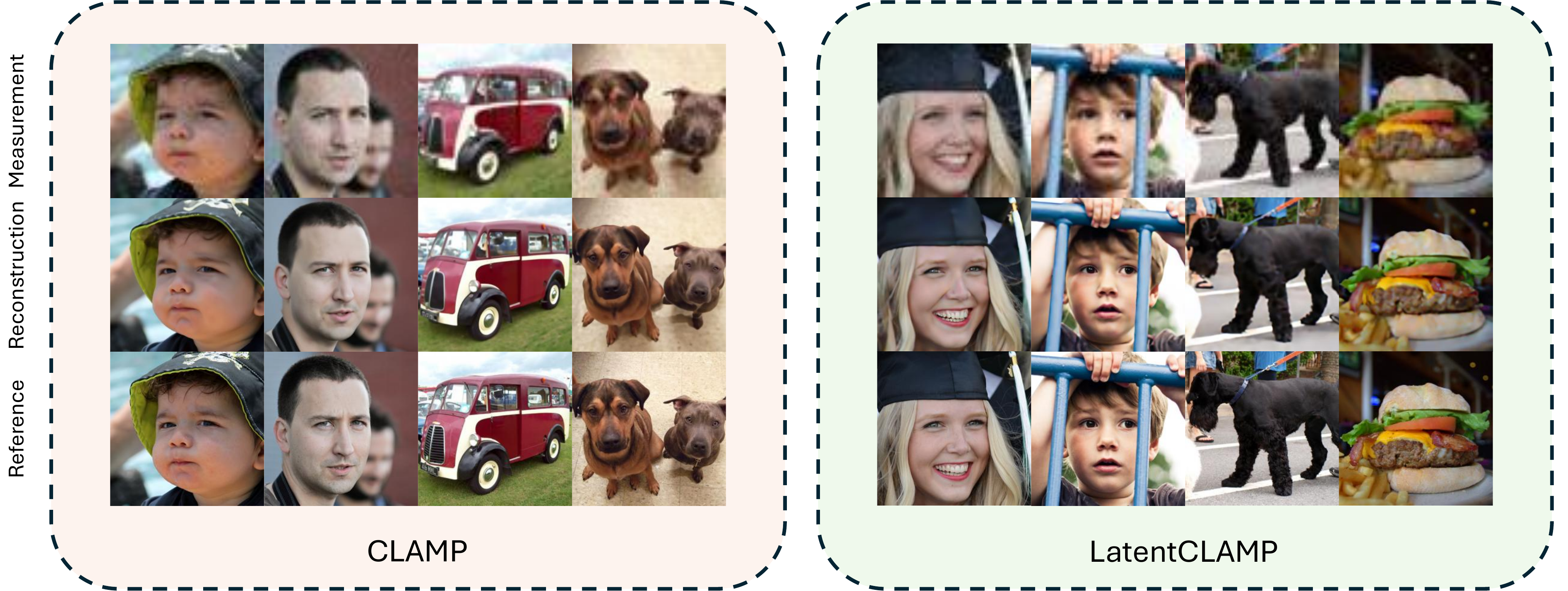}
  }
\caption{\textbf{Additional visualizations: super resolution 4$\times$ (FFHQ and ImageNet).}
Additional reconstructions on FFHQ and ImageNet for super resolution 4$\times$.}

\end{figure}
\begin{figure}
\centering
  \makebox[\textwidth][c]{%
    \includegraphics[width=\linewidth]{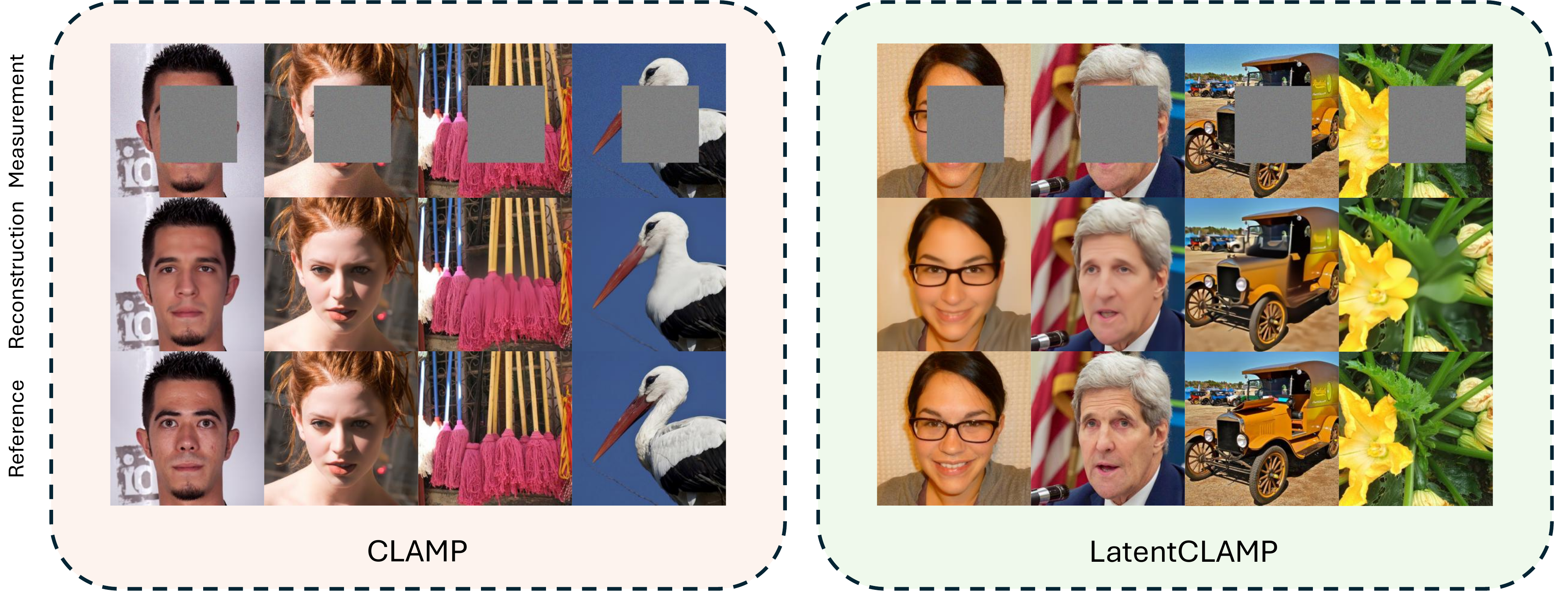}
  }
\caption{\textbf{Additional visualizations: box inpainting (FFHQ and ImageNet).}
Additional reconstructions on FFHQ and ImageNet for box inpainting}

\end{figure}
\begin{figure}
\centering
  \makebox[\textwidth][c]{%
    \includegraphics[width=\linewidth]{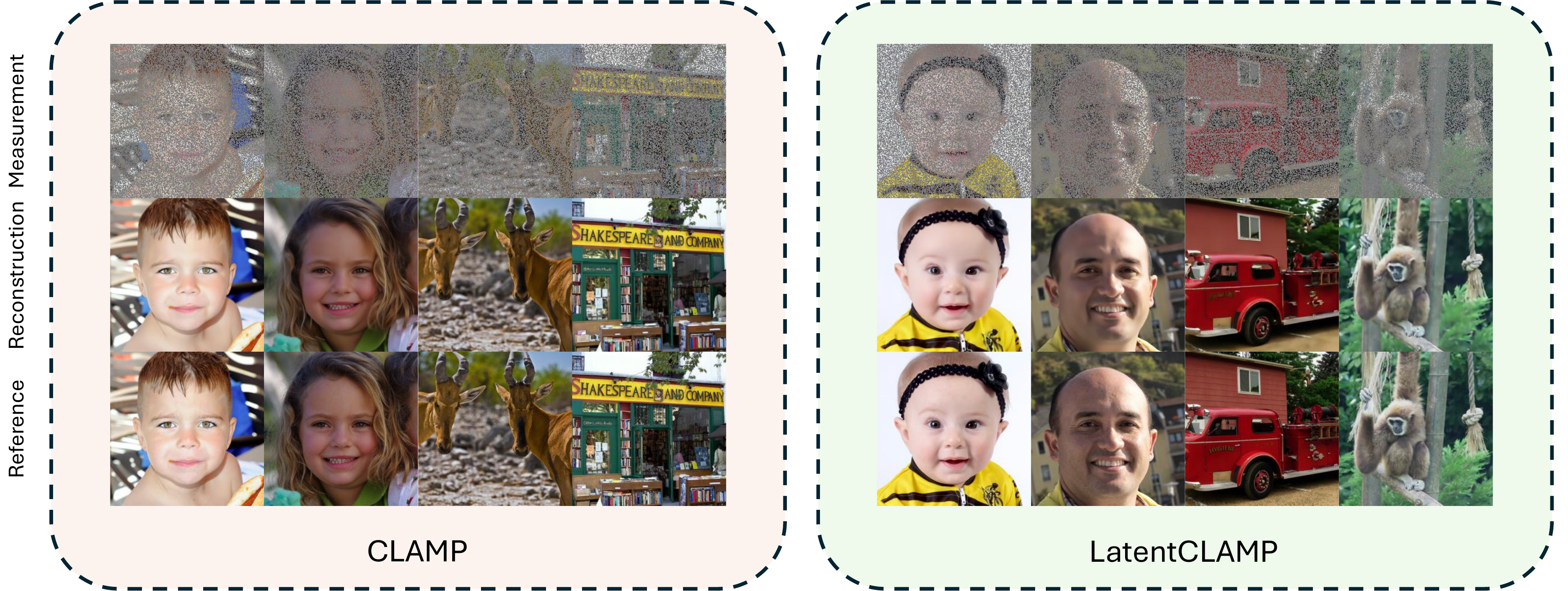}
  }
\caption{\textbf{Additional visualizations: random inpainting (FFHQ and ImageNet).}
Additional reconstructions on FFHQ and ImageNet for random inpainting.}

\end{figure}
\begin{figure}
\centering
  \makebox[\textwidth][c]{%
    \includegraphics[width=\linewidth]{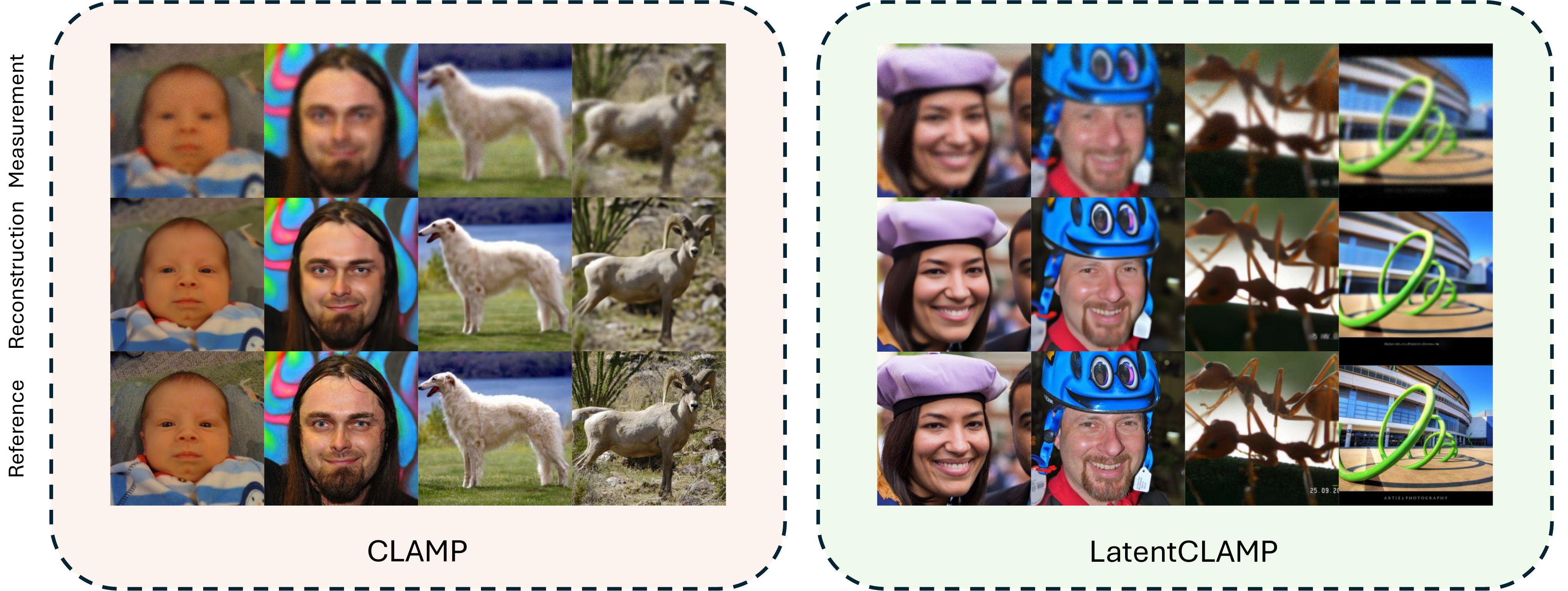}
  }
\caption{\textbf{Additional visualizations: Gaussian deblurring (FFHQ and ImageNet).}
Additional reconstructions on FFHQ and ImageNet for Gaussian deblurring.}

\end{figure}
\begin{figure}
\centering
  \makebox[\textwidth][c]{%
    \includegraphics[width=\linewidth]{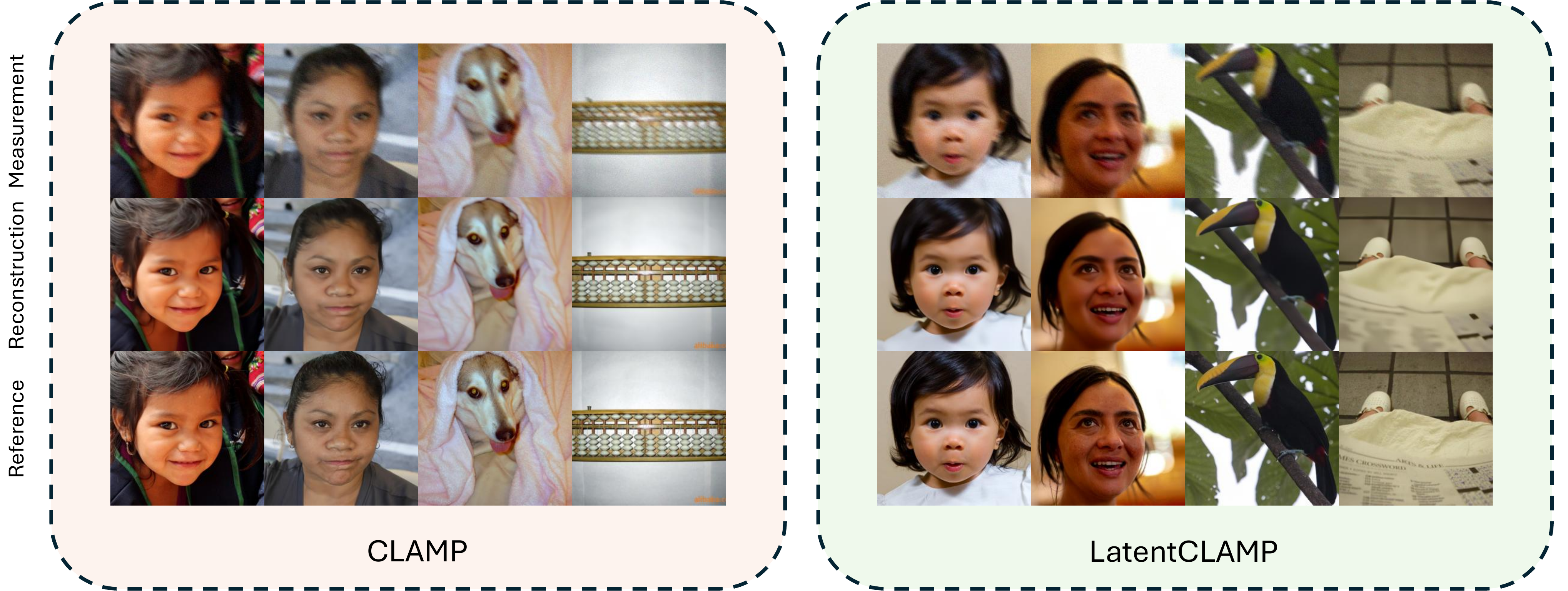}
  }
\caption{\textbf{Additional visualizations: motion deblurring (FFHQ and ImageNet).}
Additional reconstructions on FFHQ and ImageNet for motion deblurring.}

\end{figure}
\begin{figure}
\centering
  \makebox[\textwidth][c]{%
    \includegraphics[width=\linewidth]{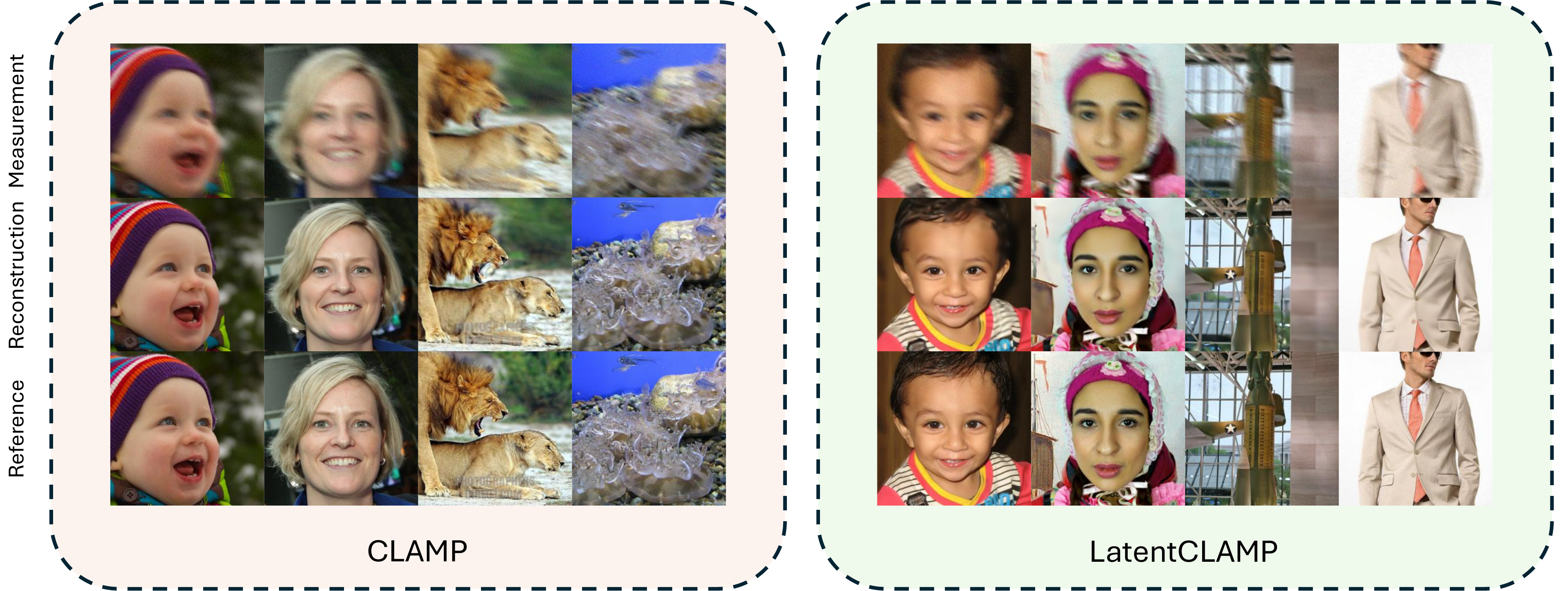}
  }
\caption{\textbf{Additional visualizations: nonlinear deblurring (FFHQ and ImageNet).}
Additional reconstructions on FFHQ and ImageNet for nonlinear deblurring.}

\end{figure}
\begin{figure}
\centering
  \makebox[\textwidth][c]{%
    \includegraphics[width=\linewidth]{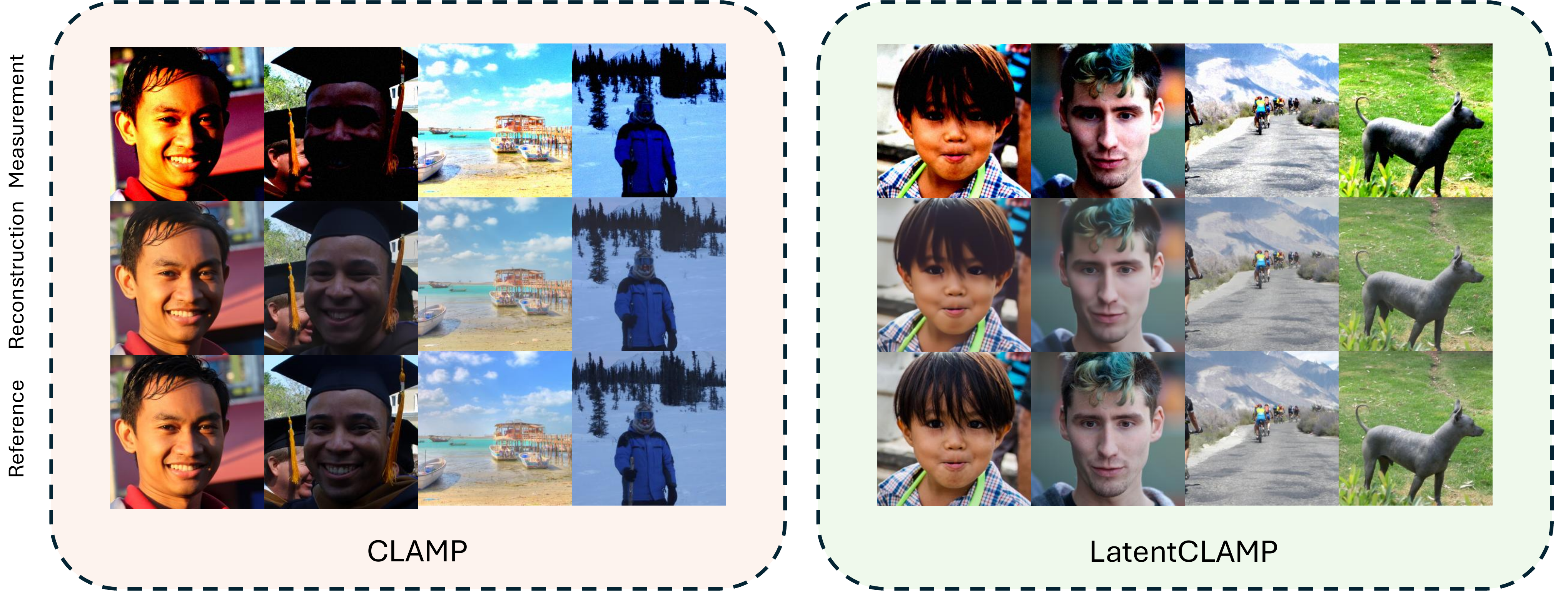}
  }
\caption{\textbf{Additional visualizations: high dynamic range (FFHQ and ImageNet).}
Additional reconstructions on FFHQ and ImageNet for high dynamic range reconstruction.}

\end{figure}

\end{document}